\documentclass[twoside,11pt]{article}
\usepackage{blindtext}

\usepackage[preprint]{jmlr2e}

\usepackage{lastpage}
\jmlrheading{23}{2022}{1-\pageref{LastPage}}{1/21; Revised 5/22}{9/22}{21-0000}{Author One and Author Two}

\ShortHeadings{Blessings of Multiple Good Arms in Multi-Objective Linear Bandits}{Blessings of Multiple Good Arms in Multi-Objective Linear Bandits} 
\firstpageno{1}

\usepackage{amsmath}
\usepackage{amssymb}
\usepackage{amsthm} 
\usepackage{abbreviations}
\usepackage{natbib} 
    \setcitestyle{authoryear,round,citesep={;},aysep={,},yysep={;}}
    \renewcommand{\cite}[1]{\citep{#1}} 
    
\usepackage{hyperref} 
    \hypersetup{ %
        pdftitle={},
        pdfsubject={},
        pdfkeywords={},
        pdfborder=0 0 0,
        pdfpagemode=UseNone,
        colorlinks=true,
        linkcolor=black,
        citecolor=mydarkblue, 
        filecolor=mydarkblue,
        urlcolor=mydarkblue,
    }

\usepackage{url}
\usepackage{mathtools}
\usepackage{amsthm}
\usepackage{algorithm}
\usepackage{algorithmic}
\usepackage{abbreviations}
\usepackage{dsfont}
\usepackage{tabularx}
\usepackage{makecell}
\usepackage{graphicx}
\usepackage{subcaption}
\usepackage{booktabs}       
\usepackage{subfiles}

\usepackage[capitalize,noabbrev]{cleveref}

\newcommand{\rglBound}{\gamma}
\newcommand{\rglBall}{\alpha}
\newcommand{\regret}{\mathbf{EPR}}
\newcommand{\truethetam}{\theta_{m}^* }
\newcommand{\hatthetamt}{\hat{\theta}_m(t)}
\newcommand{\hatthetams}{\hat{\theta}_m(s)}
\newcommand{\targetthetas}{\tilde{\theta}(s)}

\newcommand{\lambdamin}{\lambda_{\min}}
\newcommand{\lambdainc}{\lambda_\text{inc}}
\newcommand{\xmax}{x_{\max}}

\newcommand{\mogro}{\texttt{MOGRO}}
\newcommand{\mogrorw}{\texttt{MOGRO-RW}}
\newcommand{\mogrorr}{\texttt{MOGRO-RR}}
\newcommand{\molbepsgreedy}{$\texttt{MOLB-}\epsilon\texttt{-greedy}$}
\newcommand{\molbucb}{\texttt{MOLB-UCB}}
\newcommand{\molbts}{\texttt{MOLB-TS}}

\addtocontents{toc}{\protect\setcounter{tocdepth}{0}} 

\begin{document}

\title{Blessings of Multiple Good Arms \\in Multi-Objective Linear Bandits}

\author{\name Heesang Ann \email sang3798@snu.ac.kr \\
       \addr 
       Seoul National University \\
       \name Min-hwan Oh \email minoh@snu.ac.kr \\
       \addr 
       Seoul National University}

\editor{My editor}

\maketitle

\begin{abstract}
The multi-objective bandit setting has traditionally been regarded as more complex than the single-objective case, as multiple objectives must be optimized simultaneously. In contrast to this prevailing view, we demonstrate that when \emph{multiple good arms} exist for multiple objectives, they can induce a surprising benefit—\emph{implicit exploration}. Under this condition, we show that simple algorithms that greedily select actions in most rounds can nonetheless achieve strong performance, both theoretically and empirically. To our knowledge, this is the first study to introduce implicit exploration in both multi-objective and parametric bandit settings without any distributional assumptions on the contexts.
 We further introduce a framework for \emph{effective Pareto fairness}, which provides a principled approach to rigorously analyzing fairness of multi-objective bandit algorithms.
\end{abstract}


\setlength{\parindent}{0pt}

\section{Introduction}\label{sec:introduction}

Multi-objective decision-making problems have become increasingly prevalent in today’s complex real-world applications, ranging from recommendation systems to robotics. This trend naturally gives rise to \emph{multi-objective bandit problems}~\citep{Drugan2013, Turgay2018, Lu2019, Xu2022, Kim2023, Cheng2024, Crepon2024, zhang2024}, which generalize the single-objective bandit framework by incorporating multiple (potentially conflicting) objectives.
Although this extension may seem conceptually straightforward, balancing exploration and exploitation across multiple objectives significantly increases the complexity of the problem.

To address multi-objective bandit problems, most existing approaches focus on achieving Pareto optimality~\citep{Drugan2013, Tekin2018, Turgay2018, Lu2019, Kim2023, Cheng2024, park2025thompson}. 
However, these methods often involve updating empirical Pareto fronts in each round, leading to substantial computational overhead and limiting their suitability for often real-time and sequential decision-making applications.

While multi-objective problems are generally more complex than their single-objective counterparts, it is natural to ask whether multiple objectives could, in some cases, \emph{facilitate} learning rather than hinder it. Formally, we pose the following research question:
\begin{center}
    \emph{Can the presence of multiple objectives actually facilitate learning rather than hinder it?}
\end{center}
A priori, the answer is not always \emph{yes}. Nonetheless, there may be scenarios in which multiple objectives can be leveraged to achieve simpler, more efficient solutions. A positive answer to this question could reshape our perspective on multi-objective problems: instead of always resorting to increasingly complex methods, we might exploit a simpler
approach to handle multiple objectives more effectively, and to our knowledge, this perspective has been largely overlooked. 
Consequently, an important research direction is to identify the precise conditions under which multi-objective bandit problems become admissible—even potentially advantageous—for simple algorithms.

In this work, we show that the existence of \emph{good arms for multiple objectives} can enable simple algorithms that select actions greedily in most rounds to achieve strong performance.
Such ``goodness'' means that, for each objective, there is at least one arm that performs sufficiently well (and these arms may differ across objectives), a scenario commonly observed in practice.
We show that this condition gives rise to what we term \emph{implicit exploration}—the ability to collect informative feedback through greedy action selection without incurring additional exploration cost.
Concretely, we propose a novel and generic algorithm , \(\mogro\) (Algorithm~\ref{alg:mogro}), and prove that, under the suitable goodness assumption, it attains a regret bound of \(\tilde{\mathcal{O}}(\sqrt{T})\). 
To our knowledge, this is the first  algorithm that relying on \emph{implicit exploration} for multi-objective bandits.

From a broader perspective, our work is related to the literature on exploration-free linear bandits in the single-objective setting~\citep{Kannan2018, Raghavan2018, Bastani2020, kim2025LAC}, where greedy strategies can be efficient if the contexts are sufficiently diverse. However, our analysis here differs crucially: we do not rely on any stochasticity or diversity in the contexts (features). Notably, our algorithms perform effectively even in \emph{fixed} feature settings, marking the first result (to our knowledge) in which a mostly exploration-free algorithm achieves \emph{no-regret} without any distributional diversity assumptions in parametric bandits. This new insight suggests that multiple objectives can sometimes replace or augment the role that context diversity plays in single-objective settings, which can be of independent interest.

We evaluate our proposed algorithm $\mogro$, both theoretically and empirically, analyzing their regret performance and fairness. Our results show that these simple and efficient methods not only exhibit strong performance but also satisfy fairness guarantees---representing, to the best of our knowledge, the first such theoretical result on fairness in multi-objective bandits. 
Consequently, our work opens new perspectives in multi-objective bandit research by providing a new class of algorithms and offering stronger guarantees on both  Pareto optimality and fairness.

\subsection{Contributions}

\begin{itemize}

    \item We present and rigorously analyze a novel, sufficient condition on the \textit{goodness of arms}~(\cref{def:goodness}) under which implicit exploration occurs during greedy arm selection in multi-objective bandit problems \emph{without} relying on the commonly assumed context distributional assumptions in the greedy bandit literature~\citep{Kannan2018, Raghavan2018, Bastani2020, kim2025LAC}. 
    Notably, implicit exploration persists even in fixed context settings, rather than just stochastic environments. 
    Our key (and somewhat surprising) insight is that having multiple objectives can \emph{enhance} rather than hinder the learning process.
    
    \item 
    We propose and analyze a practical and generic algorithm, \(\mogro\), showing that under the goodness assumption, the algorithm that greedily select arms with respect to random weight scalarization attains \(\tilde{\mathcal{O}}(\sqrt{T})\) effective Pareto regret~(\cref{def:effective_PR}), where $T$ is the total number of rounds. 
    Crucially, our proposed algorithm does not require constructing empirical Pareto fronts, significantly reducing computational overhead compared to many existing algorithms.
    
    \item 
    We introduce the notion of \emph{effective Pareto fairness} (\cref{def:EPF}) as a criterion for evaluating multi-objective bandit algorithms and prove that our algorithm satisfies effective Pareto fairness. 
    To our knowledge, this is the first theoretical analysis of fairness in multi-objective bandit problems.
  
    \item Through extensive numerical experiments, we empirically validate our theoretical claims by demonstrating that $\mogro$ consistently outperforms existing multi-objective methods across a wide range of scenarios.
\end{itemize}

\subsection{Related work}
\paragraph{Multi-objective bandit.} The multi-objective bandit problem, an extension of the single-objective bandit framework incorporating multiple (potentially conflicting) objectives, was first introduced by \citet{Drugan2013}. They proposed two approaches using the UCB algorithm: one based on Pareto optimality and the other on scalarization. While the scalarization approach simplifies the problem by reducing it to a single-objective one~\citep{Drugan2013, yahyaa2015thompson}, the Pareto optimality approach treats all objectives equally, without making any assumptions about their interrelationships. This second approach inspired numerous studies on multi-objective bandits focused on Pareto efficiency \citep{Turgay2018, Tekin2018, Lu2019, Xu2022, Kim2023, Cheng2024, Crepon2024, park2025thompson}. Recent advancements have extended the multi-objective bandit framework to linear (contextual) settings. \citet{Lu2019} established theoretical regret bounds for the UCB algorithm within the generalized linear bandit framework. \citet{Kim2023} explored Pareto front identification in linear bandit settings. Additionally, \citet{Cheng2024} introduced two algorithms stochastic linear bandits under a hierarchy-based Pareto dominance condition. Recently, \citet{park2025thompson} introduced effective Pareto optimality, an optimality notion defined from a cumulative reward perspective, and established theoretical regret bounds for Thompson sampling algorithms in the multi-objective setting.
While these works made important strides, they largely overlook the potential for implicit exploration that can arise in multi-objective bandit setting.

\paragraph{Free exploration.} Recent research on single-objective linear contextual bandits with stochastic contexts has shown that when context diversity is sufficiently high, greedy algorithms can achieve $\tilde{\Ocal}(\sqrt{T})$ regret bounds~\citep{Bastani2020, Kannan2018, Raghavan2018, kim2025LAC}. However, the extension of these results to multi-objective bandits has been limited by a diversity assumption on the context distribution, leaving a gap in understanding how exploration can occur without the this assumption. Our work addresses this gap by focusing on implicit exploration driven by multiple good arms for multiple objectives, even in the absence of context stochasticity.  
While \citet{Bayati2020} demonstrated that greedy algorithms perform well in non-contextual single-objective settings when the number of arms is large, they relied on a $\beta$-regularity assumption related to the reward distribution. In contrast, we introduce the concept of $\rglBound$-\textit{goodness} (Assumption~\ref{assump:Good}), which generalize the notion of $\beta$-regularity to feature spaces in the multi-objective setting. 
Unlike \citet{Bayati2020}, which provided only Bayesian regret bounds---a weaker notion of regret than the frequentest regret, 
we rigorously establish the frequentest regret bounds for our proposed algorithm, $\mogro$, under this goodness assumption and for multi-objective problem settings. 
This is the first work to provide theoretical guarantees for algorithms that leverage implicit exploration in multi-objective linear bandits, without relying on diversity assumptions on the context distribution—a significant departure from the existing literature.

\section{Problem settings}
\label{sec:PS}
\subsection{Notations}
We denote by $[n]:=\{1, \ldots, n\}$ for $n \in \NN$. 
For a vector $x \in \RR^d$, we use $\|x\|_2$ and $\|x\|_A=\sqrt{x^\top A x}$ to denote to denote the $l_2$ norm and the weighted norm of $x$ induced by a positive definite matrix $A \in \RR^{d \times d}$, respectively.
We define the $d$-dimensional ball $\mathbb{B}_R^d=\{x \in \RR^d ~|~ \|x\|_2 \le R \}$ and the $(d-1)$-simplex $\mathbb{S}^{d-1}=\{(x_1,\ldots, x_d) \in \RR^d ~|~ \sum_{d'=1}^d x_{d'}=1 \}$. Finally, $\one_M$ means  $(1,\ldots,1)\in \RR^M$, and $\mathds{1}\{\mathrm{condition}\}$ means the indicator function that takes the value $1$ if the condition is true and $0$ otherwise.

\subsection{Multi-objective linear bandits}
\label{subsec:MOLB}
In each round $t\in[T]$, each feature vector $x_i \in \RR^d$ for  $i \in [K]$ is associated with stochastic reward $y_{i,m}(t)$ for objective $m \in [M]$ with mean $x_i^\top \truethetam$ where $\truethetam \in \mathbb{R}^d$ is a fixed, unknown parameter. 
After the agent pulls an arm $a(t) \in [K]$, the agent receives a stochastic reward vector $y_{a(t)}(t) = \big( y_{a(t),1}(t), \ldots , y_{a(t), M}(t) \big) \in \mathbb{R}^M$ as a bandit feedback, where $y_{a(t),m}(t)=x_{a(t)}^\top\truethetam + \eta_{a(t),m}(t)$ and $\eta_{a(t), m}(t) \in \RR$ is zero-mean noise for objective $m \in [M]$. 
To simplify notation, we denote by $x(t):=x_{a(t)}$ and $y(t):=y_{a(t)}(t)$, 
the selected arm vector in round $t$ and its rewards, respectively, with slight notational overloading. 
We assume that for all $m \in [M]$, $\eta_{a(t), m}(t)$ is conditionally $\sigma^2$-sub-Gaussian for some $\sigma>0$, i.e., for all $\lambda \in \mathbb{R}$, $\mathbb{E}[e^{\lambda \eta_{a(t), m}(t)}| \mathcal{F}_{t-1}] \le \exp\left( \lambda^2\sigma^2/2 \right)$ 
where $\mathcal{F}_{t}$ is the $\sigma$-algebra generated by $\big(\{x(s)\}_{s \in [t+1]}, \{a(s)\}_{s \in [t]}, \{y(s)\}_{s \in [t]}\big)$.

While we present our problem setting in the fixed-feature setup for clarity of exposition—highlighting our main idea of implicit exploration without relying on the context distributional diversity assumption—we also provide results under a varying-context setting in Appendix~\ref{ap_sec:analysis_stochastic}.

\subsubsection{Effective Pareto regret}

In this work, we adopt the notion of effective Pareto regret~\citep{park2025thompson} as the performance metric for multi-objective bandit algorithms, which is a stronger notion than the originally proposed Pareto regret (Definition~\ref{def:PR}) in the multi-objective bandit setting~\citep{Drugan2013}.
Effective Pareto optimality regards as optimal only those arms that are optimal from a cumulative reward perspective, and is therefore more stringent than Pareto optimality.
%
\begin{definition}[Pareto order] 
    For $u=\big(u_1,\ldots,u_M\big),~ v=\big(v_1, \ldots,v_M\big) \in \mathbb{R}^M$, the vector $u$ \textit{dominates} $v$, denoted by $v \prec u$, if and only if $v_m \le u_m$ for all $m \in [M]$, and there exists $m' \in [M]$ such that $v_{m'}<u_{m'}$. Conversely, $v (\neq u)$ is not dominated by $u$, denoted by $v \nprec u$, if there exists $m\in [M]$ such that $v_m >u_m$. 
\end{definition}
\begin{definition} [Effective Pareto front]
    Let $\mu_i \in \RR^M$ be the expected reward vector of arm $i \in [K]$. An arm is effective Pareto optimal (denoted $a_*$) if its mean reward vector is either equal to or not dominated by any convex combination of the mean reward vectors of the other arms. Formally, for any $\beta:= (\beta_i)_{i\in[K]\setminus\{a_*\}}\in\mathbb{S}^{K-2}$, 
\begin{equation*}
\mu_{a_*} = \sum_{i\in [K]\setminus\{a_*\}}\beta_{i}\mu_{i} ~~\text{ or } ~~\mu_{a_*} \not\prec \sum_{i\in[K]\setminus\{a_*\}}\beta_{i}\mu_{i}.
\end{equation*} 
The set of effective Pareto optimal arms is called the effective Pareto front, denoted as $\Ccal^*$. 
\end{definition}
While Pareto optimal arms are defined by those that are not dominated by any other arm (\cref{def:PF}), effective Pareto optimality further accounts for linear combinations of arms and thus represents a stronger notion of optimality. Thus, every effective Pareto optimal arm is Pareto optimal, whereas the converse does not always true. 

\citet{park2025thompson}  introduce this notion to construct an optimal set consisting only of arms that are promising from a cumulative reward perspective. 
For example, consider the case $\mu_1 = (1,0)$, $\mu_2 = (0,1)$, and $\mu_3 = (0.3, 0.3)$. Selecting arm $1$ and arm $2$ once each yields a higher cumulative rewards for both objective than selecting arm $3$ twice. Thus, although arm $3$ is Pareto optimal, it is difficult to regard it as optimal from a cumulative reward perspective.
In this sense, the effective Pareto front treats as optimal only those arms that are not dominated by all combinations of arms.
The following proposition establishes the relationship between effective Pareto optimality and linear scalarization.
\begin{proposition}[Theorem 1 of \citet{park2025thompson}]\label{prop:EPFtoLW} 
    For any $a_*\in\Ccal^*$, there exist $w \in \mathbb{S}^{M-1}$ satisfying $a_* = \arg\max_{i\in[K]} w^\top\mu_i$. Conversely, for any $w \in \mathbb{S}^{M-1}$, if $a_* = \arg\max_{i\in[K]} w^\top\mu_i$ is a unique arm, then $a_*\in\Ccal^*$. 
\end{proposition} 

The proposition states that an arm is effective Pareto optimal if and only if it is the optimal arm for a linearly scalarized objective under some weight vector. This result implies that one can explore the effective Pareto front without explicitly estimating it at each round, by instead optimizing linearly scalarized objectives over a range of weight vectors.

\begin{definition} [Effective Pareto regret]\label{def:effective_PR}
    For each vector $w = (w_{i})_{i\in[K]} \in\mathbb{S}^{M-1}$, let $\mu_{w} := \sum_{i\in[K]}w_{i}\mu_{i}$. Then, the \textbf{effective Pareto suboptimality gap} $\Delta_{i}$ for arm $i \in [K]$ is defined as the infimum of the scalar $\epsilon \ge 0$ such that $\mu_i$ becomes effective Pareto optimal arm after adding $\epsilon$ to all entries of its expected reward. Formally, 
    \begin{equation*}
        \Delta_{i}:= \inf\left\{\epsilon \ge 0 \mathrel{\big|} \mu_{a_t} +\epsilon \one_M \not\prec\mu_{w},\forall w\in\mathbb{S}^{K-1}\right\}.
    \end{equation*}
    Then, the cumulative \textbf{effective Pareto regret} is defined as  $\regret(T):=\sum_{t=1}^T \mathbb{E}[\Delta_{a(t)}]$, where $ \mathbb{E}[\Delta_{a(t)}]$ represents the expected effective Pareto suboptimality gap of the arm pulled at round $t$. 
\end{definition}

The goal of the agent is to minimize the cumulative effective Pareto regret, while simultaneously ensuring fairness across optimal arms, as described in the next section.

\subsubsection{Effective Pareto fairness}
\label{subsubsec:OF}

\paragraph{Beyond effective Pareto regret.} While minimizing effective Pareto regret is a stronger objective than minimizing Pareto regret (\cref{def:PR})—which is often a central goal in multi-objective bandit algorithms—it still does not fully capture the essence of the multi-objective problem. 
Paradoxically, focusing solely on Pareto type regret minimization itself allows algorithms to optimize for a single specific objective, potentially neglecting others.
Specifically, an algorithm that behaves with respect to a single-objective bandit problem may perform just as well in the effective Pareto optimal sense~\citep{Xu2022}, hence defeats the purpose of the multi-objective problem.
Therefore, multi-objective bandit algorithms should aim to balance multiple objectives, typically incorporating additional considerations such as fairness, alongside Pareto regret minimization.

\paragraph {Existing fairness criterion \citep{Drugan2013}.} In multi-objective bandits, how fairly an algorithm handles multiple objectives is considered an important factor. Fairness in multi-objective bandits was first introduced by \citet{Drugan2013}, who defined it as how evenly the Pareto front is sampled (\cref{def:PFF}). However, this definition requires tracking the selection frequency of each true Pareto optimal arm, making it unsuitable for theoretical analysis. 
Many previous studies have mentioned fairness in the selection process, but, to the best of our knowledge, none has provided a theoretical analysis of fairness~\citep{yahyaa2015thompson, Turgay2018, Lu2019}.

Furthermore, in practice, this fairness principle requires multi-objective algorithms to compute the empirical Pareto front at each arm selection, resulting in significant computational overhead~\citep{Drugan2013, yahyaa2015thompson, Turgay2018, Lu2019, park2025thompson}. Specifically, algorithms that construct the empirical Pareto front in each round incur a time complexity of $\Ocal(K^2)$ per round. This indicates that such algorithms may encounter scalability challenges in real-world applications involving a significantly large arm set.

\paragraph{Effective Pareto fairness.} To address these limitations, we propose a new notion of fairness in multi-objective bandit problems, which provides theoretical guarantees without imposing additional computational overhead on the algorithms. The new fairness criterion ensures that a multi-objective algorithm consistently selects each of the (near) effective Pareto optimal arms.

\begin{definition} [Effective Pareto fairness] \label{def:EPF}
     Given a vector $w = (w_{i})_{i\in[K]} \in\mathbb{S}^{M-1}$, let $\mu_{i, w}:=\sum_{m\in [M]}w_mx_i^\top\theta_m^*$ be the expect weighted reward of arm $i$, $a_w^*$ be the arm that has the largest expected weighted reward with respect to $w$, and $\mu_w^*:= \mu_{a_w^*, w}$. For all $\epsilon>0$, define \textbf{the effective Pareto fairness index} $\textnormal{EPFI}_{\epsilon,T}$ of an algorithm as
    \begin{equation*}
        \textnormal{EPFI}_{\epsilon,T}:=\!\inf_{ w \in \mathbb{S}^{M-1}}\!\left({1 \over T}\mathbb{E}\left[\sum_{t=1}^T \mathds{1}\{\mu_{w}^*-\mu_{a(t), w}<\epsilon\}\right]\right).
    \end{equation*}   
    Then, we say that the algorithm satisfies the \textbf{effective Pareto  fairness} if for given $\epsilon$, there exists a positive lower bound $L_\epsilon$ that satisfies $\lim_{T \rightarrow \infty}\textnormal{EPFI}_{\epsilon,T} \ge L_\epsilon$. 
\end{definition}

Intuitively, the effective Pareto fairness criterion guarantees that the algorithm consistently selects near-optimal arms across all weighted objectives $\sum_{m \in[M]}w_m\truethetam$.
The effective Pareto fairness index provides a lower bound on the proportion of rounds in which $\epsilon$-optimal arms are selected for an arbitrary weighted objective.
In light of Proposition~\ref{prop:EPFtoLW}, this implies that the algorithm selects arms from the (near) effective Pareto front, each with positive probability.
Therefore, effective Pareto fairness is an asymptotic notion that ensures the consistent selection of (near) effective Pareto optimal arms, without neglecting any of them.
Conversely, if $\lim_{T \rightarrow \infty} \text{EPFI}_{\epsilon,T} \rightarrow 0$, this implies that the algorithm eventually fails to consider optimal arms for at least one weighted objective, and thus neglects at least one effective Pareto optimal arm.

\paragraph{Pareto front approximation.} We argue that algorithms pursuing effective Pareto front fairness can address many challenges in real-world problems more efficiently than traditional approaches. A major advantage of these algorithms is that they eliminate the need to construct the empirical (effective) Pareto front at each iteration, which is typically required by traditional methods to ensure fairness. 
Although these algorithms do not approximate the Pareto front at every round, they are able to explore the entire effective Pareto front efficiently, and allow for on-demand approximation of the whole effective Pareto front by estimating the parameters of each objective. 
A detailed discussion on the improvement over \citet{Drugan2013}'s fairness notion is provided in Appendix~\ref{ap_sec:fairness_comparison}.


\section{Proposed algorithm}
\label{sec:MOG}
\subsection{Multi-Objective -- Greedy with Randomized Objective ($\mogro$) algorithm }
We propose a simple and generic algorithm named the $\mogro$ algorithm. Initially, the algorithm explores feature space until the minimum eigenvalue of the Gram matrix $V_{t-1}=\sum_{s=1}^{t-1} x(s)x(s)^\top$ exceeds a certain threshold $B$ (Line 5). After the initial exploration phase, at each round $t$, the algorithm updates the OLS estimators  $\hat\theta_m(t) := V_t^{-1}\sum_{s=1}^{t-1}x(s)y_{a(s),m}(s)$, and then greedily selects arms with respect to a randomized objective (Line 7-8). Below, we provide detailed descriptions of each phase.

\begin{algorithm}[t]
\caption{Multi-Objective -- Greedy with Randomized Objective algorithm (\mogro)}
\label{alg:mogro}
\begin{algorithmic}[1]
    \STATE \textbf{Input:} Total rounds $T$, Eigenvalue threshold $B$
    \STATE \textbf{Initialization:} $V_0 \leftarrow 0 \times I_d$, $S:$ Feature basis
    \FOR{$t = 1, \ldots, T$}
        \IF{$\lambdamin(V_{t-1}) < B$} 
            \STATE Select action $a(t) \in S$ in a round-robin manner
        \ELSE 
            \STATE Update the estimators $\Theta_t=\big(\hat{\theta}_1(t), \ldots ,\hat{\theta}_M(t)\big)$
            \STATE Select action $a(t) \leftarrow \text{GRO}(\{x_i\}_{i\in[K]}, \Theta_t)$ 
        \ENDIF
    \STATE Observe $y(t)=\big(y_{a(t),1}(t),\ldots,y_{a(t),M}(t)\big)$
    \STATE Update $V_t \leftarrow V_{t-1}+x(t)x(t)^\top$      
\ENDFOR
\end{algorithmic}
\end{algorithm}

\paragraph{Greedy with Randomized Objective (GRO).} The $\mogro$ algorithm is generic, allowing for various greedy strategies with randomized objectives to be employed in Line~8. For instance, to select arms from the entire effective Pareto front, one can leverage random weight scalarization.

\begin{algorithm}[h]
\caption{GRO with Random Weight Scalarization}
\label{alg:grorw}
\begin{algorithmic}[1]
    \STATE \textbf{Input:} Weight distribution $\Wcal$
    \STATE Randomly sample $w(t)=(w_1(t), \ldots, w_M(t)) \sim \Wcal$
    \STATE \textbf{Return } $\argmax_{i \in [K]} \sum_{m\in[M]} w_m (t)x_i ^\top \hat{\theta}_m (t)$
\end{algorithmic}
\end{algorithm}

The sampling distribution $\Wcal$ of Algorithm~\ref{alg:grorw} can be chosen in various ways depending on the relative importance of each objective or the desired selection bias among the effective Pareto front.
We analyze $\mogro$ with random weight scalarization ($\mogrorw$) in Section~\ref{sec:Analysis}.

It is worth noting that the randomness in the objective is not the source of implicit exploration, rather, an efficient way select a diverse set of near-optimal arms. Indeed, our algorithm performs well even when arms are greedily selected with respect to a deterministic objective. In the extreme case, a fully deterministic algorithm named $\mogrorr$ (Algorithm~\ref{alg:mogrorr}) that greedily selects an arm for each objective in a round-robin manner, achieves the same order of regret bound. Detailed descriptions and analysis of $\mogrorr$ can be found in Appendix~\ref{ap_sec:MOGRO_rr}.

Furthermore, while we focus on algorithms that pursue effective Pareto optimality, one may consider random hypervolume scalarization~\citep{golovin2020, zhang2024} when the goal is to explore the entire Pareto front. 

\paragraph{Initial Exploration. } 

In the fixed-feature setting, the algorithm initializes a set $S$ of features that spans the feature space and then selects arms from $S$ in a round-robin manner during the initial exploration phase. In the varying-context setting, the algorithm selects one of the arms uniformly at random during this phase. It is also important to note that the termination condition in Line~4 is used without loss of generality. When the feature vectors do not span $\mathbb{R}^d$, the $\mogro$ algorithm can instead be implemented using the more general formulation (Algorithm~\ref{alg:mogro-s}).

\begin{remark} 
Our algorithm is fundamentally different from ETC or forced-sampling algorithms. Such algorithms must collect sufficient information about the true parameters during a dedicated exploration phase. In contrast, $\mogro$ uses the initial exploration phase to ensure sufficient diversity among the estimators of objective parameters; thereafter, implicit exploration drives the estimators to continue converging toward the true parameter in all directions (Section~\ref{subsec:FreeExploration}), even beyond the initial exploration phase. Consequently, $\mogro$ requires only $\tilde{\Ocal}(\log T)$ exploration rounds (Lemma~\ref{lem:ini_rounds}), whereas ETC or forced-sampling algorithms would generally fail to achieve sublinear regret with such a limited amount of exploration.
\end{remark}

\paragraph{Practicality of Algorithm 1. } Most existing algorithms for Pareto efficiency construct the Pareto front at each round, leading to complex algorithmic structures and limited practicality.
Instead of explicitly estimating the Pareto front, we advocate employing alternative randomized objectives—such as random weight scalarization—within the GRO framework.
This design renders our algorithm easy to implement and significantly reduces computational overhead, making it well suited for real-world multi-objective problems.
Beyond these practical advantages, we show—perhaps surprisingly—that our simple algorithms can also achieve strong theoretical performance guarantees, which are typically attained only by more complex methods, when good arms exist for each objective.

\subsection{Implicit exploration induced by multiple good arms}
\label{subsec:FreeExploration}
The $\mogro$ algorithm (Algorithm~\ref{alg:mogro}) is built on the insight that exploration can arise naturally, even when the algorithm is focused solely on exploitation, as long as the multi-objective bandit problem has many good arms. In most existing multi-objective bandit studies, as the number of objectives increases, the problem setup becomes more complex, leading to more sophisticated algorithms, particularly in comparison to single-objective bandits.

However, we observe a surprising and beneficial effect: as the number of objectives increases, unlike the single-objective case, the multiple directions of good arms can implicitly induce exploration. This allows simple mostly exploration-free algorithms like $\mogro$ to achieve statistical efficiency (see Theorem~\ref{thm:EPR}).


This phenomenon is intuitive, yet it has not been rigorously examined in multi-objective settings so far. Our work is the first to formalize the conditions under which implicit exploration can occur in the presence of multiple good arms for multiple objectives, paving the way for simpler and more efficient algorithms for multi-objective bandit problems.


\section{Analysis}
\label{sec:Analysis}
In this section, we analyze the algorithm $\mogrorw$ from the perspective of effective Pareto regret and effective Pareto fairness. Our analysis is established in the fixed feature setup to expose our main idea clearly, however, we also present similar results in a stochastic environment in Appendix~\ref{ap_sec:analysis_stochastic}.
 We start with a boundedness assumption similar to those used in the linear bandit literature~\citep{Abbasi-Yadkori2011, Chu2011, agrawal2013thompson, abeille2017Linear, LihongLi2017}.

\begin{assumption} [Boundedness] \label{assump:Bdd}
    For all $i \in [K]$ and $m \in [M]$,  $\|x_i\|_2 \le 1$ and $\|\truethetam\|_2 = 1$. 
\end{assumption}
Assumption~\ref{assump:Bdd} is used to make a clean analysis for convenience and the first part of it  is in fact standard in bandit literature.
Notably, we can obtain a regret bound of the proposed algorithm that differs by at most a constant factor under the conditions $\|x_i\|_2 \le x_{\max}$ and $l \le \|\truethetam \|_2\le L$ for all $i \in [K]$ and $m \in [M]$. Our analysis focuses on the insights that multiple objectives may enhance learning under certain regularity (e.g. goodness of arms in Definition~\ref{def:goodness}) rather than always posing hindrance.
In light of these insights,
the lower bound $l$ represents the minimum contribution of each objective. We will later discuss how to extend our analysis to arbitrary bounds for feature vectors and objective parameters in Appendix~\ref{ap_sec:releasing_Bdd}. 

As stated earlier in the Introduction and Section~\ref{subsec:FreeExploration}, we are interested in the problem setting where there exist good arms for multiple objectives. We start with a simple condition that there are enough objectives to span the feature space without loss of generality.
\begin{assumption} \label{assump:OD}
    We assume $\theta_1^*$, \ldots, $\theta_M^*$ span $\RR^d$.
\end{assumption}
It is important to note that Assumption~\ref{assump:OD} is used without loss of generality. 
We can actually relax Assumption~\ref{assump:OD} so that 
 $\theta_1^*, \ldots, \theta_M^*$ span the space of feature vectors, $span(\{x_1,\ldots,x_K\})$ (see details in Appendix~\ref{ap_sec:releasing_OD}). 
 That is, it can be sufficient to assume that  $\theta_1^*$, \ldots, $\theta_M^*$ span a strict subspace of $\RR^d$ if the feature vectors span such a subspace.
 Yet, for clear exposition of our main idea, we work with Assumption~\ref{assump:OD} and define $\lambda:=\lambdamin({1 \over M}\sum_{m=1}^M{\truethetam(\truethetam)^\top})$, which is positive under Assumption~\ref{assump:OD}.

Next, we introduce the $\rglBound$-goodness condition of arms with feature vectors in multi-objective linear bandits. In brief, this condition ensures the presence of good arms in every objective direction.
\begin{definition}[Goodness of arms]\label{def:goodness}
    For fixed $\rglBound \in (0,1]$, we say that the feature vectors of the arms $\{x_1, \ldots, x_K\}$ satisfy $\rglBound$\textit{-goodness condition} if there exists $\rglBall >0 $ such that 
    \begin{gather*}
        \text{for all }
        \beta \in \mathbb{B}_{\rglBall}^d(\theta_1^*) \cup \ldots \cup \mathbb{B}_{\rglBall}(\theta_M^*), \\
        \text{ there exists }
        k \in [K]  \text{ such that }x_k^\top{\beta \over \|\beta\|_2} \ge \rglBound, 
    \end{gather*}
    and denote such $x_{k}$ as the $\rglBound$\textit{-good arm} for direction $\beta$.
\end{definition}
\begin{assumption} [Goodness]\label{assump:Good}
    The feature vectors $\{x_1, \ldots, x_K\}$ satisfy $\rglBound$-goodness for some $\rglBound \geq 1-{\lambda^2 \over 18}$.
\end{assumption}
Assumption~\ref{assump:Good} states that there exists at least one $\rglBound$-good arm for directions in the neighborhoods of the objective parameters. Let $\rglBall$ denote the value that satisfies the goodness condition defined in Definition~\ref{def:goodness}, together with 
$\rglBound$ as specified in Assumption~\ref{assump:Good}. If $\rglBall$ is greater than $\psi(\lambda, \rglBound):=\sqrt{{\lambda^2 \over 9}-{\lambda^4 \over 324}}~\rglBound-\left( 1-{\lambda^2 \over 18} \right)\sqrt{1-\rglBound^2}$, we replace the value of $\rglBall$ with $\psi(\lambda, \rglBound)$. (Since a larger $\rglBall$ imposes a stricter goodness condition, the condition remains valid even when $\rglBall$ is reduced. For detailed reason for this exchanging trick is presented in Remark~\ref{rmk:exchange_alpha}.) 

\paragraph{Practical implication of arm goodness.} The $\rglBound$-goodness condition often arises in real-world applications where each objective has at least one arm that performs reasonably well. For example, in a personalized recommendation system optimizing multiple metrics such as click-through rates, watch time, and user satisfaction, it is plausible to assume that there exists at least one item among many that delivers high click-through rates, another (possibly different from the first one) that increases watch time, and so on. Consequently, the existence of ``good arms'' across different objective directions naturally aligns with many practical scenarios, reinforcing the applicability of our theoretical findings.

\begin{remark}
The notion of $\rglBound$-goodness is related to the concept of $\beta$-regularity introduced by \citet{Bayati2020} in the non-contextual multi-armed bandit framework. Specifically, they assume that the prior distribution $\Gamma$ for each arm's expected reward $\mu$ satisfies $\mathbb{P}_{\mu}[\mu > 1 - \epsilon] = \Theta(\epsilon^{\beta})$ for every $\epsilon > 0$. Our $\rglBound$-goodness generalizes this idea to linear reward bandit problems with multiple objectives. Comparing Assumption~\ref{assump:Good_stochastic} under $d = M = 1$ with $\beta$-regularity in the stochastic context setting shows that $\rglBound$-goodness is a \emph{weaker} (and thus more general) condition than $\beta$-regularity. Detailed discussion is provided in Appendix~\ref{ap_subsec:goodness_beta}.
\end{remark}

\begin{remark}
    It is worthy noting that the above assumptions are irrelevant to the diversity assumption on context distribution which is commonly used in the existing greedy bandit literature~\citep{Kannan2018, Raghavan2018, Hao2020, Bastani2020} (see Appendix~\ref{ap_subsec:goodness_CD}). 
\end{remark}

Before we start analysis, we define two regularity indices of a weight distribution $\Wcal$. 

\begin{definition}[Regularity indices of a distribution]
Let $\Wcal$ be a distribution on $\mathbb{S}^{M-1}$ and $\theta_w^*:=\sum_{m\in [M]}w_m\theta_m^*$ be the weighted objective for $w=(w_1,\ldots,w_M) \in \mathbb{S}^{M-1}$. For given $\epsilon >0$, We define the two regularity indices of distribution $\Wcal$, $\phi_{\epsilon,\Wcal}$ and $\psi_{\epsilon,\Wcal}$ as
    \begin{align*}
        &\phi_{\epsilon,\Wcal}:=\min_{m\in [M]}\mathbb{P}_{w\sim \Wcal}\left(\left\|\theta_w^*-\theta_{m}^*\right\|_2 < \epsilon\right) \\
        &\psi_{\epsilon,\Wcal}:=\inf_{w'\in \mathbb{S}^{M-1}}\mathbb{P}_{w\sim \Wcal}\left(\left\|\theta_w^*-\theta_{w'}^*\right\|_2 < \epsilon\right).        
    \end{align*}
\end{definition}

Intuitively, both indices explains how evenly the weight distribution generates weighted objectives. Specifically, $\phi_{\epsilon,\Wcal}$ measures whether the weighted objectives are well-sampled near each individual objectives, while $\psi_{\epsilon,\Wcal}$ captures how uniformly all possible weighted objectives are sampled. If $w$ is drawn from a continuous distribution whose density is positive on its support, both indices are positive (Lemma~\ref{lem:rgl>0}).

\subsection{Effective Pareto regret bound of $\mogrorw$}
We establish the lower bound of the minimum eigenvalue on the Gram matrix that grows linearly with respect to $t$. Specifically, we make a constant lower bound for $\lambdamin\left(\mathbb{E}[x(t)x(t)^{\top}|\mathcal{H}_{t-1}]\right)$ in each round, like existing greedy bandit approaches. However, it is important to note that the expectation of the lemma below arises not from the randomness of the contexts, but rather from the randomness associated with the selection of the objective in each round. Let $T_0$ denote the number of initial rounds required until 
$\lambdamin(V_{t-1}) \ge B$ holds. 

\begin{lemma}[Increment of the minimum eigenvalue of the Gram matrix]\label{lem:mineigen_oneround} 
    Suppose Assumptions ~\ref{assump:Bdd}, ~\ref{assump:OD}, and ~\ref{assump:Good} hold. If the OLS estimator satisfies $\|\hatthetams-\truethetam\| \le {\rglBall \over 2}$, for all $m \in [M]$ and $s \ge T_0+1$, then the arm selected by $\mogrorw$ (Algorithm~\ref{alg:mogro}) satisfies
    \begin{align*}
       &\lambda_{\min}(\mathbb{E}[x(s)x(s)^\top| \mathcal{H}_{s-1}]) \ge \lambdainc,
    \end{align*}
    where $\lambdainc:=  \big(\lambda- 2 \sqrt{2+ 2\rglBall \sqrt{1-\rglBound^2}-2\rglBound \sqrt{1-\rglBall^2}}\big)$ $\times\phi_{\rglBall/2, \Wcal} M$.
\end{lemma}

The proof of the lemma is provided in Appendix~\ref{ap_subsec:pf_lem_mineigen_oneround}.

We denote the lower bound on the per-round expected minimum eigenvalue increase by $\lambdainc$. Then, the following theorem establishes the effective Pareto regret of $\mogrorw$.

\begin{theorem}[Effective Pareto regret of $\mogrorw$]\label{thm:EPR} 
    Suppose Assumptions ~\ref{assump:Bdd}, ~\ref{assump:OD}, and ~\ref{assump:Good} hold. If we run $\mogrorw$ (Algorithm~\ref{alg:mogro}) with $B= \min\big[{2\sigma \over \rglBall  }\sqrt{2 dT  \log ({dT^2})},$ $~~{16\sigma^2  \over \rglBall^2} \big( {d \over 2} \log\left(1+{2T \over d}\right)$ $+\log\left({T}\right)\big)\big]$, then the effective Pareto regret of $\mogrorw$ is bounded by 
    \begin{equation*}
        \regret(T) \le {16 \sigma \over \lambdainc}{\sqrt{ 2dT \log (dT)}}+4T_0+6M+{20d \over \lambdainc}.
    \end{equation*}
\end{theorem}

The proof of the theorem is provided in Appendix~\ref{ap_subsec:pf_thm_EPR}.

\paragraph{Discussion of Theorem~\ref{thm:EPR}.} The theorem demonstrates that the cumulative effective Pareto regret bound of $\mogrorw$ is $\tilde{\Ocal}(\sqrt{T})$. To the best of our knowledge, our study is the first to prove the frequentiest regret bound of a algorithm using implicit exploration in the linear reward setting without relying on context stochasticity. Notably, this bound does not include the term dependent on
$K$, and our algorithm performs well even when the number of arms is infinite. 
Theorem~\ref{thm:EPR} provides the theoretical foundation that if there are many good arms, simple near-greedy algorithms can outperform even more complicated exploration-based algorithms for multi-objective linear bandits (see Section~\ref{sec:exp}). 

It is important to note that a fully deterministic algorithm $\mogrorr$ achieves a similar regret bound. This results shows that the randomness in the objective is not the source of implicit exploration.

\begin{corollary}[Effective Pareto regret under greedy with round-robin objective]\label{cor:EPR_rr} 
    Suppose Assumptions ~\ref{assump:Bdd}, ~\ref{assump:OD}, and ~\ref{assump:Good} hold. If we run $\mogrorr$ (Algorithm~\ref{alg:mogrorr}) with $B= \tilde{\Ocal}(\min(\sqrt{dT}, d\log(T)))$, then the effective Pareto regret of $\mogrorr$ is bounded by 
    \begin{equation*}
        \regret(T) \le {8 \sigma \over \lambda'}{\sqrt{ 2dT \log (dT)}}+4T_0+10M,
    \end{equation*}
    where $\lambda':= \lambda- 2 \sqrt{2+ 2\rglBall \sqrt{1-\rglBound^2}-2\rglBound \sqrt{1-\rglBall^2}}$.
\end{corollary}

The proof of the corollary is provided in Appendix~\ref{ap_sec:pf_cor_EPR_rr}.

Next, we show how quickly initial exploration can be completed. The next lemma implies that $T_0$ in the bound stated in Theorem~\ref{thm:EPR} is of the order $\log{T}$.
\begin{lemma}[Number of initial rounds]\label{lem:ini_rounds}
    Suppose that Assumptions ~\ref{assump:Bdd}, ~\ref{assump:OD}, and ~\ref{assump:Good} hold. Then, the number of the exploration rounds $T_0$ of $\mogrorw$ (Algorithm~\ref{alg:mogro}) can be bounded by $T_0 \le  \left\lfloor {B / \lambdamin\left(\sum_{i \in S } x_i(x_i)^\top\right) }\right\rfloor  \times |S| $.
\end{lemma}
The proof of the lemma is given in Appendix~\ref{ap_subsec:initial_rounds}.

\begin{figure*}[ht]
    \centering
    \includegraphics[width=\linewidth, trim=0cm 17cm 0cm 0cm, clip]{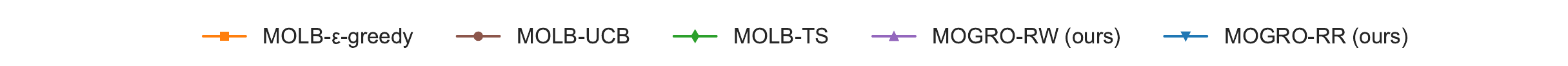} 
    \vspace{-10pt}  
    \begin{subfigure}{0.32\textwidth}
        \centering
        \includegraphics[width=\textwidth, trim=0cm 0cm 0cm 1cm, clip]{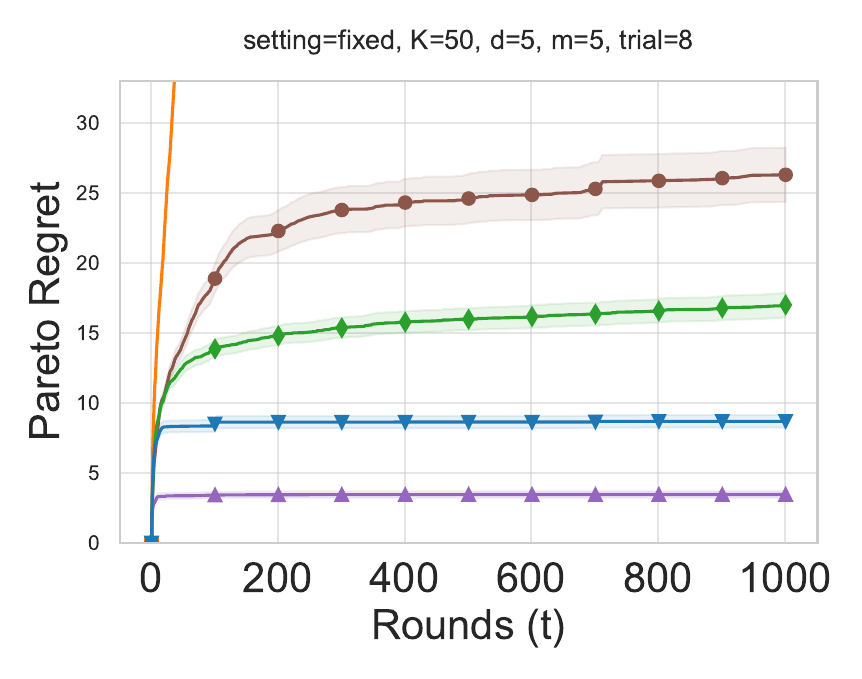}
        \vspace{-20pt}
        \caption{Cumulative Pareto regret}
        \label{fig:exp_main_d5_PR}
    \end{subfigure}
    \hfill 
    \begin{subfigure}{0.32\textwidth}
        \centering
        \includegraphics[width=\textwidth, trim=0cm 0cm 0cm 1cm, clip]{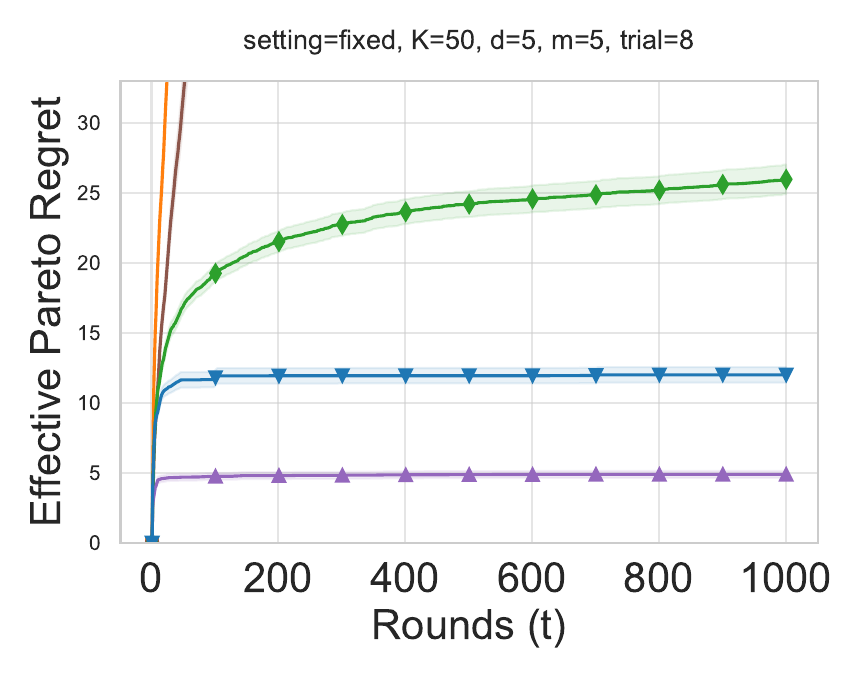}
        \vspace{-20pt}
        \caption{Cumulative effective Pareto regret}
        \label{fig:exp_main_d5_EPR}
    \end{subfigure}
    \hfill
    \begin{subfigure}{0.32\textwidth}
        \centering
        \includegraphics[width=\textwidth, trim=0cm 0cm 0cm 1cm, clip]{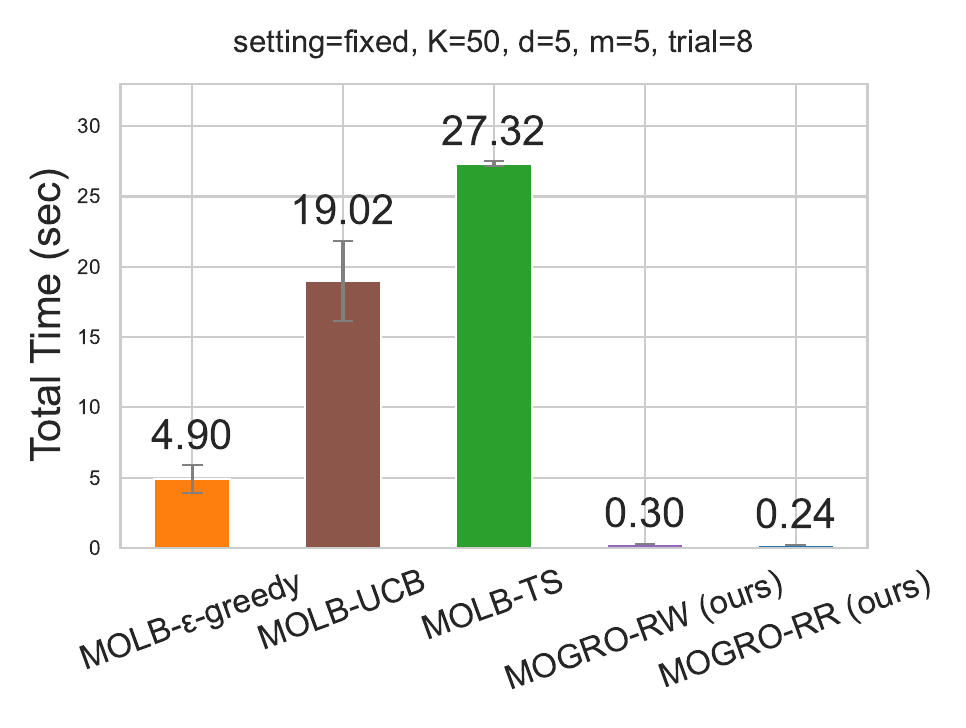}
        \vspace{-20pt}
        \caption{Total running time}
        \label{fig:exp_main_d5_time}
    \end{subfigure}
    \vspace{+10pt}

    \begin{subfigure}{0.32\textwidth}
        \centering
        \includegraphics[width=\textwidth, trim=0cm 0cm 0cm 1cm, clip]{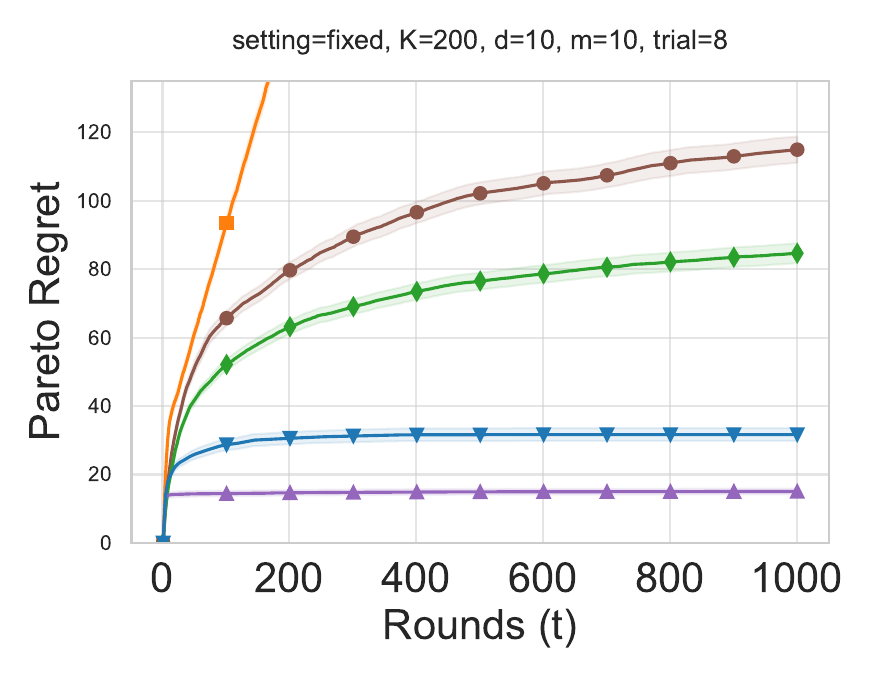}
        \vspace{-20pt}
        \caption{Cumulative Pareto regret}
        \label{fig:exp_main_d10_PR}
    \end{subfigure}
    \hfill
    \begin{subfigure}{0.32\textwidth}
        \centering
        \includegraphics[width=\textwidth, trim=0cm 0cm 0cm 1cm, clip]{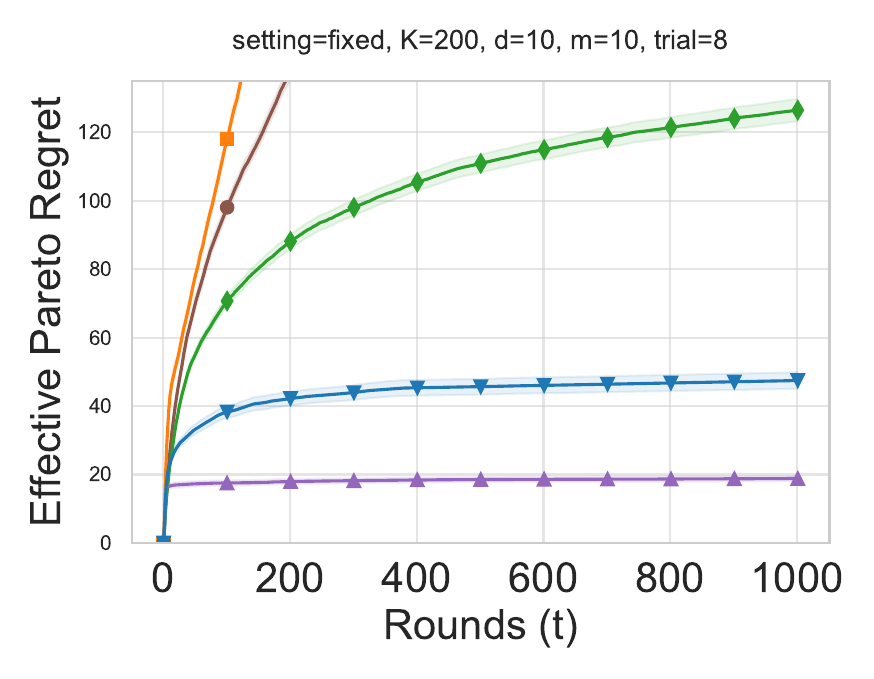}
        \vspace{-20pt}
        \caption{Cumulative effective Pareto regret}
        \label{fig:exp_main_d10_EPR}
    \end{subfigure}
    \hfill
    \begin{subfigure}{0.32\textwidth}
        \centering
        \includegraphics[width=\textwidth, trim=0cm 0cm 0cm 1cm, clip]{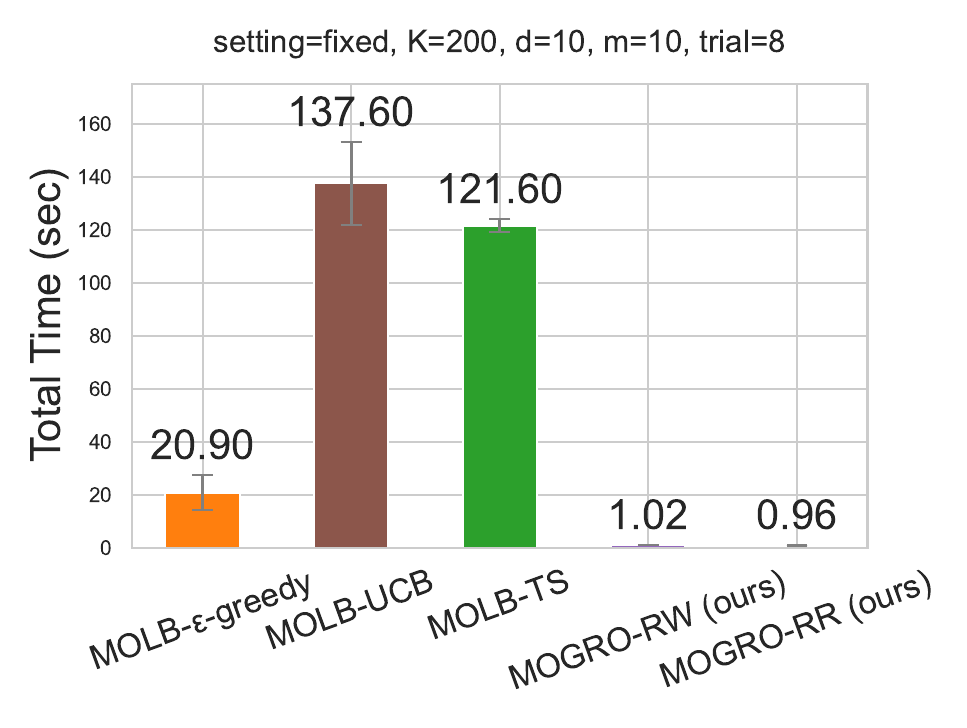}
        \vspace{-20pt}
        \caption{Total running time}
        \label{fig:exp_main_d10_time}
    \end{subfigure}
    \vspace{0pt}
    \caption{Evaluation of multi-objective bandit algorithms in the fixed-feature setting. The plots in the left two columns report the performance of the algorithms, while the plots in the rightmost column report the running time. The top row shows results for $d=5$, $K=50$, $M=5$, and the bottom row shows results for $d=10$, $K=200$, $M=10$.}
    \label{fig:exp_main}
    \vspace{-10pt}
\end{figure*}

\subsection{Effective Pareto fairness of $\mogrorw$}

We have confirmed that the $\mogrorw$ algorithm satisfies effective Pareto fairness, if the weight distribution has positive value of the regularity index  $\psi_{\epsilon, \Wcal}$. 
The algorithm satisfies that for a given $\epsilon>0$,  there exists $T_\epsilon$ such that, after round $T_\epsilon$, only $\epsilon$-optimal arms are selected with high probability. 
The following theorem establishes a lower bound on the effective Pareto fairness index.
\begin{theorem}[Effective Pareto fairness of $\mogrorw$] \label{thm:EPF}
    Suppose Assumptions ~\ref{assump:Bdd}, ~\ref{assump:OD}, and ~\ref{assump:Good} hold. Then, for any $\epsilon>0$, in the same setting as Theorem~\ref{thm:EPR}, the effective Pareto fairness index of $\mogrorw$ (Algorithm~\ref{alg:mogro}) is lower bounded by
    \begin{equation*}
        \textnormal{EPFI}_{\epsilon,T} \ge \!
       \psi_{\epsilon/3,\Wcal}\!\left({T\!-\!T_\epsilon \over T}\right) \!\left(1\!-\!{3M \over T}\!-\!{d \over (dT)^{120\sigma^2 d \over \lambdainc \epsilon^2}}\!\right),
    \end{equation*}    
    where $T_\epsilon=\max(\lfloor{1152\sigma^2 d\log(dT) \over \lambdainc^2\epsilon^2}\rfloor+T_0,~~2T_0)$.
\end{theorem}
The proof of the theorem is provided in Appendix~\ref{ap_subsec:pf_thm_EPF}.

\paragraph{Discussion of Theorem~\ref{thm:EPF}.} The theorem shows that $\mogrorw$ satisfies effective Pareto fairness, if the weight distribution satisfies $\psi_{\epsilon/3, \Wcal}>0$ for all $\epsilon>0$. Notably, for any given $\epsilon>0$, $\lim_{T \rightarrow \infty}\text{EPFI}_{\epsilon,T} \ge \psi_{\epsilon/3, \Wcal}$ and the limit does not include a term with $K$. To our knowledge, this is the first theoretical analysis of fairness of multi-objective bandit algorithms.


\section{Experiment}
\label{sec:exp}

We evaluate our proposed algorithms, $\mogrorr$  and $\mogrorr$, comparing them with basic multi-objective algorithms : $\epsilon$-Greedy algorithm (\texttt{MOLB-$\epsilon$-Greedy}), the Upper Confidence Bound algorithm (\texttt{MOLB-UCB}; linear version of algorithm in \citet{Lu2019}), and the Thompson Sampling algorithm (\texttt{MOL-TS} of \citet{park2025thompson}).  

The results in Figure~\ref{fig:exp_main} clearly demonstrate that our proposed algorithms empirically outperform the others. Both $\mogrorw$ and $\mogrorr$ achieve the lowest regret in terms of both Pareto and effective Pareto criteria. Moreover, due to their simpler structure, they are significantly faster than the other algorithms. The detailed experimental settings and additional results are provided in Appendix~\ref{ap_sec:exp}. 

\section{Conclusion}

In this work, we introduced $\mogro$, a multi-objective bandit algorithm that selects arms greedily in most rounds. We identified sufficient conditions under which implicit exploration arises, enabling our algorithm to achieve $\tilde{\Ocal}(\sqrt{T})$ effective Pareto regret bounds. We further introduced the notion of effective Pareto fairness, which ensures consistent selection of effective Pareto-optimal arms, and showed that $\mogro$ satisfies this criterion. Overall, our results provide a new perspective: the presence of multiple good arms can enhance learning in multi-objective bandit problems.

\bibliographystyle{plainnat}
\bibliography{references}

@inproceedings{abeille2017Linear,
	title        = {{Linear Thompson Sampling Revisited}},
	author       = {Abeille, Marc and Lazaric, Alessandro},
	year         = 2017,
	month        = {20--22 Apr},
	booktitle    = {Proceedings of the 20th International Conference on Artificial Intelligence and Statistics},
	publisher    = {PMLR},
	series       = {Proceedings of Machine Learning Research},
	volume       = 54,
	pages        = {176--184},
	editor       = {Singh, Aarti and Zhu, Jerry},
	organization = {PMLR}
}

@inproceedings{agrawal2013thompson,
	title        = {Thompson sampling for contextual bandits with linear payoffs},
	author       = {Agrawal, Shipra and Goyal, Navin},
	year         = 2013,
	booktitle    = {International conference on machine learning},
	pages        = {127--135},
	organization = {PMLR}
}

@article{Abbasi-Yadkori2011,
  title={Improved algorithms for linear stochastic bandits},
  author={Abbasi-Yadkori, Yasin and P{\'a}l, D{\'a}vid and Szepesv{\'a}ri, Csaba},
  journal={Advances in neural information processing systems},
  volume={24},
  year={2011}
}

@inproceedings{Chu2011,
  title={Contextual bandits with linear payoff functions},
  author={Chu, Wei and Li, Lihong and Reyzin, Lev and Schapire, Robert},
  booktitle={Proceedings of the Fourteenth International Conference on Artificial Intelligence and Statistics},
  pages={208--214},
  year={2011},
  organization={JMLR Workshop and Conference Proceedings}
}

@inproceedings{Drugan2013,
  title={Designing multi-objective multi-armed bandits algorithms: A study},
  author={Drugan, Madalina M and Nowe, Ann},
  booktitle={The 2013 international joint conference on neural networks (IJCNN)},
  pages={1--8},
  year={2013},
  organization={IEEE}
}

@inproceedings{Kevton2020,
  title={Randomized exploration in generalized linear bandits},
  author={Kveton, Branislav and Zaheer, Manzil and Szepesvari, Csaba and Li, Lihong and Ghavamzadeh, Mohammad and Boutilier, Craig},
  booktitle={International Conference on Artificial Intelligence and Statistics},
  pages={2066--2076},
  year={2020},
  organization={PMLR}
}

@inproceedings{LihongLi2017,
  title={Provably optimal algorithms for generalized linear contextual bandits},
  author={Li, Lihong and Lu, Yu and Zhou, Dengyong},
  booktitle={International Conference on Machine Learning},
  pages={2071--2080},
  year={2017},
  organization={PMLR}
}

@article{Tekin2018,
   title={Multi-objective Contextual Multi-armed Bandit With a Dominant Objective},
   volume={66},
   ISSN={1941-0476},
   number={14},
   journal={IEEE Transactions on Signal Processing},
   publisher={Institute of Electrical and Electronics Engineers (IEEE)},
   author={Tekin, Cem and Turgay, Eralp},
   year={2018}, pages={3799–3813} }

@inproceedings{Turgay2018,
  title={Multi-objective contextual bandit problem with similarity information},
  author={Turgay, Eralp and Oner, Doruk and Tekin, Cem},
  booktitle={International Conference on Artificial Intelligence and Statistics},
  pages={1673--1681},
  year={2018},
  organization={PMLR}
}

@article{Lu2019,
author = {Lu, Shiyin and Wang, Guanghui and Hum Yao and Zhang, Lijun},
journal = {Proceedings of the 28th International Joint Conference on Artificial Intelligence},
pages = {3080-3086},
title = {Multi-objective generalized linear bandits},
volume = {},
year = {2019}
}

@inproceedings{Xu2022,
  title={Pareto regret analyses in multi-objective multi-armed bandit},
  author={Xu, Mengfan and Klabjan, Diego},
  booktitle={International Conference on Machine Learning},
  pages={38499--38517},
  year={2023},
  organization={PMLR}
}

@inproceedings{Cheng2024,
  title={Hierarchize Pareto Dominance in Multi-Objective Stochastic Linear Bandits},
  author={Cheng, Ji and Xue, Bo and Yi, Jiaxiang and Zhang, Qingfu},
  booktitle={Proceedings of the AAAI Conference on Artificial Intelligence},
  volume={38},
  pages={11489--11497},
  year={2024}
}

@InProceedings{Crepon2024,
  title = 	 {Sequential learning of the {P}areto front for multi-objective bandits},
  author =       {{c}repon, \'{e}lise and Garivier, Aur\'{e}lien and M Koolen, Wouter},
  booktitle = 	 {Proceedings of The 27th International Conference on Artificial Intelligence and Statistics},
  pages = 	 {3583--3591},
  year = 	{2024},
  editor = 	 {Dasgupta, Sanjoy and Mandt, Stephan and Li, Yingzhen},
  volume = 	 {238},
}

@article{Kim2023,
  title={Learning the Pareto Front Using Bootstrapped Observation Samples},
  author={Kim, Wonyoung and Iyengar, Garud and Zeevi, Assaf},
  journal={arXiv preprint arXiv:2306.00096},
  year={2023}
}

@inproceedings{yahyaa2015thompson,
  title={Thompson Sampling for Multi-Objective Multi-Armed Bandits Problem.},
  author={Yahyaa, Saba Q and Manderick, Bernard},
  booktitle={ESANN},
  year={2015}
}

@article{Kannan2018,
  title={A smoothed analysis of the greedy algorithm for the linear contextual bandit problem},
  author={Kannan, Sampath and Morgenstern, Jamie H and Roth, Aaron and Waggoner, Bo and Wu, Zhiwei Steven},
  journal={Advances in neural information processing systems},
  volume={31},
  year={2018}
}

@article{Bastani2020,
  title={Mostly exploration-free algorithms for contextual bandits},
  author={Bastani, Hamsa and Bayati, Mohsen and Khosravi, Khashayar},
  journal={Management Science},
  volume={67},
  number={3},
  pages={1329--1349},
  year={2021},
  publisher={INFORMS}
}

@article{kim2025LAC,
    title={Local Anti-Concentration Class: Logarithmic Regret for Greedy Linear Contextual Bandit}, 
    author={Kim, Seok-Jin and Oh, Min-hwan },
    journal={Advances in Neural Information Processing Systems},
    volume={},
    pages={},
    year={2025}
}

@inproceedings{Raghavan2018,
  title={The externalities of exploration and how data diversity helps exploitation},
  author={Raghavan, Manish and Slivkins, Aleksandrs and Wortman, Jennifer Vaughan and Wu, Zhiwei Steven},
  booktitle={Conference on Learning Theory},
  pages={1724--1738},
  year={2018},
  organization={PMLR}
}

@inproceedings{Hao2020,
  title={Adaptive exploration in linear contextual bandit},
  author={Hao, Botao and Lattimore, Tor and Szepesvari, Csaba},
  booktitle={International Conference on Artificial Intelligence and Statistics},
  pages={3536--3545},
  year={2020},
  organization={PMLR}
}

@inproceedings{Bayati2020,
 author = {Bayati, Mohsen and Hamidi, Nima and Johari, Ramesh and Khosravi, Khashayar},
 booktitle = {Advances in Neural Information Processing Systems},
 pages = {1713--1723},
 publisher = {Curran Associates, Inc.},
 title = {Unreasonable Effectiveness of Greedy Algorithms in Multi-Armed Bandit with Many Arms},
 volume = {33},
 year = {2020}
}

@article{Tropp2011,
  title={User-friendly tail bounds for matrix martingales},
  author={Tropp, Joel A},
  journal={ACM Report},
  volume={1},
  year={2011}
}

@misc{golovin2020,
    title={Random Hypervolume Scalarizations for Provable Multi-Objective Black Box Optimization}, 
    author={Daniel Golovin and Qiuyi Zhang},
    booktitle ={Proceedings of The 37th International Conference on Machine Learning, Online},
    pages = 	 {11096 - 11105},
    year = 	 {2020},
}

@article{zhang2024,
      title={Optimal Scalarizations for Sublinear Hypervolume Regret}, 
      author={Qiuyi Zhang},
      year={2024},
    journal={Advances in Neural Information Processing Systems},
    volume={38},
    pages={}
}

@article{Auer2002,
author = {Auer, Peter and Cesa-Bianchi, Nicol\`{o} and Freund, Yoav and Schapire, Robert E.},
title = {The Nonstochastic Multiarmed Bandit Problem},
journal = {SIAM Journal on Computing},
volume = {32},
number = {1},
pages = {48-77},
year = {2002}
}

@InProceedings{Chu11a,
  title = 	 {Contextual Bandits with Linear Payoff Functions},
  author = 	 {Chu, Wei and Li, Lihong and Reyzin, Lev and Schapire, Robert},
  booktitle = 	 {Proceedings of the Fourteenth International Conference on Artificial Intelligence and Statistics},
  pages = 	 {208--214},
  year = 	 {2011},
  editor = 	 {Gordon, Geoffrey and Dunson, David and Dudík, Miroslav},
  volume = 	 {15},
  series = 	 {Proceedings of Machine Learning Research},
  address = 	 {Fort Lauderdale, FL, USA},
  month = 	 {11--13 Apr},
  publisher =    {PMLR},
}

@InProceedings{park2025thompson,
      title={Thompson Sampling for Multi-Objective Linear Contextual Bandit}, 
      author={Park, Somangchan and Ann, Heesang and Oh, Min-hwan },
     booktitle =   {Advances in Neural Information Processing Systems},  
    year={2025},
    volume={39}
}

\clearpage
\appendix
\onecolumn

\renewcommand{\contentsname}{Contents of Appendix}
\addtocontents{toc}{\protect\setcounter{tocdepth}{2}}
{
  \hypersetup{hidelinks}
  \tableofcontents
}

\crefalias{section}{appendix}
\crefalias{subsection}{appendix}
\crefalias{subsubsection}{appendix}
\setcounter{tocdepth}{2}    
\counterwithin{table}{section}
\counterwithin{figure}{section}
\counterwithin{equation}{section}
\counterwithin{algorithm}{section}
\counterwithin{theorem}{section}
\counterwithin{lemma}{section}
\counterwithin{corollary}{section}
\counterwithin{proposition}{section}
\counterwithin{assumption}{section}
\counterwithin{definition}{section}
\counterwithin{remark}{section}
\counterwithin{condition}{section}
\clearpage
\section{Additional notations}
We define the $d$-dimensional ball $\mathbb{B}_R^d=\{x \in \RR^d ~|~ \|x\|_2 \le R \}$ and the $(d-1)$-dimensional simplex $\mathbb{S}_R^{d-1}=\{(x_1, \ldots,x_d) \in \mathbb{R}^d ~|~ x_1+\ldots+x_d = 1 \}$. $R$ can be omitted for simplicity if $R=1$, i.e. $\mathbb{B}^d:=\mathbb{B}_1^d$.
We also denote the positive orthant of $\mathbb{R}^d$ by $\mathbb{R}_+^d$. For matrices $A$ and $B$, we write $A \succeq B$ to indicate that $A - B$ is positive definite. The $i$-th unit vector in $\mathbb{R}^d$ is denoted by $e_i^{(d)}$, and when the dimension $d$ is clear from the context, we simply write $e_i$. We define the spanning space of feature vectors $x_1, \ldots, x_K$ as $S_x$, and its orthogonal complement as $S_x^\perp$. The projection map onto $S_x$ is denoted by $\pi_{S_x}: \mathbb{R}^d \to S_x$, and when the space $S_x$ is clear from the context, we simply write $\pi_s$.

\section{Evaluation metrics for multi-objective bandit algorithms}
\subsection{Pareto regret~\citep{Drugan2013}} 
\label{ap_sec:regret_comparison}
The following definitions refer to the original notions of Pareto optimality and Pareto regret for multi-objective algorithms, as first introduced by \citet{Drugan2013}.

\begin{definition} [Pareto front]\label{def:PF}
    Let $\mu_i \in \RR^M$ be the expected reward vector of arm $i \in [K]$. Then, arm $i$ is \textit{Pareto optimal} if and only if $\mu_i$ is not dominated by $\mu_{i'}$ for all $i' \in [K]$. The \textit{Pareto front} is the set of all Pareto optimal arms.
\end{definition}

\begin{definition} [Pareto regret]\label{def:PR}
     We denote \textbf{Pareto suboptimality gap} $\Delta_{i}$ for arm $i \in [K]$ as the infimum of the scalar $\epsilon \ge 0$ such that $\mu_i$ becomes Pareto optimal arm after adding $\epsilon$ to all entries of its expected reward. Formally, 
    \begin{equation*}
        \Delta_{i}:= \inf \left\{ \epsilon \mid (\mu_i + \epsilon) \nprec \mu_{i'}, \forall i' \in [K] 
    \right\}.
    \end{equation*}
    Then, the cumulative \textbf{Pareto regret} is defined as  $\mathbf{PR}(T):=\sum_{t=1}^T \mathbb{E}[\Delta_{a(t)}]$, where $ \mathbb{E}[\Delta_{a(t)}]$ represents the expected Pareto suboptimality gap of the arm pulled at round $t$. 
\end{definition}

By definition, Pareto optimality is a weaker notion than effective Pareto optimality; consequently, the Pareto regret of an algorithm is always upper bounded by its effective Pareto regret. A detailed discussion of how effective Pareto regret differs from—and improves upon—the original definition of Pareto regret is provided in \citet{park2025thompson}.

\subsection{Fairness~\citep{Drugan2013}} 
\label{ap_sec:fairness_comparison}
The fairness criterion for multi-objective bandit algorithms first introduced by \citet{Drugan2013} is given below, and we refer to it as \textit{Pareto fairness}.

\begin{definition} [Pareto fairness] \label{def:PFF}
     Let $T_i^* (n)$ be the number of rounds an optimal arm $i$ is pulled, and $\mathbb{E}[T^*(n)]$ be the expected number of times optimal arms are selected. The \textit{unfairness} of a multi-objective bandit algorithm is defined as the \textit{variance} of the arms in Pareto front $\Acal^*$, 
     \begin{equation*}
         \phi = {1 \over |\Acal^*|} \sum_{i \in \Acal^*} \big(T_i^* (n) - \mathbb{E}[T^*(n)]\big)^2.
     \end{equation*}
     For a perfectly fair usage of optimal arms, we have that $\phi \rightarrow 0$. 
\end{definition}

The key advantages of effective Pareto fairness (EPF) over Pareto front fairness (PF) are summarized as follows:

\begin{itemize}
    \item EPF guarantees consistent selection of the effective Pareto front, whereas PF focuses on equal selection across the entire Pareto front. Consequently, EPF accommodates differences in the relative importance of Pareto optimal arms, while PF implicitly assumes equal importance among them. These distinctions are reflected in the corresponding fairness indices: EPF uses the lower bound on the selection ratio of each optimal arm, whereas PF employs the variance of the selection frequencies across optimal arms.
    
    \item Statistical analysis is feasible with EPF but not with PF, as PF requires the number of times each true optimal arm is pulled, which can only be computed in simulated studies.  In contrast, the definition of EPF incorporates an $\epsilon$ argument, enabling theoretical analysis. Detailed theoretical analysis of fairness is provided in Appendix~\ref{ap_subsec:pf_thm_EPF}.

    \item Algorithms grounded in the EPF perspective do not require computing the empirical Pareto front, whereas PF-based algorithms incur additional computational overhead due to the need for empirical Pareto front estimation. Paradoxically, to the best of our knowledge, even sophisticated algorithms that aim to enforce PF lack both theoretical guarantees and empirical evaluations of their fairness indices.
\end{itemize}


\section{$\rglBound$-goodness}
\label{ap_sec:goodness}
In this section, we introduce the concept of $\rglBound$-goodness, compare it with the alternative regularity condition employed in another greedy bandit study \citet{Bayati2020}, and clarify the distinction between $\rglBound$-goodness and context diversity (Assumption~3 in \citealt{Bastani2020}) assumption, which is commonly used in the existing greedy bandit literature. 

We first extend the definition of a $\rglBound$-good arm in Definition~\ref{def:goodness} to an arbitrary vector.

\begin{definition}[$\rglBound$-good vectors]\label{def:good_vectors}
    For fixed $\rglBound \in (0,1]$, we say that the vector $x \in \mathbb{R}^d$ is $\rglBound$-good for the direction of $\theta \in \mathbb{R}^d$ if $x^\top{\theta \over \|\theta\|_2} \ge \rglBound$ holds.
\end{definition}

The following naturally arises from the definition; however, it plays a pivotal role in applying the goodness assumption to the analysis.

\begin{proposition} \label{prop:good_optimal}
    If there exists a $\rglBound$-good arm for $\theta$, then the optimal arm for $\theta$ is also $\rglBound$-good. 
\end{proposition}

\begin{proposition}\label{prop:expectation_good}
    Suppose $x$ is a random variable that can only take values corresponding to $\rglBound$-good arms for $\theta$. Then, $\mathbb{E}[x]$ is also $\rglBound$-good for $\theta$.
\end{proposition}

The above proposition holds because the region $\{x \in \mathbb{B}^d ~|~ x^\top{\theta \over \|\theta\|_2} \ge \rglBound \}$ is convex. 

\subsection{$\rglBound$-goodness condition for stochastic contexts setup}
\label{ap_subsec:goodness_stochastic}
Before explaining the meaning of $\rglBound$-goodness, we first extend the $\rglBound$-goodness condition to be applicable in a stochastic context setup. In a multi-objective linear contextual bandit framework, the stochastic context setup assumes that the context set $\chi(t)=\{x_i(t) \in \mathbb{R}^d, i \in [K]\}$ in each round $t$ is drawn from some unknown distribution $P_{\chi}(t)$. Detailed explanations regarding this problem can be found in Section~\ref{ap_subsec:MOLB_stochastic}. Under the stochastic context setup, we introduce the definition of goodness with respect to the context distribution and present the $\rglBound$-goodness assumption as follows. 

\begin{definition}[Goodness of arms -- stochastic context version]\label{def:goodness_stochastic}
    For fixed $\rglBound \le 1$, we say that the distribution $P_{\chi}(t)$ of feature vector set $\chi(t)$ satisfies $\rglBound$\textit{-goodness} condition if there exists a positive number $q_{\rglBound}$ that satisfies 
    \begin{equation*}
        \text{for all unit vectors } \beta \in \mathbb{R}^{d}, ~\mathbb{P}_{\chi(t)}[\exists i \in  [K],~ x_i(t)^\top\beta \ge \rglBound ] \ge q_{\rglBound}.
    \end{equation*} 
\end{definition}
\begin{assumption} [$\rglBound$-goodness -- stochastic context version]\label{assump:Good_stochastic}
    We assume $P_{\chi}(t)$ satisfies $\rglBound$\textit{-goodness} condition for all $t \in [T]$, with $\rglBound> 1-{\lambda^2 \over 18}$. 
\end{assumption}

Different from fixed version, the goodness condition requires only the positive probability $q_\rglBound$ of the presence of $\rglBound$-good arms not the existence of them (i.e. $q_\rglBound=1$). Instead, the condition requires $\rglBound$-good arms for not only the neighborhood of objective parameters but also all directions. In other words, $\rglBound$-goodness signifies that for any unit vector $\beta \in \mathbb{R}^{d}$, there exists at least one $\rglBound$-good arm for direction $\beta$ with a probability of at least $q_{\rglBound}$. Intuitively, if the probability density function of the context distribution is continuous and strictly positive in a neighborhood of the unit sphere, then for any direction $\beta$, a $\rglBound$-good arm exists with positive probability. The following lemma formalizes this concept.
\begin{lemma} \label{lem:rgl_continuous}
    Suppose $x_1(t), \ldots, x_K(t)$ are continuous variables with density function $f_1, \ldots, f_{K}$. If $f=f_1 + \ldots + f_{K}$ is a bounded function and there exist $r \in (0,1)$ satisfies $f$ is always positive at $\{x \in \RR^d~|~ r < \|x\|_2 < 1\}$), then $P_{\chi}(t)$ satisfies $\rglBound$-goodness for all $\rglBound \in (0,1)$.   
\end{lemma}
\begin{proof}    
Fix $\rglBound \in (0,1)$. From the definition of $f$, $f/K$ is the probability density function of $X=\text{Uniform}(x_1(t), \ldots , x_K(t))$. Define $p_{\beta}=\mathbb{P}_{\chi(t)}[X^\top\beta \ge \rglBound  ]$ for unit vector $\beta \in \mathbb{R}^{d}$. Then, for all unit vectors $\beta \in \mathbb{R}^{d}$, 
\begin{equation*}
     p_{\beta}=\mathbb{P}_{\chi(t)}[X^\top\beta \ge \rglBound] =\int_{\{x \in \mathbb{B}^R ~|~x^\top \beta \ge \rglBound\}} {f(x) \over K} dx  \ge \int_{\{x \in \mathbb{B}^R ~|~x^\top \beta \ge \max(\rglBound, r)\}} {f(x) \over K} dx>0.
\end{equation*}

Consider the function $F: ~\beta \stackrel {F}{\rightarrow} p_{\beta}$. From the boundedness of $f$, we can easily check $F$ is continuous. By the fact that the compactness is preserved by continuous functions, $\{p_{\beta}~|~ \beta \in \mathbb{R}^{d}, \|\beta\|_2=1\}$ is compact. Define $q_{\rglBound}:=\min\{p_{\beta}| \beta \in \mathbb{R}^{d}, \|\beta\|_2=1\}$, then we have $q_{\rglBound}>0$, since $p_{\beta}>0$ for all unit vectors $ \beta$.
Then, for all unit vectors $\beta \in \mathbb{R}^{d}$
\begin{equation*}
     \mathbb{P}_{\chi(t)}[\exists i \in  [K],~ x_i(t) ^\top\beta \ge \rglBound] \ge \mathbb{P}_{\chi(t)}[X^\top\beta \ge \rglBound ] = p_{\beta} \ge q_{\rglBound} 
\end{equation*}\
\end{proof}

\begin{remark} The above lemma states that if the set of arm $\chi(t)$ includes just a single continuous variable whose support can cover the unit sphere, then $\rglBound$-goodness will hold for all $\rglBound<1$ regardless of the distributions of the remaining arms.
\end{remark}

\subsection{$\rglBound$-goodness vs $\beta$-regularity}
\label{ap_subsec:goodness_beta}
In \citet{Bayati2020}, they assume the prior distribution $\Gamma$ of the expected reward $\mu$ of each arm satisfies $\mathbb{P}_{\mu}[\mu> 1-\epsilon]=\Theta({\epsilon}^{\beta})$ for all $\epsilon>0$ in non-contextual MAB setting. Let's compare this with $\rglBound$-goodness when $m=d=1$. We claim that $\rglBound$-goodness can be considered weaker than $\beta$-regularity from three perspectives. 

The most significant difference is that in $\beta$-regularity, the probability that the expected reward $\mu_i$ exceeds $1-\epsilon$ is required for all arm $i \in [K]$, along with the assumption that $\mu_i$'s are drawn independently from prior $\Gamma$. In contrast, in $\rglBound$-goodness, it is sufficient to ensure that the probability that one of the $K$ arms satisfies $x_i(t)^\top \beta \ge \rglBound$, without the need for the independence assumption between arm vectors. Secondly, unlike $\beta$-regularity, $\rglBound$-goodness does not require a specific relationship like $\Theta(1-\rglBound)$ between the probability of the existence of near-optimal arms $\mathbb{P}_{\chi(t)}[\exists i \in  [K],~ x_i(t)^\top\beta \ge \rglBound ]$ and the threshold $\rglBound$ ; instead, it focuses on the existence of a positive lower bound $q_{\rglBound}$. 
Lastly, the $\beta$-regularity assumes the probability of $\mu > 1-\epsilon$ for all $\epsilon>0$, while this work does not mandate $\rglBound$-goodness for $\rglBound$ very close to $1$; it is sufficient to hold $\rglBound$-goodness only for some $\rglBound \ge 1-({\lambda \over 18})^2$.

\subsection{$\rglBound$-goodness vs context diversity}
\label{ap_subsec:goodness_CD}

In recent years, there has been significant interest in the optimality of the Greedy algorithm in single-objective bandit problems~\citep{Bastani2020, Kannan2018, Raghavan2018, Hao2020}. A common theme among these studies is the assumption that feature vectors follow a distribution satisfying specific diversity conditions. For example, \citet{Bastani2020} assume the existence of a positive constant $\lambda$ such that for each vector $u \in \mathbb{R}^d$ and context vector $x_i(t)$, $\lambda_{\min} \big(\mathbb{E}[x_i(t)x_i(t)^\top\mathds{1}\{x_i(t)^\top u \ge 0\}\big) \ge \lambda$. The $\rglBound$-goodness condition fundamentally differs from traditional context diversity assumptions. Below, we provide examples where the $\rglBound$-goodness condition holds, while traditional diversity conditions do not. 

\textbf{Example 1} (Containing fixed arms)
Imagine a situation where one feature vector is a continuous variable while the other arms are fixed. For example, let $x_1(t)$ be uniformly distributed over $\mathbb{B}^{d}$ while $x_2(t)=x_2, \ldots, x_K(t)=x_K$ are fixed at some points in $\mathbb{R}^{d}$. By Lemma~\ref{lem:rgl_continuous}, $P_{\chi}(t)$ satisfies $\rglBound$-goodness for all $\rglBound \in (0,1)$. However, it is easy to see that diversity is not satisfied because $\lambda_{\min} \big(\mathbb{E}[x_2(t)x_2(t)^\top\mathds{1}\{x_2(t)^\top u \ge 0\}\big) = \lambda_{\min}(x_2x_2^\top \mathds{1}\{x_2^\top u \ge 0\}) \le \lambda_{\min}(x_2x_2^\top)=0$.

\textbf{Example 2} (Low-randomness distribution) 
Consider a scenario where the feature vectors are drawn from a finite set of discrete points. Despite the lack of diversity, if these points are strategically chosen to cover a unit sphere adequately, the goodness condition can still be satisfied. For example, suppose there is a set of points $P=\{a_1, a_2, \ldots, a_N\}$ that contains $\sqrt{1-{\rglBound}^2}$-net of the unit sphere. Assume that $x_1(t)$ be chosen uniformly from the $d-1$ points and other arms $x_2(t), \ldots, x_K(t)$ be chosen from the remaining points. Obviously, $P_{\chi}(t)$ satisfies $\rglBound$-goodness with $q_{\rglBound} \ge {1 \over N}$. In contrast, $\lambda_{\min} \big(\mathbb{E}[x_1(t)x_1(t)^\top\mathds{1}\{x_1(t)^\top u \ge 0\}\big) =0$ since there are only $d-1$ candidates that can be $x_1(t)$. Therefore, context diversity does not hold in this scenario.

Although $\rglBound$-goodness encompasses cases where the context diversity assumption is not covered, there is no inclusion relationship between the two conditions. Here is an example where $\rglBound$-goodness does not hold, but context diversity does.

\textbf{Example 3} (Proper support) 
Consider a case where $1$ is given as the upper bound of the $l_2$ norm of feature vectors, but the actual support of feature vectors is smaller. For instance, if $x_i(t)$ follows a uniform distribution over $\mathbb{B}_{1/2}^d$ for all $i \in [K]$ and $t \in [T]$, then context diversity still holds (\citet{Bastani2020}), but $\rglBound$-goodness does not hold for $\rglBound>1/2$.

\section{Analysis of $\mogrorw$ with fixed features}
\label{ap_sec:analysis_fixed}

\subsection{Linear growth of minimum eigenvalue of the Gram matrix}
\label{ap_subsec:pf_lem_mineigen_oneround}
The proof of Lemma~\ref{lem:mineigen_oneround} is presented in 
Section~\ref{sec:mineigen_oneround_proof}, the proof of linear growth of minimum eigenvalue of the gram matrix is presented in Section~\ref{sec:mineigen_linear}, and their supporting lemmas are presented in Section~\ref{sec:mineigen_oneround_lemma}.

\subsubsection{Technical lemmas for Lemma~\ref{lem:mineigen_oneround}}
\label{sec:mineigen_oneround_lemma}
The following lemma states that, after sufficient exploration rounds, the distance between the $\rglBound$-good arms for the OLS estimator of the objective parameters and the respective true objective parameters can be bounded.
\begin{lemma}
\label{lem:dist_good_arm} 
    Given Assumptions~\ref{assump:Bdd}, assume the target objective estimator $\targetthetas :=\sum_{m\in[M]}w_m\hatthetams$ satisfies $\|\targetthetas-\truethetam\| \le \rglBall$ for some $m \in [M]$ and $s \ge 1$. 
    If $x \in \mathbb{B}^d$ is $\rglBound$-good for $\targetthetas$, then the distance between $x$ and $\truethetam$ is bounded by 
    \begin{equation*}
        \|\truethetam-x\|_2 \le \sqrt{2+ 2\rglBall \sqrt{1-\rglBound^2}-2\rglBound\sqrt{1-\rglBall^2}}.
    \end{equation*}
\end{lemma}
%
\begin{figure}[t]
\begin{center}
    \includegraphics[width=0.5\textwidth, trim=6cm 4cm 6cm 4cm, clip]{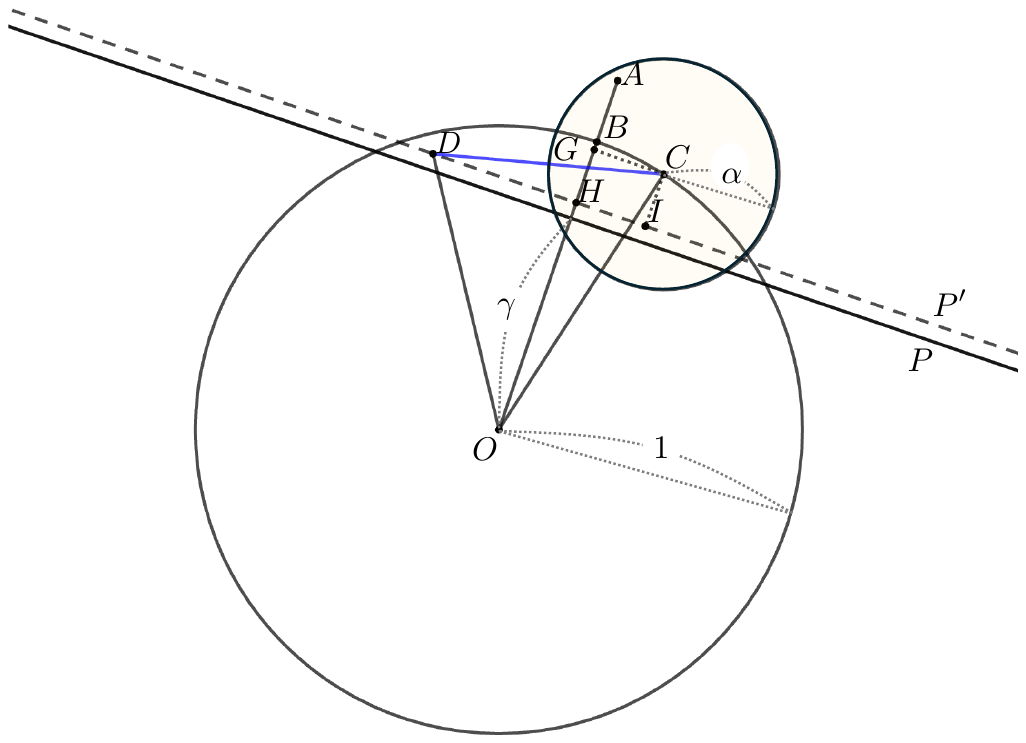} 
    \caption{The larger circle represents the unit sphere in $\RR^d$ while the interior of smaller circle indicates the region where $\targetthetas$ may exist. Then, the blue line illustrates the distance between $\truethetam$ and the $\gamma$-good arm for $\targetthetas$.}
    \label{fig:near_optimal_zone}
\end{center}
\end{figure}
\begin{proof}

Let the origin be denoted by $O$, and define $\targetthetas:=\vec{OA}$, ${\targetthetas \over \|\targetthetas\|}:=\vec{OB}$, $\theta_m^*=\vec{OC}$, and let a $\gamma$-good arm $x$ be denoted by $\vec{OD}$. By Assumption 1, $C$ lies on the unit $d$-dimension sphere centered at $O$(sphere $O$). By the assumption of the lemma, $A$ lies on or inside the sphere centered $C$ with radius $\rglBall$(sphere $C$), and $B$ is the intersection point of the extension of $OA$ with sphere $O$. Define the hyperplane $P$, orthogonal to $OB$, that passes through the point dividing $OB$ in the ratio $\gamma : 1-\gamma$. Then, by the definition of $\gamma$-good arms, point $D$ lies on or inside the unit sphere $O$, and must be located on or beyond the hyperplane $P$. 

Let $G$ and $H$ denote the foot of the perpendiculars from $C$ and $D$ onto $OB$, respectively. Since $D$ is located on or beyond the hyperplance $P$, $OH\ge \gamma$ and (the distance between $D$ and $OB$ )$=DH = \sqrt{OD^2-OH^2}\le \sqrt{1 - \gamma^2}$.
Letting (the distance between $C$ and $OB$ )$=CG = l ~(\le \rglBall)$, we have
$GH = OG - OH = \sqrt{OC^2-CG^2}-OH \le \sqrt{1 - l^2} - \gamma$. The equality holds when $OD=1$ and $OH=\gamma$, i.e., $D$ is lying on the hyperplane $P$. 

Now, consider the hyperplane $P'$ that passes through both $D$ and $H$ and is orthogonal to $OB$ (Figure 3 illustrates the case when $P = P'$). Let $I$ be the foot of the perpendicular from $C$ to $P'$. Then, $\square CGHI$ is a rectangle and we can bound $DI \le DH + IH = DH + CG \le \sqrt{1 - \gamma^2} + l$. The equality holds for both inequality when $D$, $H$, and $I$ lie on the same line, in that order, and $D$ is lying on the hyperplane $P$.

Therefore, by the Pythagorean theorem,
\begin{align*}
        \|\truethetam-x\|_2^2&=CD^2 = DI^2 + CI^2 = DI^2 + GH^2\\
        &\le (\sqrt{1-\gamma^2 }+l)^2 + (\sqrt{1-l^2 }-\gamma)^2\\
        &=1-\gamma^2+l^2+2(\sqrt{1-\gamma^2})l + 1-l^2 +\gamma^2 -2(\sqrt{1-l^2})\gamma\\
        &=2+2(\sqrt{1-\gamma^2})l-2(\sqrt{1-l^2})\gamma\\
        &\le 2+2(\sqrt{1-\gamma^2})\rglBall-2(\sqrt{1-\rglBall^2})\gamma.
\end{align*}
The last inequality holds since $l \le \rglBall$. 
\end{proof}

\subsubsection{Proof of Lemma~\ref{lem:mineigen_oneround}}
\label{sec:mineigen_oneround_proof}
\begin{proof}
For $s \ge T_0 +1$ and $m \in [M]$, let $E_{\bar{m}}(s)$ be the event that the weighted objective $\sum_{m \in [M]}w_m\theta_m^*$ in round $s$ satisfies $\left\|\sum_{m\in [M]}w_m\theta_m^*-\theta_{\bar{m}}^*\right\|_2 < \rglBall/2$. Then, on $E_{\bar{m}}(s)$, if $\|\hatthetams-\truethetam\| \le \rglBall/2$ holds, then the following holds. 
\begin{align*}
    \left\|\sum_{m \in [M]}w_m\hatthetamt-\truethetam\right\|_2
     &\le 
    \left\|\sum_{m \in [M]}w_m\hatthetamt-\sum_{m \in [M]}w_m\theta_m^*\right\|_2
    +   \left\|\sum_{m \in [M]}w_m\theta_m^*-\truethetam\right\|_2\\
    & \le \sum_{m \in [M]}w_m\left\|\hatthetamt-\theta_m^*\right\|_2 +  \left\|\sum_{m \in [M]}w_m\theta_m^*-\truethetam\right\|_2 \\
    & < {\rglBall \over2} +{\rglBall \over2} =\rglBall
\end{align*}
Thus, by Assumption~\ref{assump:Good}, there exists $\rglBound$-good arm for the weighted objective $\sum_{m \in [M]}w_m\hatthetamt$ in round $t$. 

Since the arm selected by $\mogrorw$ satisfies
\begin{align} \label{eq:mineigen_oneround}
    \mathbb{E}[x(s)x(s)^\top| \mathcal{H}_{s-1}] &= \sum_{m=1}^M \mathbb{E}[x(s)x(s)^\top| E_m(s), \mathcal{H}_{s-1}] \mathbb{P}[E_m(s) | \mathcal{H}_{s-1}] \nonumber\\
    &\succeq \phi_{\rglBall/2, \Wcal}\sum_{m=1}^M  \mathbb{E}[x(s)x(s)^\top| E_m(s), \mathcal{H}_{s-1}] \nonumber\\
    &\succeq  \phi_{\rglBall/2, \Wcal} \sum_{m=1}^M \mathbb{E}[x(s)| E_m(s), \mathcal{H}_{s-1}]\mathbb{E}[x(s)| E_m(s), \mathcal{H}_{s-1}]^\top. 
\end{align}

Let $x_{r(m)}:=\mathbb{E}[x(s)| E_m(s), \mathcal{H}_{s-1}]$. Then,   $x_{r(m)}$ is the expectation of the selected arm for the target objective $\sum_{m \in [M]}w_m\hatthetams$ in round $s$ when it satisfies $\left\|\sum_{m \in [M]}w_m\hatthetams -\truethetam\right\|_2<\rglBall$.
Since there always exist $\rglBound$-good arms for the weighted objective $\sum_{m \in [M]} w_m \hatthetams$ on $E_m(S)$ for all $m \in [M]$, the selected arms are also $\rglBound$-good (Proposition~\ref{prop:good_optimal}), and hence so is $x_{r(m)}$ (Proposition~\ref{prop:expectation_good}).
Thus, by Lemma~\ref{lem:dist_good_arm}, we can get $\|{x_{r(m)}} - \theta_{m}^*\|_2 \le \sqrt{2+ 2\rglBall \sqrt{1-\rglBound^2}-2\rglBound\sqrt{1-\rglBall^2}}$ for all $m \in [M]$. 

Now, we will bound $\lambdamin\left(\sum_{m \in [M]}x_{r(m)}\big(x_{r(m)} \big)^{\top}\right)$. For any unit vector $u\in \mathbb{B}^d$, the following holds 
\begin{align*}
    u^\top\left(\sum_{m \in [M]}x_{r(m)}\big(x_{r(m)} \big)^{\top}\right)u &=\sum_{m \in [M]}\left\langle u,x_{r(m)}\right\rangle^2\\
    &=\sum_{m \in [M]}\left\langle u, \theta_{m}^*+(x_{r(m)}-\theta_{m}^*) \right\rangle ^2\\
    &=\sum_{m \in [M]} \{\left\langle u, \theta_{m}^* \right\rangle ^2+ \left\langle u, x_{r(m)}-\theta_{m}^* \right\rangle ^2+2 \left\langle u, \theta_{m}^*\right\rangle \left\langle u,x_{r(m)}-\theta_{m}^* \right\rangle\}\\
    &\ge  u^\top\left(\sum_{m \in [M]} \truethetam \big(\truethetam \big)^\top\right)u +0 -2 \sqrt{2+ 2\rglBall \sqrt{1-\rglBound^2}-2\rglBound\sqrt{1-\rglBall^2}}M \\
    &\ge \lambda M- 2 \sqrt{2+ 2\rglBall \sqrt{1-\rglBound^2}-2\rglBound\sqrt{1-\rglBall^2}}M.
\end{align*}

This implies 
\begin{equation*}
    \lambdamin\left(\sum_{m \in [M]}x_{r(m)}\big(x_{r(m)} \big)^{\top}\right) \ge \left(\lambda- 2 \sqrt{2+ 2\rglBall \sqrt{1-\rglBound^2}-2\rglBound \sqrt{1-\rglBall^2}}\right)M.
\end{equation*}

Plugging into Eq.~\eqref{eq:mineigen_oneround}, we obtain
\begin{align*}
    \lambda_{\min}(\mathbb{E}[x(s)x(s)^\top| \mathcal{H}_{s-1}]) &\ge  \phi_{\rglBall/2, \Wcal}~ \lambdamin\left(\sum_{m \in [M]}x_{r(m)}\big(x_{r(m)} \big)^{\top}\right) \\
    &\ge  \left(\lambda- 2 \sqrt{2+ 2\rglBall \sqrt{1-\rglBound^2}-2\rglBound \sqrt{1-\rglBall^2}}\right){\phi_{\rglBall/2, \Wcal}M}.
\end{align*}
\end{proof}

\begin{remark} \label{rmk:exchange_alpha}
Recall that $\rglBall$ denotes the value satisfying the goodness condition defined in Definition~\ref{def:goodness}, together with $\rglBound$ as specified in Assumption~\ref{assump:Good}. If $\rglBall$ is too large, the term $\lambda- 2 \sqrt{2+ 2\rglBall \sqrt{1-\rglBound^2}-2\rglBound \sqrt{1-\rglBall^2}}$ may become negative. That is, although the assumptions are satisfied for a given $\alpha$, the value of ${\alpha \over 2}$ may be too large to be used as an exploration threshold for $\|\hatthetamt - \truethetam\|$. To ensure a positive increment of the minimum eigenvalue of the Gram matrix, whenever $\rglBall$ exceeds $\psi(\lambda, \rglBound):=\sqrt{{\lambda^2 \over 9}-{\lambda^4 \over 324}}~\rglBound-\left( 1-{\lambda^2 \over 18} \right)\sqrt{1-\rglBound^2}$, we replace $\rglBall$ with $\psi(\lambda, \rglBound)$. Since a larger $\rglBall$ corresponds to a stricter goodness condition, reducing $\rglBall$ preserves the validity of the condition. With this replacement, the right-hand side is lower bounded by
$\left(\lambda- 2 \sqrt{2+ 2\rglBall \sqrt{1-\rglBound^2}-2\rglBound \sqrt{1-\rglBall^2}}\right){\phi_{\rglBall/2, \Wcal}M} \ge {\lambda {\phi_{\rglBall/2, \Wcal}M} \over 3}$,
thereby guaranteeing a positive increment of the eigenvalues of the Gram matrix.
\end{remark}

\subsubsection{Proof of linear growth of minimum eigenvalue of the Gram matrix}
\label{sec:mineigen_linear}
The following lemma shows that the minimum eigenvalue of the Gram matrix increases at a rate $O(\lambdainc)$, where $\lambdainc=\big(\lambda- 2 \sqrt{2+ 2\rglBall \sqrt{1-\rglBound^2}-2\rglBound \sqrt{1-\rglBall^2}}\big) {\phi_{\rglBall/2, \Wcal}M}$.
\begin{lemma}[Minimum eigenvalue growth of Gram matrix]\label{lem:mineigen_linear}
    Suppose Assumptions ~\ref{assump:Bdd}, ~\ref{assump:OD}, and ~\ref{assump:Good} hold. Assume the OLS estimator satisfies $\|\hatthetams-\truethetam\|_2 \le {\rglBall \over 2}$ for all $m \in [M]$ and $s \ge T_0+1$. Then for $t \ge T_0$, the following holds for the minimum eigenvalue of the Gram matrix of arms selected by $\mogrorw$
    \begin{equation*}
        \mathbb{P}\left[ \lambda_{\min}\left(\sum_{s=1}^{t} x(s)x(s)^\top\right) \le B + {\lambdainc\over 2}(t-T_0) \right] \le de^{-{\lambdainc}(t-T_0)  \over 10},
    \end{equation*}
    where $\lambdainc=\left(\lambda - 2 \sqrt{2+2\rglBall\sqrt{1-\rglBound^2}-2\rglBound\sqrt{1-\rglBall^2}}\right) \phi_{\rglBall/2, \Wcal}M$.
\end{lemma}

\begin{proof}
By the subadditivity of minimum eigenvalue and Lemma~\ref{lem:mineigen_oneround}, for $t \ge T_0+1$, 
\begin{align*}
    \lambda_{\min}\left(\sum_{s=T_0+1}^{t}\mathbb{E}[x(s)x(s)^\top| \mathcal{H}_{s-1}]\right) \ge  \sum_{s=T_0+1}^{t}\lambda_{\min}(\mathbb{E}[x(s)x(s)^\top| \mathcal{H}_{s-1}]) \ge {\lambdainc}(t-T_0). 
\end{align*}
In other words, $\mathbb{P}[\lambda_{\min}(\sum_{s=T_0+1}^{t}\mathbb{E}[x(s)x(s)^\top| \mathcal{H}_{s-1}]) \ge  \lambdainc(t-T_0)]=1$ holds for $t \ge T_0+1$. By applying Lemma~\ref{lem:Tropp_3.1} to compute the lower bound of the minimum eigenvalue of the Gram matrix after exploration, we have 
\begin{equation*}
    \mathbb{P}\left[ \lambda_{\min}\left(\sum_{s=T_0+1}^{t} x(s)x(s)^\top\right) \le {\lambdainc \over 2}(t-T_0)\right] \le d({e^{0.5} \over 0.5^{0.5}})^{-{\lambdainc}(t-T_0)} \le de^{-\lambdainc(t-T_0)  \over 10}. 
\end{equation*}
Therefore, by subadditivity of minimum eigenvalue, for $t \ge T_0$
\begin{equation*}
    \mathbb{P}\left[ \lambda_{\min}\left(\sum_{s=1}^{t} x(s)x(s)^\top\right) \le B+{\lambdainc \over 2}(t-T_0)\right] \le de^{-\lambdainc(t-T_0)  \over 10}. 
\end{equation*}
\end{proof}

\subsection{Proof of the effective Pareto regret bound}
\label{ap_subsec:pf_thm_EPR}

Theorem~\ref{thm:EPR} is proven by deriving an $l_2$ bound on $\hatthetamt-\truethetam$. This is enabled by Lemma~\ref{lem:mineigen_linear}, which shows that the minimum eigenvalue of the Gram matrix grows linearly with $t$ with high probability, thereby allowing us to obtain the desired bound. The proof of Theorem~\ref{thm:EPR} is presented in 
Section~\ref{sec:EPR_proof} and its supporting lemmas are presented in Section~\ref{sec:EPR_lemma}.

\subsubsection{Technical lemmas for Theorem~\ref{thm:EPR}}
\label{sec:EPR_lemma}
 To apply Lemma~\ref{lem:mineigen_oneround}, a sufficient number of initial exploration is required to ensure its preconditions are satisfied. We discuss this requirement in the next section (Section~\ref{ap_subsec:initial_rounds}). In the current section, we assume this condition is met via Lemma~\ref{lem:set_B}, and proceed to prove Theorem~\ref{thm:EPR}. With Lemma ~\ref{lem:mineigen_linear}, we are ready to derive the $l_2$ bound of $\hatthetamt-\truethetam$ for $m \in [M]$.
\begin{lemma}
\label{lem:thetahat_bound}
    Fix $\delta>0$. Under the same conditions as those in Lemma ~\ref{lem:mineigen_linear}, with probability at least $1-M\delta- de^{-\lambdainc(t-T_0)  \over 10}$, the OLS estimator $\hatthetamt$ of $\truethetam$  satisfies 
    \begin{equation*}
         \left\|\hat{\theta}_m(t+1)-\truethetam\right\|_2 \le {4\sigma\ \over \lambdainc} \sqrt{ d  \log (dt / \delta) \over {t-T_0}},
    \end{equation*}
    where $\lambdainc$ is same as in lemma~\ref{lem:mineigen_linear}.
\end{lemma}
\begin{proof}

From the closed form of the OLS estimators, for all $m \in [M]$, 
\begin{align*}
    \left\|\hat{\theta}_m(t+1)-\truethetam\right\|_2 & = \left\|\left(\sum_{s=1}^{t} x(s)x(s)^\top\right)^{-1}\sum_{s=1}^{t} x(s) \eta_{a(s), m}(s)\right\|_2 \nonumber \\
    & \le {1 \over \lambdamin\left(\sum_{s=1}^{t} x(s)x(s)^\top\right)}\left\|\sum_{s=1}^{t} x(s) \eta_{a(s), m}(s)\right\|_2 
\end{align*}

For the denominator, we have $\lambdamin\left(V_{t}\right) \ge B + {\lambdainc \over 2}(t-T_0) $  for $t \ge T_0$, with probability at least $1- de^{-\lambdainc(t-T_0)  \over 10}$, by Lemma~\ref{lem:mineigen_linear}. To bound the $l_2$ norm of $S_{t, m}:=\sum_{s=1}^{t} x(s) \eta_{a(s), m}(s)$, we can use Lemma~\ref{lem:Kannan_A1}, the martingale inequality of ~\citet{Kannan2018}. The lemma states for fixed $m \in [M]$, $\|S_{t, m}\|_2 \le \sigma\sqrt{2 dt  \log (dt / \delta)}$ holds with probability at least $1- \delta $. Therefore, with probability at least $1- M\delta- de^{-\lambdainc(t-T_0)  \over 10}$, for all $m\in [M]$ and $t \ge 2T_0$,
\begin{equation*}
    \left\|\hat{\theta}_m(t+1)-\truethetam\right\|_2 \le  {\sigma\sqrt{ 2dt  \log (dt / \delta)} \over B + {\lambdainc (t-T_0) / 2}} 
     \le {4\sigma\ \over \lambdainc} \sqrt{ d  \log (dt / \delta) \over {t-T_0}}.
\end{equation*}
The last inequality holds when $t \ge 2T_0$. 
\end{proof}

\subsubsection{Proof of Theorem~\ref{thm:EPR}}
\label{sec:EPR_proof}
\begin{proof}
By Lemma~\ref{lem:set_B}, if $B$ is set by $B= \min\left[{2\sigma \over \rglBall  }\sqrt{2 dT  \log ({dT^2})},~~{16\sigma^2  \over \rglBall^2} \left( {d \over 2} \log\left(1+{2T \over d}\right)+\log\left({T}\right)\right)\right]$, we have $\|\hatthetamt-\truethetam\| \le {\rglBall \over 2}$ for all $m \in [M]$ and $t \ge T_0+1$ with probability at least $1-{2M \over T}$. Let $E$ be the event that $\| \hat{\theta}_m(t+1)-\theta_m^* \| \le {4 \sigma \over \lambdainc} \sqrt{ d  \log (dt T)\over {t-T_0}}$ holds for all $t\ge2T_0$ and $m \in [M]$. Then, $\mathbb{P}(\bar{E}) \le {2M \over T} + {M \over T} + de^{-{\lambdainc}(t-T_0)  \over 10}$ by Lemma~\ref{lem:thetahat_bound}.  

Fix round $t \in [T]$. Let $w (t)\in \Delta^M$ be the generated weight vector in round $t$, and $a_{w(t)}^*$ be the true optimal arm for the weighted objective $\sum_{m\in [M]}w_m(t)\truethetam$. Then, by Proposition~\ref{prop:EPFtoLW}, $a_{w(t)}^*$ is in the effective Pareto Front, and so we have
\begin{align} \label{eq:weightsumbdd}
    \Delta_{a(t)}(t) & := \min_{m \in [M]} \left(x_{a_{w(t)}^*}^\top \truethetam-x_{a(t)}^\top \truethetam\right) \nonumber\\
    &\le \sum_{m\in [M]}\left(w_m(t)x_{a_{w(t)}^*}^\top \truethetam-w_m(t)x_{a(t)}^\top \truethetam\right)\nonumber \\
    &= x_{a_{w(t)}^*}^\top \left(\sum_{m\in [M]}w_m(t)\truethetam\right)-x_{a(t)}^\top\left(\sum_{m\in [M]}w_m(t)\truethetam\right)\nonumber\\
    & \le 2 \left\|\sum_{m\in [M]}w_m(t) \hatthetamt-\sum_{m\in [M]}w_m(t) \theta_m^*\right\|_2 \nonumber\\
    & = 2 \left\|\sum_{m\in [M]}w_m(t) \left(\hatthetamt-\theta_m^* \right)\right\|_2 \nonumber\\
    & \le 2 \sum_{m\in [M]}w_m(t) \left\|\hatthetamt-\theta_m^* \right\|_2 
\end{align}
with Assumption~\ref{assump:Bdd}.

Let $\Delta_{\max}$ be the maximum suboptimality gap. For $t \ge 2T_0$, 
\begin{align*}
    \mathbb{E}[\Delta_{a(t+1)}(t+1)] &\le \mathbb{E}[\Delta_{a(t+1)}(t+1)~|~E]+ \mathbb{P}(E)\Delta_{\max}\\
    &\le 2 \mathbb{E}\left[\sum_{m\in [M]}w_m(t+1) ~\| \hat{\theta}_{m}(t+1)-\theta_{m}^* \|_2~|~E\right]+ \left({3M \over T}  + de^{-{\lambdainc }(t-T_0)  \over 10}\right)\Delta_{\max}\\
    &\le  {8\sigma \over \lambdainc} \sqrt{ d  \log (dt T)\over {t-T_0}}+ \left({3m \over T}  + de^{-{\lambdainc }(t-T_0)  \over 10}\right)\Delta_{\max}.
\end{align*}

Then, the Pareto regret is bounded by
\begin{align*}
    \regret(T) &=  \sum_{t= 2T_0}^{T-1} \mathbb{E}[\Delta_{a(t+1)}(t+1)] + 2T_0\Delta_{\max}\\
    & \le   \sum_{t= 2T_0}^{T} {8\sigma \over \lambdainc } \sqrt{ d  \log (dt T)\over {t-T_0}}+ \{({3M \over T})T + \sum_{t= 2T_0}^T  de^{-{\lambdainc }(t-T_0)  \over 10} + 2T_0\}\Delta_{\max}  \\
    & \le  {8\sigma \over \lambdainc} \sqrt{2 d  \log (dT)}  \int_0^T {1 \over \sqrt{t}}dt+ \left(2T_0+3M +\sum_{t= 2T_0}^T  de^{-{\lambdainc }(t-T_0)  \over 10} \right)\Delta_{\max}  \\
    & \le {16\sigma \over \lambdainc}{\sqrt{ 2dT \log (dT)}}+\left(2T_0+3M+{10d \over \lambdainc} \right)\Delta_{\max} \\
    & \le {16\sigma \over \lambdainc}{\sqrt{ 2dT \log (dT)}}+2\left(2T_0+3M+{10d \over \lambdainc}\right). 
\end{align*}
The last inequality holds because we have $\Delta_{\max} \le 2$ under Assumption ~\ref{assump:Bdd}. 
\end{proof}

\subsection{Proof of Theorem~\ref{thm:EPF} (Effective Pareto Fairness of $\mogrorw$)}
\label{ap_subsec:pf_thm_EPF}

\begin{proof}
By setting $B$ as the same value in Theorem~\ref{thm:EPR}, we obtain $\|\hatthetams-\truethetam\| \le \rglBall/2$ for $s \ge T_0+1$ by lemma~\ref{lem:set_B}. Let $E$ be the event that $\| \hat{\theta}_m(t+1)-\theta_m^* \| \le {4\sigma \over \lambdainc } \sqrt{ d  \log (dt T)\over {t-T_0}}$ holds for all $t\ge2T_0$ and $m \in [M]$. Then, $\mathbb{P}(\bar{E}) \le {3M \over T}+ de^{-{\lambdainc }(t-T_0)  \over 10}$ by Lemma~\ref{lem:thetahat_bound}. 

Fix $\epsilon>0$, and define the event $\Omega_{w,t}$ for all $w \in \mathbb{S}^{M-1}$ as
\begin{equation*}
    \Omega_{w,t} := \left\{\omega \in \Omega ~|~ \text{Target Objective for round } t  \text{ satisfies } \left\|  \sum_{m \in [M]} w_m(t)\truethetam -  \sum_{m \in [M]} w_m \truethetam \right\|_2 < {\epsilon \over 3}\right\},
\end{equation*}
where $w(t)\in \mathbb{S}^{M-1}$ is the weight vector for round $t$. generated from $\Wcal$. 
Then, by the definition of regularity index $\psi$, $\mathbb{P}(\Omega_{w,t}) \ge \psi_{\epsilon/3 , \Wcal}$ for all $w \in \mathbb{S}^{M-1}$ and $t \le T$.

For $t \ge 2T_0$, on $\Omega_{w,t+1} \cap E$, we have that  
\begin{align*}
    \mu_{w}^*-\mu_{a(t+1), w} &=\mu_{a_{w}^*,w}-\mu_{a(t+1), w}\\
    &\le \big(\mu_{a_{w}^*,w}-\mu_{a_{w}^*,w(t+1)} \big) +\big(\mu_{a_{w}^*,w(t+1)} - \mu_{a(t+1), w(t+1)} ) +
    (\mu_{a(t+1), w(t+1)}-\mu_{a(t+1), w}) \\
    &\le (\|x_{a_w^*}\|_2 +\|x_{a(t+1)}\|_2  ) \left\|  \sum_{m \in [M]} w_m(t)\truethetam -  \sum_{m \in [M]} w_m \truethetam \right\|_2 \\
    & ~~~~~ + \sum_{m\in [M]}\left(w_m(t+1)x_{a_{w}^*}^\top \truethetam-w_m(t+1)x_{a(t+1)}^\top \truethetam\right) \\
    &  \le {2 \over 3} \epsilon + \sum_{m\in [M]}\left(w_m(t+1)x_{a_{w}^*}^\top \truethetam-w_m(t+1)x_{a(t+1)}^\top \truethetam\right)  ~~~~~~(\because \text{Assumption~\ref{assump:Bdd}})\\
    & \le {2 \over 3} \epsilon +  2\sum_{m\in [M]}w_m(t+1)\| \hat{\theta}_m(t+1)-\truethetam \|_2 ~~~~~~~~~~~~~~~~~~~~~~~~~~~(\because \text{Equation~\ref{eq:weightsumbdd}})\\
    &\le {2 \over 3} \epsilon + {8\sigma \over \lambdainc } \sqrt{ d  \log (dt T)\over {t-T_0}} \\
    & < {2 \over 3} \epsilon + {8\sigma \over \lambdainc } \sqrt{2 d  \log (d T)\over {t-T_0}}.
\end{align*}

Let $T_\epsilon=\max(\lfloor{1152\sigma^2 d\log(dT) \over \lambdainc ^2 \epsilon^2}\rfloor+T_0,~~2T_0)$. Then, on $\Omega_{w,t+1} \cap E$, we have $\mu_{m}^*-\mu_{a(t+1), m} < \epsilon$ for all $t>T_\epsilon$. 

Then, for all $m \in [M]$,
\begin{align*}
    {1 \over T}\mathbb{E}\left[\sum_{t=1}^T\mathds{1}\{\mu_{m}^*-\mu_{a(t), m}<\epsilon\}\right] 
    &\ge {1 \over T}\sum_{t=1}^T \mathbb{E}[\mathds{1}\{\mu_{m}^*-\mu_{a(t), m}<\epsilon\}~|~\Omega_{w,t}]~ \mathbb{P}(\Omega_{w,t})\\
    & \ge {\psi_{\epsilon/3 , \Wcal} \over T}\sum_{t=T_\epsilon}^{T-1}\mathbb{E}[\mathds{1}\{\mu_{m}^*-\mu_{a(t+1), m}<\epsilon\}~|~\Omega_{w,t+1}\cap E]~ \mathbb{P}(E)\\
    & \ge  {\psi_{\epsilon/3 , \Wcal}  \over T}\sum_{t=T_\epsilon}^{T-1}1\cdot~ \mathbb{P}(E)\\
    & \ge {\psi_{\epsilon/3 , \Wcal} \over T}\sum_{t=T_\epsilon}^{T-1} \left( 1-{3M \over T}-de^{-{\lambdainc}(T_\epsilon-T_0) \over 10} \right)\\
    & \ge {\psi_{\epsilon/3 , \Wcal}  \over T}(T-T_\epsilon) \left(1-{3M \over T}-d({1 \over dT})^{116\sigma^2 d \over \lambdainc \epsilon^2}\right).
\end{align*}

Therefore, the effective Pareto fairness index can be bounded by 
\begin{equation*}
    \textnormal{EPFI}_{\epsilon,T} \ge 
    {\psi_{\epsilon/3 , \Wcal}}\left({T-T_\epsilon \over T}\right) \left(1-{3M \over T}-d({1 \over dT})^{116\sigma^2 d \over \lambdainc \epsilon^2}\right),
\end{equation*} 
\end{proof}


\subsection{The parameter $B$ and the number of initial rounds}
\label{ap_subsec:initial_rounds}
In this section, we discuss the appropriate value of $B$, the threshold of the minimum eigenvalue of the Gram matrix. For convenience, denote $V_t:=\sum_{s=1}^{t} x(s)x(s)^\top$ and $S_t:=\sum_{s=1}^{t} x(s)\eta_{a(s)}(s)^\top$. When the minimum eigenvalue of the empirical covariance matrix $V_{T_0-1}$ exceeds a certain threshold, we can guarantee the $l_2$ bound of the OLS estimator $\hat{\theta}(t)$ of $\theta_*$ for $t \ge {T_0}$ with high probability. I.e.,   
\begin{equation}\label{eq:thresholds}
    \lambdamin(V_{T_0-1}) \ge f(a)~~ \Rightarrow ~~\text{for all } t \ge T_0,~~~\big\|\hat{\theta}(t)-\theta_*\big\|_2 \le a
\end{equation}
If we set $B=f({\rglBall \over 2})$, then with high probability, $\|\hat{\theta}_m(t)-\truethetam\|_2 \le {\rglBall \over 2}$ after initial rounds. 

\citet{Kevton2020} suggest $f(a)$ that satisfies Eq.(\ref{eq:thresholds}) using a bound of $\|S_t\|_{{V_{t-1}}^{-1}}$. However, a small mistake was made in their process: the bound they derived by modifying Theorem 1 of \citet{Abbasi-Yadkori2011} is actually a bound for $\|\sum_{s=\tau_0+1}^{t} x(s)\eta_{a(s)}(s)^\top\|_{{V_{t-1}}^{-1}}$, where $\tau_0=\min\{t \ge 1: V_t \succ 0 \}$, not $\|S_t\|_{{V_{t-1}}^{-1}}$. To address this problem, the simplest approach would be to use the bound of $\|S_t\|_2$ suggested by \citet{Kannan2018}. Alternatively, we can use the bound of $\|S_t\|_{{V_{t-1}}^{-1}}$ proposed by \citet{LihongLi2017}. The following lemma explains how the theoretical value of the initial parameter $B$, given by $\tilde{\Ocal}(\min(\sqrt{dT}, d\log T))$, can be derived through these two approaches.
\begin{lemma}\label{lem:set_B}
    Given Assumption 1, for any $a>0$ and $\delta >0$, if we run $\mogro$ with  
    \begin{equation*}
        B= \min\left[{\sigma \over a }\sqrt{2 dT  \log ({dT \over \delta})},~~{4\sigma^2  \over a^2} \left( {d \over 2} \log\left(1+{2T \over d}\right)+\log\left({1 \over \delta}\right)\right)\right],
    \end{equation*}
    then with probability at least $1-2M\delta$, the OLS estimator satisfies $\|\hatthetamt-\truethetam\|_2 \le a$ for all $m \in [M]$ and $t \ge T_0+1$. 
\end{lemma}
\begin{proof}
First we will bound $B$ using the fact 
\begin{align*}
    \left\|\hat{\theta}_m(t)-\truethetam\right\|_2 & =\left\|{({V_{t-1})}^{-1}}S_{t-1, m}\right\|_2 \le {1 \over \lambdamin\left(V_{t-1}\right)}\|S_{t-1, m}\|_2, 
\end{align*}
where $S_{t,m}:=\sum_{s=1}^{t} x(s)\eta_{a(s),m}(s)^\top$.

Since for fixed $m\in[M]$, $\|S_{t-1,m}\|_2\le \sigma\sqrt{2dt  \ln (td / \delta)}$ holds  for all $t\le T$ with probability at least $1- \delta $ by Lemma ~\ref{lem:Kannan_A1} and it is obvious that $\lambdamin(V_{t-1})\ge\lambdamin(V_{T_0-1})$ for $t \ge T_0$, we have
$\left\|\hat{\theta}_m(t)-\truethetam\right\|_2 \le a $ for all $m \in [M]$ and $t \ge T_0+1$  with probability at least $1-M\delta$ when the value of $B$ set to ${\sigma \over a }\sqrt{2 dT  \log ({dT /\delta})}$. 

Alternatively, we can use the fact
\begin{equation*}
    \left\|\hat{\theta}_m(t)-\truethetam\right\|_2^2 ={(S_{t-1,m})}^\top {{V_{t-1}}^{-1}}{{V_{t-1}}^{-1}}S_{t-1, m} \le {1 \over \lambdamin(V_{t-1})}\|S_{t-1, m}\|^2_{{V_{t-1}}^{-1}}.
\end{equation*}
By Lemma~\ref{lem:Chen_8},  for fixed $m\in[M]$, $\|S_{t-1, m}\|_{{V_{t-1}}^{-1}}^2 \le 4 \sigma^2 ( {d \over 2} \log(1+{2t \over d})+\log({1 \over \delta}))$
holds for all $t \le T$ with probability at least $1-\delta$, and hence, we have $\left\|\hat{\theta}_m(t)-\truethetam\right\|_2 < a $ for all $m \in [M]$ and $t \ge T_0+1$ with probability at least $1-M\delta$  by setting $B$ to ${4 \sigma^2 \over a^2} ( {d \over 2} \log(1+{2T \over d})+\log({1 \over \delta}))$. 

Therefore, if we set $B= \min\left[{\sigma \over a }\sqrt{2 dT  \log ({dT \over \delta})},~~{4\sigma^2  \over a^2} \left( {d \over 2} \log\left(1+{2T \over d}\right)+\log\left({1 \over \delta}\right)\right)\right]$, we have  $\left\|\hat{\theta}_m(t)-\truethetam\right\|_2 < a $ for all $m \in [M]$ and $t \ge T_0+1$ with probability at least $1-2M\delta$.
\end{proof}

\subsubsection{Proof of Lemma~\ref{lem:ini_rounds}}
\begin{proof}
Let $S$ be the feature set selected during initial rounds and $\lambda_S:=\lambdamin\left(\sum_{x_i \in S} x_i(x_i)^\top\right)$ Then, for any $\tau \ge  \lfloor {B \over \lambda_S}\rfloor \times|S|$, if we keep playing with the initial values for $T_0$ rounds,
\begin{align*}
    \lambdamin\left(\sum_{s=1}^{\tau}x(s)x(s)^\top\right)
    \ge \lambda_S \times {\tau \over |S|} =
    {\lambda_S }\lfloor  {B \over \lambda_S}\rfloor  \ge B.
\end{align*}
Hence, we have $T_0 \le \lfloor {B \over  \lambda_S}\rfloor \times |S|$.
\end{proof}


\section{Deterministic version of $\mogro$ algorithm ($\mogrorr$)}
\label{ap_sec:MOGRO_rr}

\subsection{Greedy with Round-Robin Objective}
\label{ap_subsec:RMOG}
We propose a fully deterministic version of $\mogro$ algorithm named $\mogrorr$ algorithm, which selects arm greedily with respect to the objectives in round-robin fashion for each round. The algorithm represents the simplest variant of $\mogro$ and is very easy to implement, yet it learns the objectives highly efficiently. In the rest of the section, we establish the statistical guarantees of $\mogrorr$.

\begin{algorithm}[t]
\caption{Multi-Objective -- Greedy with Round-Robin Objective Algorithm (\mogrorr)}
\label{alg:mogrorr}
\begin{algorithmic}[1]
    \STATE \textbf{Input:} Total rounds $T$, Eigenvalue threshold $B$
    \STATE \textbf{Initialization:} $V_0 \leftarrow 0 \times I_d$, $S:$ feature basis
    \FOR{$t = 1, \ldots, T$}
        \IF{$\lambdamin(V_{t-1}) < B$} 
            \STATE Select action $a(t) \in S$ in round-robin manner
        \ELSE 
            \STATE Update the estimators $\Theta_t=\big(\hat{\theta}_1(t), \ldots ,\hat{\theta}_M(t)\big)$
            \STATE Select the target objective $m \leftarrow t~\texttt{mod}~M$ \COMMENT{If $m=0$, then $m \leftarrow M$}
            \STATE Select action $a(t) \in arg\max_{i \in [K]} x_i^\top\hatthetamt$ 
        \ENDIF
    \STATE Observe $y(t)=\big(y_{a(t),1}(t),\ldots,y_{a(t),M}(t)\big)$
    \STATE Update $V_t \leftarrow V_{t-1}+x(t)x(t)^\top$      
\ENDFOR
\end{algorithmic}
\end{algorithm}

\subsection{Linear growth of minimum eigenvalue}
We establish the lower bound of the minimum eigenvalue on the Gram matrix that grows linearly with respect to $t$. Specifically, we leverage the presence of good arms for each objective to guarantee a constant lower bound on the growth of the Gram matrix's minimum eigenvalue over a single round-robin cycle. Let $T_0$ denote the number of initial rounds required until the condition $\lambdamin(V_{t-1}) \ge B$ holds. 

\begin{lemma}[Increment of the minimum eigenvalue of the Gram matrix]\label{lem:mineigen_onecycle} 
    Suppose that Assumptions ~\ref{assump:Bdd}, ~\ref{assump:OD}, and ~\ref{assump:Good} hold. If the OLS estimator satisfies $\|\hatthetams-\truethetam\| \le \rglBall$, for all $m \in [M]$ and for all $s \ge T_0+1$, then the selected arms for a single cycle $s=t_0,~t_0+1,~ \ldots~,~t_0+M-1$ ($t_0 \ge T_0+1$) by $\mogrorr$ satisfy 
    \begin{equation*}
        \lambdamin\left(\sum_{s=t_0}^{t_0+M-1}x(s)x(s)^{\top}\right) \ge \left(\lambda- 2 \sqrt{2+ 2\rglBall \sqrt{1-\rglBound^2}-2\rglBound \sqrt{1-\rglBall^2}}\right)M.
    \end{equation*}
\end{lemma}
\begin{proof}
For $s=t_0, \ldots, t_0+M-1 ~(t_0 \ge T_0+1)$, $\|\hatthetams-\truethetam\| <\rglBall$ for all $m\in [M]$. Then, by Assumption~\ref{assump:Good} and Proposition~\ref{prop:good_optimal}, the selected arm $x(s)$ are $\rglBound$-good arms for the corresponding target objectives $m(s)$ in round $s=t_0, \ldots, t_0+M-1$. 
By Lemma~\ref{lem:dist_good_arm}, we can get $\|x(s) - \theta_{m(s)}^*\|_2 \le \sqrt{2+ 2\rglBall \sqrt{1-\rglBound^2}-2\rglBound\sqrt{1-\rglBall^2}}$ for all $m \in [M]$. 

For any unit vector $u\in \mathbb{B}^d$, the following holds 
\begin{align*}
    u^\top\left(\sum_{s=t_0}^{t_0+M-1}x(s)\big(x(s) \big)^{\top}\right)u &=\sum_{s=t_0}^{t_0+M-1}\left\langle u,x(s)\right\rangle^2\\
    &=\sum_{s=t_0}^{t_0+M-1}\left\langle u, \theta_{m(s)}^*+(x(s)-\theta_{m(s)}^*) \right\rangle ^2\\
    &=\sum_{s=t_0}^{t_0+M-1} \{\left\langle u, \theta_{m(s)}^* \right\rangle ^2+ \left\langle u, x(s)-\theta_{m(s)}^* \right\rangle ^2+2 \left\langle u, \theta_{m(s)}^*\right\rangle \left\langle u,x(s)-\theta_{m(s)}^* \right\rangle\}\\
    &\ge  u^\top\left(\sum_{s=t_0}^{t_0+M-1} \theta_{m(s)}^*\big(\theta_{m(s)}^*\big)^\top\right)u +0 -2 \sqrt{2+ 2\rglBall \sqrt{1-\rglBound^2}-2\rglBound\sqrt{1-\rglBall^2}}M \\
    &\ge \lambda M- 2 \sqrt{2+ 2\rglBall \sqrt{1-\rglBound^2}-2\rglBound\sqrt{1-\rglBall^2}}M.
\end{align*}

The last line holds because the target objectives in round $s=t_0, \ldots, t_0+M-1$ are all different $M$ objectives. Therefore, we have 
\begin{equation*}
    \lambdamin\left(\sum_{s=t_0}^{t_0+M-1}x(s)x(s)^{\top}\right) \ge \left(\lambda- 2 \sqrt{2+ 2\rglBall \sqrt{1-\rglBound^2}-2\rglBound \sqrt{1-\rglBall^2}}\right)M.
\end{equation*}
\end{proof}

\begin{lemma}[Minimum eigenvalue growth]\label{lem:mineigen} 
    Suppose Assumptions~\ref{assump:Bdd}, ~\ref{assump:OD}, and ~\ref{assump:Good} hold, and fix $\delta>0$. If we run $\mogrorr$ with $B= \min\left[{\sigma \over \rglBall  }\sqrt{2 dT  \log ({dT \over \delta})},~~{4\sigma^2  \over \rglBall^2} \left( {d \over 2} \log\left(1+{2T \over d}\right)+\log\left({1 \over \delta}\right)\right)\right]$, then with probability $1-2M\delta$, the following holds for the minimum eigenvalue of the Gram matrix 
    \begin{equation*}
        \lambdamin\left(\sum_{s=1}^{t-1} x(s)x(s)^\top\right) \ge B + {\lambda \over 3} (t-T_0-M),
    \end{equation*}
    for $T_0+M\le t \le T$.
\end{lemma}
\begin{proof}
If we choose $B$ as stated in the lemma, the OLS estimator satisfies $\|\hatthetams-\truethetam\| \le \rglBall$ for all $s \ge T_0+1$ and $m \in [M]$ with probability $1-2M\delta$, by Lemma~\ref{lem:set_B}.  
Thus, by applying Lemma~\ref{lem:mineigen_onecycle} to every single round after exploration, we have, for $t \ge T_0+M$,  
\begin{align*}
    \lambdamin\left(\sum_{s=1}^{t-1}x(s)x(s)^\top\right) &\ge  \lambdamin\left(\sum_{s=1}^{T_0}x(s)x(s)^\top\right)+\lambdamin\left(\sum_{s=T_0+1}^{t-1}x(s)x(s)^\top\right)\\
    &\ge B + \left[{t-1-T_0 \over M}\right] \times {\lambda \over 3}M,   \\
    &\ge B + {\lambda \over 3}(t-T_0-M).
\end{align*}
\end{proof}

With Lemma ~\ref{lem:mineigen}, we are ready to derive the $l_2$ bound of $\hatthetamt-\truethetam$ for $m \in [M]$.
\begin{lemma}
\label{lem:thetahat_bound_rr}
    Fix $\delta>0$. Under the same conditions as those in Lemma ~\ref{lem:mineigen}, with probability at least $1-3M\delta$, for all $m\in [M]$ and $t \ge 2T_0 +2M$, the OLS estimator $\hatthetamt$ of $\truethetam$  satisfies 
    \begin{equation*}
         \left\|\hatthetamt-\truethetam\right\|_2 \le {6\sigma\ \over \lambda} \sqrt{ d  \log (dt / \delta) \over {t-T_0-M}}.
    \end{equation*}
\end{lemma}
\begin{proof}
From the closed form of the OLS estimators, for all $m \in [M]$, 
\begin{align*}
    \left\|\hatthetamt-\truethetam\right\|_2 & = \left\|\left(\sum_{s=1}^{t-1} x(s)x(s)^\top\right)^{-1}\sum_{s=1}^{t-1} x(s) \eta_{a(s), m}(s)\right\|_2 \nonumber \\
    & \le {1 \over \lambdamin\left(\sum_{s=1}^{t-1} x(s)x(s)^\top\right)}\left\|\sum_{s=1}^{t-1} x(s) \eta_{a(s), m}(s)\right\|_2 
\end{align*}

For the denominator, we have $\lambdamin\left(V_{t-1}\right) \ge B + {\lambda \over 3}(t-T_0-M) $  for $t \ge T_0+M$, with probability at least $1-2M \delta$, by Lemma~\ref{lem:mineigen}. To bound the $l_2$ norm of $S_{t-1, m}:=\sum_{s=1}^{t-1} x(s) \eta_{a(s), m}(s)$, we can use Lemma~\ref{lem:Kannan_A1}, the martingale inequality of ~\citet{Kannan2018}. The lemma states that for fixed $m \in [M]$, $\|S_{t-1, m}\|_2 \le \sigma\sqrt{2 dt  \log (dt / \delta)}$ holds with probability at least $1- \delta $. Therefore, with probability at least $1- 3M\delta $, for all $m\in [M]$ and $t \ge 2T_0 +2M$,
\begin{equation*}
    \left\|\hatthetamt-\truethetam\right\|_2 \le  {\sigma\sqrt{ 2dt  \log (dt / \delta)} \over B + \lambda(t-T_0-M)/3} 
     \le {6\sigma\ \over \lambda} \sqrt{ d  \log (dt / \delta) \over {t-T_0-M}}.
\end{equation*}
The last inequality holds when $t \ge 2T_0 +2M$. 
\end{proof}


\subsection{Effective Pareto regret bound of $\mogrorr$} \label{ap_sec:pf_cor_EPR_rr}
The following corollary presents the effective Pareto regret of $\mogrorr$. It demonstrates that the randomness in objectives is not a source of implicit exploration. 
\begin{corollary}[Effective Pareto regret bound of $\mogrorr$ (formal version of Corollary~\ref{cor:EPR_rr})]\label{cor:EPR_rr_f} 
    Suppose that Assumptions ~\ref{assump:Bdd}, ~\ref{assump:OD}, and ~\ref{assump:Good} hold. If we run $\mogrorr$ with $B= \min\left\{{\sigma \over \rglBall  }\sqrt{2 dT  \log ({dT^2})},~~{4\sigma^2  \over \rglBall^2} \left( {d \over 2} \log\left(1+{2T \over d}\right)+\log\left({T}\right)\right)\right\}$, then the effective Pareto regret of $\mogrorr$ is upper-bounded by 
    \begin{equation*}
        \regret(T) \le {8 \sigma \over \lambda'}{\sqrt{ 2dT \log (dT)}}+4T_0+10M,
    \end{equation*}
    where $\lambda'=\lambda- 2 \sqrt{2+ 2\rglBall \sqrt{1-\rglBound^2}-2\rglBound \sqrt{1-\rglBall^2}}$.
\end{corollary}

\begin{proof} 
Let $E$ be the event that $\left\|\hatthetamt-\truethetam\right\|_2 \le {6\sigma\ \over \lambda} \sqrt{ d  \log (dtT) \over {t-T_0-M}}$ holds for all $m \in [M]$ and $t \ge 2T_0 + 2M$.
Then, $\mathbb{P}(\bar{E}) \le {3M \over T}$ by Lemma~\ref{lem:thetahat_bound_rr} with $\delta= {1 \over T}$.

Let $m(t)$ be the target objective for round $t$ and $a_m^*$ be the optimal arm with respect to objective $m$, hence it is effective Pareto optimal. Then, the suboptimality gap on round $t$ is bounded by 
\begin{equation*}
    \Delta_{a(t)}(t) \le \big(x_{a_{m(t)}^*}\big)^\top\theta_{m(t)}^*-x(t)^\top\theta_{m(t)}^* \le 2\| \hat{\theta}_{m(t)}(t)-\theta_{m(t)}^* \|_2.
\end{equation*}

Let $\Delta_{\max}$ be the maximum suboptimality gap. 
For $t \ge 2T_0 + 2M$, 
\begin{align*}
    \mathbb{E}[\Delta_{a(t)}(t)] &\le \mathbb{E}[\Delta_{a(t)}(t)~|~E]+ \mathbb{P}(E)\Delta_{\max}\\
    &\le 2 \mathbb{E}[ ~\| \hat{\theta}_{m(t)}(t)-\theta_{m(t)}^* \|_2~|~E]+ {3M \over T}\Delta_{\max}\\
    &\le {12\sigma\ \over \lambda} \sqrt{ d  \log (dtT) \over {t-T_0-M}}+ {3M \over T}\Delta_{\max}.
\end{align*}

Then, the Pareto regret is bounded by
\begin{align*}
    \regret(T) &=  \sum_{t= 2T_0+2M+1}^T \mathbb{E}[\Delta_{a(t)}(t)] + (2T_0+2M)\Delta_{\max}\\
    & \le  \sum_{t= 2T_0+2M+1}^T {12\sigma\ \over \lambda} \sqrt{ d  \log (dtT) \over {t-T_0-M}}+ \{({3M \over T})T+2T_0+2M  \}\Delta_{\max}  \\
    & \le  {12\sigma \over \lambda} \sqrt{2 d  \log (dT)}  \int_0^T {1 \over \sqrt{t}}dt+ \{2T_0+5M \}\Delta_{\max}  \\
    & \le {24\sigma \over \lambda}{\sqrt{ 2dT \log (dT)}}+2\{2T_0+5M\}. 
\end{align*}
The last inequality holds because we have $\Delta_{\max} \le 2$ under Assumption ~\ref{assump:Bdd}.  
\end{proof}


\section{Stochastic contexts setup}
\label{ap_sec:analysis_stochastic}
We verified that our proposed algorithms are statistically efficient even in stochastic context settings. In this section, we demonstrate the effective Pareto regret bound and effective Pareto fairness of $\mogrorw$ algorithm in a stochastic context setting. Notably, $\mogrorr$ can also be analyzed theoretically using the same approach.

\subsection{Problem setting}
\label{ap_subsec:MOLB_stochastic}
In multi-objective linear contextual bandit under stochastic contexts setup, the set of feature vectors $\chi(t)=\{x_i(t) \in \mathbb{R}^d, i \in [K]\}$ is drawn from some unknown distribution $P_{\chi}(t)$ in each round $t=1,\ldots,T$. Each arm's feature $x_i(t) \in \chi(t)$  for $i \in [K]$ need not be independent of each other and can possibly be correlated. In this case, we denote $x_{a(t)}(t)$ as $x(t)$. Other settings are identical to the fixed arms case in Section~\ref{subsec:MOLB}.

\paragraph{Effective Pareto regret metric}
Effective Pareto regret can be defined in the same way as in the fixed-arm case~\citep{park2025thompson}. The key difference is that in the fixed-arm setting, each arm’s expected reward remains constant over time, and hence the effective Pareto front does not change. In contrast, in the contextual setup, the expected reward of each arm varies across rounds, and consequently the effective Pareto front also evolves. Therefore, the definition of effective Pareto regret is taken with respect to the effective Pareto front at each round. 

\begin{definition} [Effective Pareto front]
    Let $\mu_i (t)\in \RR^M$ be the expected reward vector of arm $i \in [K]$ in round $t$. An arm is effective Pareto optimal (denoted $a_*$) in round $t$ if its mean reward vector is either equal to or not dominated by any convex combination of the mean reward vectors of the other arms. Formally, for any $\beta:= (\beta_i)_{i\in[K]\setminus\{a_*\}}\in\mathbb{S}^{K-2}$, 
\begin{equation*}
\mu_{a_*}(t) = \sum_{i\in [K]\setminus\{a_*\}}\beta_{i}\mu_{i}(t) ~~\text{ or } ~~\mu_{a_*}(t) \not\prec \sum_{i\in[K]\setminus\{a_*\}}\beta_{i}\mu_{i}(t).
\end{equation*} 
The set of all effective Pareto optimal arms is called the effective Pareto front, denoted as $\Ccal^*(t)$. 
\end{definition}

\begin{definition} [Effective Pareto regret]\label{def:EPR_stochastic}
    For each vector $w = (w_{i})_{i\in[K]} \in\mathbb{S}^{M-1}$ and $t\in [T]$, let $\mu_{w}(t) := \sum_{i\in[K]}w_{i}\mu_{i}(t)$. Then, the \textbf{effective Pareto suboptimality gap} $\Delta_{i}(t)$ for arm $i \in [K]$ is defined as the infimum of the scalar $\epsilon \ge 0$ such that $\mu_i(t)$ becomes effective Pareto optimal arm after adding $\epsilon$ to all entries of its expected reward. Formally, 
    \begin{equation*}
        \Delta_{i}(t):= \inf\Big\{\epsilon \ge 0 \mathrel{\bigg|} \mu_{a_t}(t) +\epsilon \one \not\prec\mu_{w}(t),\forall w\in\mathbb{S}^{K-1}\Big\}.
    \end{equation*}
    Then, the cumulative \textbf{effective Pareto regret} is defined as  $\regret(T):=\sum_{t=1}^T \mathbb{E}[\Delta_{a(t)}(t)]$, where $ \mathbb{E}[\Delta_{a(t)}(t)]$ represents the expected effective Pareto suboptimality gap of the arm pulled at round $t$. 
\end{definition}

\paragraph{Effective Pareto fairness}
Effective Pareto fairness can also be defined in the same way as in the fixed-arm case. For each round, the Effective Pareto fairness index is defined with respect to the Effective Pareto optimal arm for each objective, which may vary over time.

\begin{definition} [Effective Pareto fairness] \label{def:EPF_stochastic}
     Given a weight vector $w=(w_1,\ldots,w_M) \in \mathbb{S}^{M-1}$, let $\mu_{i, w}(t):=\sum_{m\in [M]}w_mx_i(t)^\top\theta_m^*$ be the expect weighted reward of arm $i$ in round $t$, $a_w^*(t)$ be the arm that has the largest expected weighted reward with respect to the weight vector $w$, and $\mu_w^*(t):= \mu_{a_w^*(t), m}$. For all $\epsilon>0$, define \textbf{the effective Pareto fairness index} $\textnormal{EPFI}_{\epsilon,T}$ of an algorithm as
    \begin{equation*}
        \textnormal{EPFI}_{\epsilon,T}:=\inf_{ w \in \mathbb{S}^{M-1}}\left({1 \over T}\mathbb{E}\left[\sum_{t=1}^T \mathds{1}\{\mu_{w}^*(t)-\mu_{a(t), w}(t)<\epsilon\}\right]\right).
    \end{equation*}   
    Then, we say that the algorithm satisfies the \textbf{effective Pareto  fairness} if for given $\epsilon$, there exists a positive lower bound $L_\epsilon$ that satisfies $\lim_{T \rightarrow \infty}\textnormal{EPFI}_{\epsilon,T} \ge L_\epsilon$. 
\end{definition}

\subsection{Effective Pareto regret bound of $\mogrorw$ with stochastic contexts}
\label{ap_subsec:EPR_stochastic}

To analyze $\mogrorw$ under stochastic setup, the following assumption is essential to guarantee that the feature vectors in round $t$ are not influenced by previous rounds $s=1, \ldots, t-1$. 
\begin{assumption}[Independently distributed contexts]\label{assump:indep}
    The context sets $\chi(1), \ldots , \chi(T)$, drawn from unknown distribution $P_{\chi}(1), \ldots ,P_{\chi}(T)$ respectively, are independently distributed across time. 
\end{assumption}
All of the greedy linear contextual bandit with stochastic contexts assumes the independence of context sets. It is important to note that feature vectors within the same round are allowed to be dependent, even under Assumption~\ref{assump:indep}. 

As in the regret bounds of $\mogrorw$ and $\mogrorr$ in fixed arm setting, the key to deriving the regret bound of $\mogrorw$ in the stochastic contextual setup is to establish the linear growth of the minimum eigenvalue of the Gram matrix.

\begin{lemma}[Increment of the minimum eigenvalue of the Gram matrix]\label{lem:mineigen_oneround_stochastic} 
    Suppose Assumptions ~\ref{assump:Bdd}, ~\ref{assump:OD}, ~\ref{assump:Good_stochastic}, and ~\ref{assump:indep} hold. If the OLS estimator satisfies $\|\hatthetams-\truethetam\| \le {\rglBall \over 2}$, for all $m \in [M]$ and $s \ge T_0+1$, then the arm selected by $\mogrorw$ satisfies
    \begin{align*}
       &\lambda_{\min}(\mathbb{E}[x(s)x(s)^\top| \mathcal{H}_{s-1}]) 
       \ge { \left(\lambda- 2 \sqrt{2+ 2\rglBall \sqrt{1-\rglBound^2}-2\rglBound \sqrt{1-\rglBall^2}}\right)q_\rglBound\phi_{\rglBall/2, \Wcal} M},
    \end{align*}
    where $q_\rglBound$ is defined in Definition~\ref{def:goodness_stochastic}. 
\end{lemma}
\begin{proof}
For $s \ge T_0 +1$ and $m \in [M]$, let $E_{\bar{m}}(s)$ be the event that the weighted objective $\sum_{m \in [M]}w_m(s)\theta_m^*$ in round $s$ satisfies $\left\|\sum_{m\in [M]}w_m(s)\theta_m^*-\theta_{\bar{m}}^*\right\|_2 < \rglBall/2$, where $w_m(s) \sim \Wcal$ is the weight for round $s$. We also define $R(s)$ as the event that there exists $\gamma$-good arm for the weighted objective $\sum_{m \in [M]}w_m(s)\theta_m^*$ in round $s$. 

On $E_{\bar{m}}(s)$, if $\|\hatthetams-\truethetam\| \le \rglBall/2$ holds, than the following holds
\begin{align*}
    \left\|\sum_{m \in [M]}w_m\hatthetamt-\truethetam\right\|_2
     &\le 
    \left\|\sum_{m \in [M]}w_m\hatthetamt-\sum_{m \in [M]}w_m\theta_m^*\right\|_2
    +   \left\|\sum_{m \in [M]}w_m\theta_m^*-\truethetam\right\|_2\\
    & \le \sum_{m \in [M]}w_m\left\|\hatthetamt-\theta_m^*\right\|_2 +  \left\|\sum_{m \in [M]}w_m\theta_m^*-\truethetam\right\|_2 \\
    & < {\rglBall \over2} +{\rglBall \over2} =\rglBall. 
\end{align*}
Thus, by Assumption~\ref{assump:Good_stochastic}, $\mathbb{P}[E_{\bar{m}}(s)R(s)]\ge q_\rglBound\phi_{\alpha/2, \Wcal}$. 

Since the arm selected by $\mogrorw$ satisfies
\begin{align} \label{eq:mineigen_oneround_stochastic}
    \mathbb{E}[x(s)x(s)^\top| \mathcal{H}_{s-1}] &= \sum_{m=1}^M \mathbb{E}[x(s)x(s)^\top| E_m(s), R(s), \mathcal{H}_{s-1}] \mathbb{P}[E_m(s), R(s)| \mathcal{H}_{s-1}] \nonumber\\
    &\succeq q_\rglBound \phi_{\rglBall/2, \Wcal} \sum_{m=1}^M  \mathbb{E}[x(s)x(s)^\top| E_m(s), \mathcal{H}_{s-1}] \nonumber\\
    &\succeq q_\rglBound \phi_{\rglBall/2, \Wcal} \sum_{m=1}^M \mathbb{E}[x(s)| E_m(s), \mathcal{H}_{s-1}]\mathbb{E}[x(s)| E_m(s), \mathcal{H}_{s-1}]^\top. 
\end{align}

Let $x_{r(m)}:=\mathbb{E}[x(s)| E_m(s), R(s), \mathcal{H}_{s-1}]$. Then,   $x_{r(m)}$ is the expectation of the selected arm for the target objective $\sum_{m \in [M]}w_m\hatthetams$ in round $s$ when it satisfies $\left\|\sum_{m \in [M]}w_m\hatthetams -\truethetam\right\|_2<\rglBall$.
Since there always exist $\rglBound$-good arms for the weighted objective $\sum_{m \in [M]} w_m \hatthetams$ on $E_m(S)$ for all $m \in [M]$, the selected arms are also $\rglBound$-good (Proposition~\ref{prop:good_optimal}), and hence so is $x_{r(m)}$ (Proposition~\ref{prop:expectation_good}).
Thus, by Lemma~\ref{lem:dist_good_arm}, we can get $\|{x_{r(m)}} - \theta_{m}^*\|_2 \le \sqrt{2+ 2\rglBall \sqrt{1-\rglBound^2}-2\rglBound\sqrt{1-\rglBall^2}}$ for all $m \in [M]$. 

Now, we will bound $\lambdamin\left(\sum_{m \in [M]}x_{r(m)}\big(x_{r(m)} \big)^{\top}\right)$. For any unit vector $u\in \mathbb{B}^d$, the following holds 
\begin{align*}
    u^\top\left(\sum_{m \in [M]}x_{r(m)}\big(x_{r(m)} \big)^{\top}\right)u &=\sum_{m \in [M]}\left\langle u,x_{r(m)}\right\rangle^2\\
    &=\sum_{m \in [M]}\left\langle u, \theta_{m}^*+(x_{r(m)}-\theta_{m}^*) \right\rangle ^2\\
    &=\sum_{m \in [M]} \{\left\langle u, \theta_{m}^* \right\rangle ^2+ \left\langle u, x_{r(m)}-\theta_{m}^* \right\rangle ^2+2 \left\langle u, \theta_{m}^*\right\rangle \left\langle u,x_{r(m)}-\theta_{m}^* \right\rangle\}\\
    &\ge  u^\top\left(\sum_{m \in [M]} \truethetam \big(\truethetam \big)^\top\right)u +0 -2 \sqrt{2+ 2\rglBall \sqrt{1-\rglBound^2}-2\rglBound\sqrt{1-\rglBall^2}}M \\
    &\ge \lambda M- 2 \sqrt{2+ 2\rglBall \sqrt{1-\rglBound^2}-2\rglBound\sqrt{1-\rglBall^2}}M.
\end{align*}

This implies 
\begin{equation*}
    \lambdamin\left(\sum_{m \in [M]}x_{r(m)}\big(x_{r(m)} \big)^{\top}\right) \ge \left(\lambda- 2 \sqrt{2+ 2\rglBall \sqrt{1-\rglBound^2}-2\rglBound \sqrt{1-\rglBall^2}}\right)M.
\end{equation*}

Plugging into Eq.~\eqref{eq:mineigen_oneround_stochastic}, we obtain
\begin{align*}
    \lambda_{\min}(\mathbb{E}[x(s)x(s)^\top| \mathcal{H}_{s-1}]) &\ge  q_\gamma\phi_{\rglBall/2, \Wcal}~ \lambdamin\left(\sum_{m \in [M]}x_{r(m)}\big(x_{r(m)} \big)^{\top}\right) \\
    &\ge  \left(\lambda- 2 \sqrt{2+ 2\rglBall \sqrt{1-\rglBound^2}-2\rglBound \sqrt{1-\rglBall^2}}\right)q_\gamma{\phi_{\rglBall/2, \Wcal}M}.
\end{align*}
\end{proof}

Then, the regret bound can then be derived by the same logic as in the proof of Theorems ~\ref{thm:EPR} and ~\ref{cor:EPR_rr}. The following corollary demonstrates that the $\mogrorw$ algorithm also possesses a $\tilde{\Ocal}{(\sqrt{T})}$-regret bound in the case of stochastic contexts. 

\begin{corollary}[Pareto Regret Bound of $\mogrorw$ with Stochastic Contexts] \label{cor:EPR_stochastic}
    Suppose Assumptions ~\ref{assump:Bdd}, ~\ref{assump:OD}, ~\ref{assump:Good_stochastic}, and ~\ref{assump:indep} hold. If we run $\mogrorw$ with $B= \min\big[{2\sigma \over \rglBall  }\sqrt{2 dT  \log ({dT^2})},$ $~~{16\sigma^2  \over \rglBall^2} \left( {d \over 2} \log\left(1+{2T \over d}\right)+\log\left({T}\right)\right)\big]$, the effective Pareto regret of $\mogrorw$ is bounded by 
    \begin{equation*}
        \regret(T) \le {16 \sigma \over \lambdainc'}{\sqrt{ 2dT \log (dT)}}+4T_0+6M+{20d \over \lambdainc'}.
    \end{equation*}
    where $\lambdainc':=\left(\lambda- 2 \sqrt{2+ 2\rglBall \sqrt{1-\rglBound^2}-2\rglBound \sqrt{1-\rglBall^2}}\right)q_\gamma{\phi_{\rglBall/2, \Wcal}M}$ and $q_\rglBound$ in Definition~\ref{def:goodness_stochastic}.    
\end{corollary}

Additionally, in a stochastic setup, $T_0$ can also be bounded at a scale of $\Ocal(B)$ with high probability.

\subsection{Effective Pareto fairness of $\mogrorw$ with stochastic contexts}
\label{ap_subsec:OF_stochastic}

The following corollary implies that $\mogrorw$ satisfies objective fairness. 

\begin{corollary}[Effective Pareto Fairness of $\mogrorw$ with Stochastic Contexts] \label{cor:EPF_stochastic}
    Suppose Assumptions ~\ref{assump:Bdd}, ~\ref{assump:OD}, ~\ref{assump:Good_stochastic}, and ~\ref{assump:indep} hold. Then, the effective Pareto fairness index of $\mogrorw$ satisfies    
    \begin{equation*}
        \textnormal{EPFI}_{\epsilon,T} \ge 
       \psi_{\epsilon/3,\Wcal}\left({T-T_\epsilon \over T}\right) \left(1-{3M \over T}-d\left({1 \over dT}\right)^{120\sigma^2 d \over \lambdainc' \epsilon^2}\right),
    \end{equation*}    
    where $T_\epsilon=\max(\lfloor{1152\sigma^2 d\log(dT) \over {(\lambdainc')}^2\epsilon^2}\rfloor+T_0,~~2T_0)$,
    in the same setting as Corollary~\ref{cor:EPR_stochastic}.
\end{corollary}

Notably, for any given $\epsilon>0$, $\lim_{T \rightarrow \infty}\text{EPFI}_{\epsilon,T} \ge \psi_{\epsilon/3,\Wcal}$. 

\section{Relaxation of the boundedness assumption}
\label{ap_sec:releasing_Bdd}

In this section, we explain how to release the boundedness assumption, Assumption~\ref{assump:Bdd}. In conclusion, we can obtain results of the same scale as Theorems~\ref{thm:EPR} and ~\ref{thm:EPF} for any arbitrary bound $\|x_i\|_2 \le x_{\max}$ and $l \le \truethetam \le L$ for all $m \in [M]$. For clarity, we will separately discuss how to release the $l_2$ norm bounds on the feature vector and the objective parameters in Appendix ~\ref{ap_subsec:releasing_Bdd_x} and ~\ref{ap_subsec:releasing_Bdd_o}, respectively. However, It is important to note that there is no issue in applying the same argument even when the bound on the feature vectors and the bound on the objective parameters are released simultaneously. We present how to release the boundedness assumption in fixed features setting, but the same reasoning can be applied to the case of stochastic contexts.

\subsection{Releasing bound on feature vectors}\label{ap_subsec:releasing_Bdd_x}
We demonstrate how the minimum eigenvalue of the Gram matrix can increase linearly when the $l_2$ norm of the feature vectors is bounded by an arbitrary upper bound $x_{\max}$. Since the $\rglBound$-goodness assumption is related to the scale of the feature, we modify the $\rglBound$-goodness assumption correspondingly.
\begin{assumption} [Boundedness] \label{assump:Bdd_release_x}
    For all $i \in [K]$ and $m \in [M]$,  $\|x_i\|_2 \le x_{\max}$ and $\|\truethetam\|_2 = 1$. 
\end{assumption}
\begin{assumption} [$\rglBound$-Goodness]\label{assump:Good_release_x}    We assume $\{x_1, \ldots, x_K\}$ satisfies $\rglBound$-goodness with $\rglBound > {x_{\max}\over \lambda}\sqrt{2\sqrt{1+\lambda^2}-2}$.
\end{assumption}
The following lemma is the key to the release process. 

\begin{lemma} 
\label{lem:dist_good_arm_release_x}
    Given Assumptions  ~\ref{assump:Bdd_release_x}, the target objective estimator $\targetthetas :=\sum_{m\in[M]}w_m\hatthetams$ satisfies $\|\targetthetas-\truethetam\| \le \rglBall$ for some $m \in [M]$ and $s \ge 1$.
    If $x \in \mathbb{B}_{\xmax}^d$ is $\rglBound$-good for $\targetthetas$, then the distance between $x \over \rglBound$ and $\truethetam$ is bounded by 
    \begin{equation*}
        \left\|\truethetam-{x \over \rglBound}\right\|_2 \le \sqrt{1+({\xmax \over \rglBound})^2 + 2 \rglBall \sqrt{({\xmax \over \rglBound})^2-1}-2\sqrt{1-\rglBall^2}}.
    \end{equation*}
\end{lemma}
\begin{proof}
Consider the case when $x \over \rglBound$ is the farthest from $\truethetam$.
\begin{figure}[t]
\begin{center}
    \includegraphics[width=0.6\textwidth, trim=2cm 3.5cm 6cm 3cm, clip]{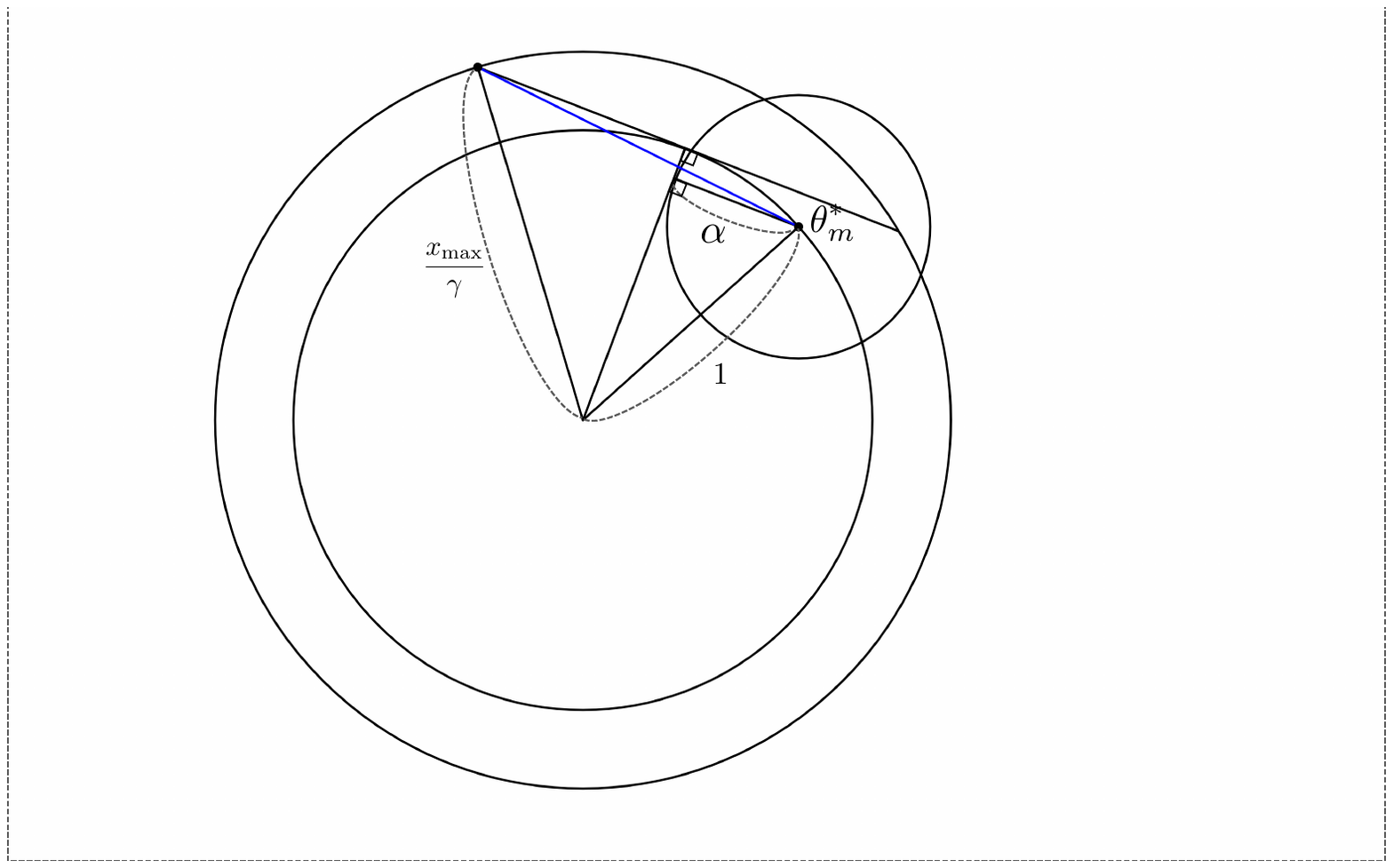} 
    \caption{The interior of the circle with radius ${x_{\max} \over \rglBound}$ represents the region where $x \over \rglBound$ may exist in $\RR^d$, while that of the smallest circle indicates the region where $\targetthetas$ may exist. Then, the blue line illustrates the case when $x \over \rglBound$ is farthest from the $\truethetam$.}
    \label{fig:near_optimal_zone_xmax}
\end{center}
\end{figure}
As we easily can see from Figure~\ref{fig:near_optimal_zone_xmax}, 
\begin{align*}
        \left\|\truethetam-{x \over \rglBound}\right\|_2^2 &\le \left(\rglBall + \sqrt{({\xmax \over \rglBound})^2-1}\right)^2 + (1-\sqrt{1-\rglBall^2})^2 \\
        &= 1+({\xmax \over \rglBound})^2 + 2 \rglBall \sqrt{({\xmax \over \rglBound})^2-1}-2\sqrt{1-\rglBall^2}. 
\end{align*}
\end{proof}

\begin{lemma}[Increment of the minimum eigenvalue of the Gram matrix]\label{lem:mineigen_oneround_release_x} 
    Suppose Assumptions ~\ref{assump:Bdd_release_x}, ~\ref{assump:OD}, and ~\ref{assump:Good_release_x} hold. If the OLS estimator satisfies $\|\hatthetams-\truethetam\| \le {\rglBall \over 2}$, for all $m \in [M]$ and $s \ge T_0+1$, then the arm selected by $\mogrorw$ satisfies
    \begin{equation*}
       \lambda_{\min}(\mathbb{E}[x(s)x(s)^\top| \mathcal{H}_{s-1}]) \ge \left(\lambda \rglBound^2- 2\xmax \sqrt{\rglBound^2 + \xmax^2+2\rglBall\sqrt{\xmax^2-\rglBound^2}-2\rglBound^2\sqrt{1-\rglBall^2}}\right)\phi_{\rglBall/2, \Wcal} M
    \end{equation*}
\end{lemma}
\begin{proof}
For $s \ge T_0 +1$ and $m \in [M]$, let $E_{\bar{m}}(s)$ be the event that the weighted objective $\sum_{m \in [M]}w_m\theta_m^*$ in round $s$ satisfies $\left\|\sum_{m\in [M]}w_m\theta_m^*-\theta_{\bar{m}}^*\right\|_2 < \rglBall/2$. From Equation~\ref{eq:mineigen_oneround}, we have
\begin{equation*}
    \mathbb{E}[x(s)x(s)^\top| \mathcal{H}_{s-1}] \succeq  \phi_{\rglBall/2, \Wcal} \sum_{m=1}^M \mathbb{E}[x(s)| E_m(s), \mathcal{H}_{s-1}]\mathbb{E}[x(s)| E_m(s), \mathcal{H}_{s-1}]^\top.   
\end{equation*}

Let $x_{r(m)}:=\mathbb{E}[x(s)| E_m(s), \mathcal{H}_{s-1}]$. Then, $x_{r(m)}$ is $\rglBound$-good for all $m\in[M]$ (see the proof of Lemma~\ref{lem:mineigen_oneround}).

Now, we will bound $\lambdamin\left(\sum_{m \in [M]}x_{r(m)}\big(x_{r(m)} \big)^{\top}\right)$. 

By Lemma~\ref{lem:dist_good_arm_release_x}, $\left\|\truethetam-{x_{r(m)} \over \rglBound}\right\|_2 \le \sqrt{ 1+({\xmax \over \rglBound})^2 + 2 \rglBall \sqrt{({\xmax \over \rglBound})^2-1}-2\sqrt{1-\rglBall^2}}$ holds for all $m \in [M]$. Then, 
\begin{align*}
     \lambdamin\left(\sum_{m \in [M]}x_{r(m)}\big(x_{r(m)} \big)^{\top}\right) & = {\rglBound^2}\lambdamin\left(\sum_{m \in [M]}{x_{r(m)} \over \rglBound}\left({x_{r(m)} \over \rglBound} \right)^{\top}\right) \\
     &\ge \rglBound^2 \left[ \lambda_{\max}\left(\sum_{m \in [M]} \truethetam (\truethetam)^\top \right) - 2M \left( {x_{\max} \over \rglBound}\right) \left\| \truethetam - {x _{r(m)} \over \rglBound}\right\|\right]\\
     &\ge \lambda\rglBound^2 -2x_{\max} \sqrt{\rglBound^2 + \xmax^2+2\rglBall\sqrt{\xmax^2-\rglBound^2}-2\rglBound^2\sqrt{1-\rglBall^2}}.
\end{align*}

Therefore, we obtain
\begin{align*}
    \lambda_{\min}(\mathbb{E}[x(s)x(s)^\top| \mathcal{H}_{s-1}]) &\ge  \phi_{\rglBall/2, \Wcal}~ \lambdamin\left(\sum_{m \in [M]}x_{r(m)}\big(x_{r(m)} \big)^{\top}\right) \\
    &\ge  \left(\lambda\rglBound^2 -2x_{\max} \sqrt{\rglBound^2 + \xmax^2+2\rglBall\sqrt{\xmax^2-\rglBound^2}-2\rglBound^2\sqrt{1-\rglBall^2}}\right){\phi_{\rglBall/2, \Wcal}M}.
\end{align*}
\end{proof}

The above lemma means that even when Assumptions~\ref{assump:Bdd} and ~\ref{assump:Good} are replaced by Assumptions~\ref{assump:Bdd_release_x} and~\ref{assump:Good_release_x}, respectively, we can still obtain a regret bound that differs by at most a constant factor. Furthermore, using the same argument as before, we can also verify the objective fairness with replaced assumptions.


\subsection{Releasing bound on objective parameters}\label{ap_subsec:releasing_Bdd_o}
In this section, we present how to handle objective parameters with varying $l_2$ norm sizes. The $\rglBound$-goodness assumption is related to the scale of the objectives either, the  $\rglBound$-goodness assumption is modified again correspondingly.
\begin{assumption} [Boundedness] \label{assump:Bdd_release_o}
    For all $i \in [K]$ and $m \in [M]$,  $\|x_i\|_2 \le 1$ and $l \le \|\truethetam\|_2 \le L$. 
\end{assumption}

\begin{assumption} [$\rglBound$-Goodness]\label{assump:Good_release_o}
    We assume $\{x_1, \ldots, x_K\}$ satisfies $\rglBound$-goodness with $\rglBound > 1-{\lambda^2 \over 8L^4}$.
\end{assumption}
The following lemma is the key to the release process. 
\begin{lemma} 
\label{lem:NearOptimalZone_lL}
    Given Assumptions ~\ref{assump:Bdd_release_o}, assume the target objective estimator $\targetthetas :=\sum_{m\in[M]}w_m\hatthetams$ satisfies $\|\targetthetas-\truethetam\| \le \rglBall$ for some $m \in [M]$ and $s \ge 1$.
    If $x \in \mathbb{B}^d$ is $\rglBound$-good for $\targetthetas$, then the distance between $x $ and ${\truethetam\over \|\truethetam\|_2}$ is bounded by 
    \begin{equation*}
        \left\|{\truethetam \over \|\truethetam\|_2}-x\right\|_2 \le \sqrt{2+ {2\rglBall \over l} \sqrt{1-\rglBound^2}-2\rglBound\sqrt{1-{\rglBall^2 \over l^2}}}.
    \end{equation*}
\end{lemma}
\begin{proof}
Consider the case when $x$ is the farthest from ${\truethetam\over \|\truethetam\|_2}$.
\begin{figure}[ht]
\begin{center}
    \includegraphics[width=0.7\textwidth, trim=2cm 4cm 6cm 3cm, clip]{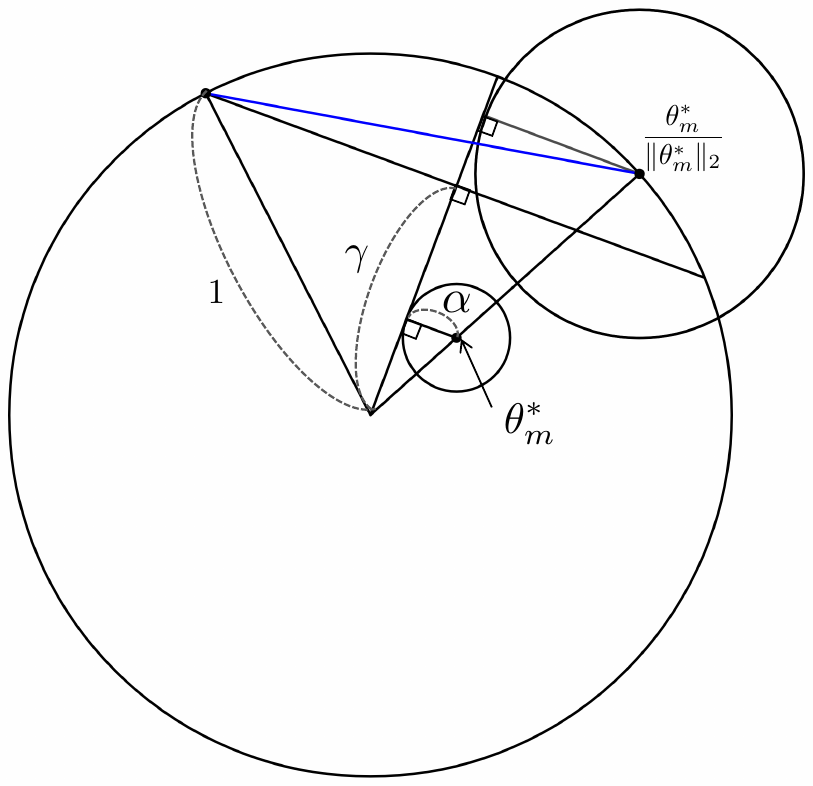} 
    \caption{The larger circle represents the unit sphere in $\RR^d$ while the interior of the smallest circle indicates the region where $\targetthetas$ may exist. Then, the blue line illustrates the case when $x$ is farthest from the $\truethetam \over \|\truethetam\|_2$.}
    \label{fig:near_optimal_zone_lL}
\end{center}
\end{figure}
As we easily can see from Figure~\ref{fig:near_optimal_zone_lL}, we can obtain the following result from Lemma~\ref{lem:dist_good_arm} by replacing $\rglBall$ by ${\rglBall \over l}$. 
\begin{align*}
        \left\|{\truethetam \over \|\truethetam\|_2}-x\right\|_2 \le \sqrt{2 \left(1+ \left({\rglBall \over l}\right) \sqrt{1-\rglBound^2}-\rglBound\sqrt{1-\left({\rglBall \over l}\right)^2}\right)}.
\end{align*}
\end{proof}

\begin{lemma}[Increment of the minimum eigenvalue of the Gram matrix]\label{lem:mineigen_oneround_release_o}
    Suppose Assumptions \ref{assump:OD}, ~\ref{assump:Bdd_release_o}, and ~\ref{assump:Good_release_o} hold. If the OLS estimator satisfies $\|\hatthetams-\truethetam\| \le {\rglBall \over 2}$, for all $m \in [M]$ and $s \ge T_0+1$, then the arm selected by $\mogrorw$ satisfies
    \begin{equation*}
       \lambda_{\min}(\mathbb{E}[x(s)x(s)^\top| \mathcal{H}_{s-1}]) \ge \left({\lambda \over L^2}- 2 \sqrt{2+ {2\rglBall \over l} \sqrt{1-\rglBound^2}-2\rglBound \sqrt{1-{\rglBall^2 \over l^2}}}\right)\phi_{\rglBall/2, \Wcal} M.
    \end{equation*}
\end{lemma}
The above lemma can be derived from  $\lambdamin\left({1 \over M}\sum_{m=1}^M\big({\truethetam \over \|\truethetam\|_2}\big)\big({\truethetam \over \|\truethetam\|_2}\big)^\top \right) \ge {\lambda \over L^2}$ and Lemma~\ref{lem:NearOptimalZone_lL}. 

Therefore, we can still obtain a regret bound that differs by at most a constant factor and the objective fairness criterion when Assumptions~\ref{assump:Bdd} and ~\ref{assump:Good} are replaced by Assumptions~\ref{assump:Bdd_release_o} and~\ref{assump:Good_release_o}, respectively.

\section{Relaxation of Assumption~\ref{assump:OD}}
\label{ap_sec:releasing_OD}

Until now, we have conducted the analysis under the assumption that the feature vectors span $\mathbb{R}^d$.  In this section, we present a sufficient condition under which our proposed algorithms perform well when the feature vectors do not span $\RR^d$ and explain how this leads to regret bounds and fairness.

\textbf{Intuition.} It is evident that any bandit algorithm cannot obtain information about the true objective parameters in the direction of $S_x^\perp$ while interacting with feature vectors $x_1, \ldots, x_K$. In other words, during the process of estimating the objective parameters, no estimator can converge to the true parameters in the direction of space $S_x^\perp$.
Interestingly, from the perspective of regret and optimality, this poses no significant issue. This can be expressed mathematically as for any pair of arms $i, j \in [K]$ and $m \in [M]$,
\begin{equation*}
    x_{i}^\top\truethetam-x_{j}^\top\truethetam=x_{i}^\top(\pi_S (\truethetam))-x_{j}^\top(\pi_S (\truethetam)).
\end{equation*}
The above equation explains why regret and optimality are determined solely by the projection vector of the objective parameters onto $S_x$.

\begin{algorithm}[t]
\caption{Multi-Objective -- Greedy with Randomized Objective algorithm (General version of \mogro)}
\label{alg:mogro-s}
\begin{algorithmic}[1]
    \STATE \textbf{Input:} Total rounds $T$, Eigenvalue threshold $B$
    \STATE \textbf{Initialization:} $V_0 \leftarrow 0 \times I_d$, $S:$ feature basis
    \FOR{$t = 1, \ldots, T$}
        \IF{$\min_{\|\beta\|=1,~ \beta \in S_x}\left( { \sum_{s=1}^{t-1} \left\langle \beta,~ x(s) \right\rangle^2 }\right) < B$}
            \STATE Select action $a(t) \in S$ in round-robin manner
        \ELSE 
            \STATE Update the estimators $\Theta_t=(\hat{\theta}_1(t), \ldots, \hat{\theta}_M(t)) $ by Equation~\ref{eq:thetaupdate}
            \STATE Select action $a(t) \leftarrow \text{GRO}(\{x_i\}_{i\in[K]}, \Theta_t)$ 
        \ENDIF
    \STATE Observe $y(t)=\big(y_{a(t),1}(t),\ldots,y_{a(t),M}(t)\big)$
    \STATE Update $V_t \leftarrow V_{t-1}+x(t)x(t)^\top$      
\ENDFOR
\end{algorithmic}
\end{algorithm}

Algorithm~\ref{alg:mogro-s} provides a general formulation of the $\mogro$ algorithm for use when the feature vectors do not span $\mathbb{R}^d$. In this case, it is impossible to satisfy $\lambda_{\min}(V_{t-1}) > B ~( > 0)$, which is the terminate condition of initial exploration phase stated in Algorithm~\ref{alg:mogro}. Therefore, the initial exploration criterion should be modified for the case when the feature vectors do not span $\mathbb{R}^d$. Instead of $\lambda_{\min}(V_{t-1}) > B$, we can use $\min_{\|\beta\|=1,~ \beta \in S_x}\left( { \sum_{s=1}^{t-1} \left\langle \beta,~ x(s) \right\rangle^2 }\right)> B$. Additionally, under this condition, a unique least squares solution no longer exists. Therefore, for each round $t$, we use an arbitrary solution $\hatthetamt$ of the equation 
\begin{equation}\label{eq:thetaupdate}
    \big(\sum_{s=1}^{t-1}x(s)x(s)^\top \big)~\theta = \sum_{s=1}^{t-1}x(s)y_{a(s),m}(s).
\end{equation}

Notably, at least one solution exists after initial phase since $x(1), \ldots, x(t-1)$ span $S_x$. By extending Algorithm~\ref{alg:mogro} in this way, we can conduct the same analysis as before. 

The following present the revised versions of Assumptions~\ref{assump:Bdd}, \ref{assump:OD}, and \ref{assump:Good} for the case where the feature vectors are not required to span $\RR^d$.

\begin{assumption} [Boundedness] \label{assump:Bdd_projection}
      For all $i \in [K]$ and $m \in [M]$,  $\|x_i\|_2 \le 1$ and $\|\pi_s(\truethetam)\|_2 = 1$ hold. 
\end{assumption}
Once again, the above assumption is intended for a clear analysis. The analyses conducted in this section can be also extended to arbitrary bounds $\|x_i\|_2 \le x_{\max}$ and $l \le \pi_s(\truethetam) \le L$ for all $m \in [M]$ by the same process in Appendix ~\ref{ap_sec:releasing_Bdd}.
\begin{assumption} \label{assump:OD_projection}
    We assume $\theta_1^*$, \ldots, $\theta_M^*$ span $S_x$.
\end{assumption}
 In the following analysis, we define $\lambda_1:=\min_{\|\beta\|=1,~\beta \in S_x}\left({1 \over M} { \sum_{m=1}^M \left\langle \beta,~ \truethetam \right\rangle^2  }\right)$. Then, given Assumption~\ref{assump:OD_projection}, $\lambda_1$ is always positive and clearly, $\lambda_1=\min_{\|\beta\|=1,~\beta \in S_x}\left({1 \over M} { \sum_{m=1}^M \left\langle \beta,~ \pi_S(\truethetam) \right\rangle^2  }\right)$.

Next, we reconsider how to define $\rglBound$-goodness. If the feature vectors do not span $\mathbb{R}^d$, it becomes important to determine whether $\rglBound$-good arms exist near the direction of $\pi_S(\truethetam)$ rather than $\truethetam$. The following definition clarifies this concept.

\begin{definition} [$\rglBound$-goodness] \label{def:goodness_projection}
    For fixed $\rglBound \in (0,1]$, we say that the set of feature vectors $\{x_1, \ldots, x_K\}$ satisfies $\rglBound$\textit{-goodness} condition when there exists $\rglBall >0 $ that satisfies  
    \begin{equation}
        \text{for all } \beta \in \mathbb{B}_{\rglBall}(\pi_S(\theta_1^*)) \cup \ldots \cup \mathbb{B}_{\rglBall}(\pi_S(\theta_M^*)),~
        \text{there exists }i \in [K],~~~ x_i^\top{\beta \over \|\beta\|_2} \ge \rglBound.
    \end{equation} 
\end{definition}
\begin{assumption} [$\rglBound$-goodness]\label{assump:Good_projection}
    We assume $\{x_1, \ldots, x_K\}$ satisfies $\rglBound$-regular with $\rglBound > 1-{\lambda_1^2 \over 18}$.
\end{assumption}
Once again, in the following analysis, $\rglBall$ denote the value that satisfies the goodness condition defined in Definition~\ref{def:goodness_projection}, in conjunction with 
$\rglBound$ as specified in Assumption~\ref{assump:Good_projection}. 

The only question is how to construct an $l_2$ bound on $\pi_S\big(\hatthetams\big)-\pi_S\big(\truethetam\big)$ without utilizing the minimum eigenvalue of the Gram matrix, which is zero when $S_x \subsetneq \RR^d$. The key idea is that we can use $\min_{\|\beta\|=1,~ \beta \in S_x}\left( { \sum_{s=1}^{t-1} \left\langle \beta,~ x(s) \right\rangle^2 }\right)$ to fulfill the role previously played by the minimum eigenvalue. We present 2 Lemmas, Lemma~\ref{lem:mineigen_oneround_projection} and Lemma~\ref{lem:thetahat_bound_projection}, to explain the idea. First, The following demonstrates the linear growth of $\min_{\|\beta\|=1,~ \beta \in S_x}\left( { \sum_{s=1}^{t-1} \left\langle \beta,~ x(s) \right\rangle^2 }\right)$.
\begin{lemma}[Increment of the alternative value of minimum eigenvalue]\label{lem:mineigen_oneround_projection}
    Suppose Assumptions ~\ref{assump:Bdd_projection}, \ref{assump:OD_projection}, and ~\ref{assump:Good_projection} hold. Assume a least square solution $\hatthetams$ satisfies $\|\pi_S\big(\hatthetams\big)-\pi_S\big(\truethetam\big)\| \le {\rglBall \over 2}$, for all $m \in [M]$ and $s \ge T_0+1$, then the arm selected by $\mogrorw$ satisfies
    \begin{equation*}
       \min_{\|\beta\|=1,~ \beta \in S_x}\mathbb{E}\left[ { \sum_{m \in [M]} \left\langle \beta,~ x(s) \right\rangle^2 }| \mathcal{H}_{s-1}\right] \ge \left(\lambda_1- 2 \sqrt{2 +2\rglBall\sqrt{1-\rglBound^2}-2\rglBound\sqrt{1-\rglBall^2}}\right)\phi_{\rglBall/2, \Wcal} M.
    \end{equation*}
\end{lemma}
\begin{proof}
For $s \ge T_0 +1$ and $m \in [M]$, let $E_{\bar{m}}(s)$ be the event that the weighted objective $\sum_{m \in [M]}w_m\theta_m^*$ in round $s$ satisfies $\left\|\sum_{m\in [M]}\pi_S(w_m\theta_m^*)-\pi_S(\theta_{\bar{m}}^*)\right\|_2 < \rglBall/2$. From Equation~\ref{eq:mineigen_oneround}, for any unit vector $\beta \in S_x$, we have
\begin{equation*}\label{eq:mineigen_oneround_projection}
    \mathbb{E}\left[ {\left\langle \beta,~ x(s) \right\rangle^2 }| \mathcal{H}_{s-1}\right] \ge\phi_{\rglBall/2, \Wcal} \sum_{m=1}^M \beta^\top \left(\mathbb{E}[x(s)| E_m(s), \mathcal{H}_{s-1}]\mathbb{E}[x(s)| E_m(s), \mathcal{H}_{s-1}]^\top\right)\beta.   
\end{equation*}

Let $x_{r(m)}:=\mathbb{E}[x(s)| E_m(s), \mathcal{H}_{s-1}]$. Then, $x_{r(m)}$ is $\rglBound$-good for $\pi_S(\truethetam)$ for all $m\in[M]$ (see the proof of Lemma~\ref{lem:mineigen_oneround}).

Since the greedy selection of $\hatthetams$ is equal to that of $\pi_S\big(\hatthetams \big)$, for the same reason as Lemma~\ref{lem:dist_good_arm}, we can get $\|{x_{r(m)}} - \pi_S\big(\theta_{m}^*\big)\|_2 \le \sqrt{2 \left(1+ \rglBall \sqrt{1-\rglBound^2}-\rglBound\sqrt{1-\rglBall^2}\right)}$.

Then, we have
\begin{align*}
    &\beta^\top\left(\sum_{m=1}^{M} x_{r(m)} x_{r(m)}^\top\right)\beta\\
    &=\sum_{m=1}^{M} \{\left\langle \beta,~ \pi_S\big(\theta_{m(s)}^*\big)\right\rangle ^2+ \left\langle \beta,~ x_{r(m)}-\pi_S\big(\theta_{m(s)}^*\big)\right\rangle ^2+2 \left\langle \beta,~ \pi_S\big(\theta_{m(s)}^*\big)\right\rangle \left\langle \beta,~x_{r(m)}-\pi_S\big(\theta_{m(s)}^*\big)\right\rangle\}\\
    &\ge M\lambda_1- 2\sqrt{2 \left(1+ \rglBall \sqrt{1-\rglBound^2}-\rglBound\sqrt{1-\rglBall^2}\right)}M.
\end{align*}

Plugging this into Inequality~\ref{eq:mineigen_oneround_projection}, for any unit vector $\beta \in S_x$, we have
\begin{align*}\label{eq:mineigen_oneround_projection}
    \mathbb{E}\left[ \left\langle \beta,~ x(s) \right\rangle^2 | \mathcal{H}_{s-1}\right] &\ge\phi_{\rglBall/2, \Wcal} \beta^\top \left(\sum_{m=1}^M x_{r(m)}x_{r(m)}^\top\right)\beta\\
    &\ge \left(\lambda_1- 2 \sqrt{2 +2\rglBall\sqrt{1-\rglBound^2}-2\rglBound\sqrt{1-\rglBall^2}}\right)\phi_{\rglBall/2, \Wcal} M.
\end{align*}

\end{proof}

The next lemma shows how to derive $l_2$ bound on $\pi_S\big(\hatthetams\big)-\pi_S\big(\truethetam\big)$ with $\min_{\|\beta\|=1,~ \beta \in S_x}\left( { \sum_{s=1}^{t-1} \left\langle \beta,~ x(s) \right\rangle^2 }\right)$. 
\begin{lemma}
\label{lem:thetahat_bound_projection}
    For all $m \in [M]$ and $s \in [T]$, any least square solution $\hatthetamt$ of $\big(\sum_{s=1}^{t-1}x(s)x(s)^\top \big)~\theta = \sum_{s=1}^{t-1}x(s)y_{a(s),m}(s)$ satisfies
    \begin{equation*}
         \left\|\pi_S\big(\hatthetams\big)-\pi_S\big(\truethetam\big)\right\|_2 \le {\|\sum_{s=1}^{t-1}x(s)\eta_{a(s),m}(s)\|_2 \over \min_{\|\beta\|=1,~ \beta \in S_x}\left( { \sum_{s=1}^{t-1} \left\langle \beta,~ x(s) \right\rangle^2 }\right)}. 
    \end{equation*}
\end{lemma}
\begin{proof}
From the definition of $\hatthetamt$, we have  
\begin{equation*}
    \big(\sum_{s=1}^{t-1}x(s)x(s)^\top \big)~(\hatthetamt-\truethetam) = \sum_{s=1}^{t-1}x(s)\eta_{a(s),m}(s). 
\end{equation*}

Since the row space of $\big(\sum_{s=1}^{t-1}x(s)x(s)^\top \big)$ is in $S$, 
\begin{align*}
     \left\|\sum_{s=1}^{t-1}x(s)\eta_{a(s),m}(s)\right\|_2 &= {\left\|\big(\sum_{s=1}^{t-1}x(s)x(s)^\top \big)~\left(\pi_S\big(\hatthetams\big)-\pi_S\big(\truethetam\big)\right)\right\|_2 } \\
     &\ge \min_{\|\beta\|=1,~ \beta \in S_x} \left(\beta^\top \big(\sum_{s=1}^{t-1}x(s)x(s)^\top \big) \beta\right) \left\| \pi_S\big(\hatthetams\big)-\pi_S\big(\truethetam\big)\right\|_2.
\end{align*}
The last inequality holds by Lemma~\ref{lem:simple_linalg}. 
\end{proof}
With above two lemmas, we can obtain the same effective Pareto regret bound and effective Pareto fairness as in Theorem ~\ref{thm:EPR} and ~\ref{thm:EPF}.


\section{Lower bound}\label{sec:lower_bound}

\begin{theorem}
\label{thm:LB}
Suppose Assumptions~\ref{assump:Bdd}, ~\ref{assump:OD}, and ~\ref{assump:Good} hold, and $d\ge2$ and $K \ge d^2$. For any algorithm choosing action $a(t)$ at round $t$, there exists a worst-case problem instance such that the Pareto regret of the algorithm  is lower bounded as
\begin{equation*}
\sup_{(\theta_1^*, \ldots, \theta_M^*)} \regret(T) =\Omega(\sqrt{dT}).
\end{equation*}
\end{theorem}

\paragraph{Discussion of Theorem~\ref{thm:LB}.}The above theorem shows that the regret bound for our algorithm in Theorem~\ref{thm:EPR} is optimal in terms of $d$ and $T$. This bound matches the lower bound of \citet{Chu11a} in the single-objective setting. However, in their work, the $d$ term in the lower bound is obtained by partitioning the time horizon and carefully designing the features within each partition. As a result, their analysis requires the condition $T \ge d^2$, and their approach is not applicable in our fixed feature setting. Instead, we obtain the $d$ term by partitioning the set of $K$ arms, which leads to the requirement that $K \ge d^2$ in our analysis.

 For our convenience, we define the following augmented parameter set, which will be used throughout the remainder of this section.
\begin{definition} 
 The augmented parameter set $\Theta$ is a set of combinations of objective parameters, i.e.,  $\Theta :=\big\{\big(\theta_1^{(1)}, \theta_2^{(1)}, \ldots, \theta_M^{(1)}\big), \big(\theta_1^{(2)}, \theta_2^{(2)}, \ldots, \theta_M^{(2)}\big), \ldots, \big(\theta_1^{(d)}, \theta_2^{(d)}, \ldots, \theta_M^{(d)}\big) \big\}$ , where each $\theta_m^{(j)} \in \RR^d$ for all $j \in [d]$ and for all $m \in [M]$.  The set of the first objective parameters in each of instances in~$\Theta$ is defined as $\Theta_1 =\big\{\theta_1^{(1)}, \theta_1 ^{(2)}, \ldots, \theta_1^{(d)} \big\}$.  
\end{definition}
Note that $\Theta$ represents $d$ separate multi-objective problem instances, while $\Theta_1$ represents $d$ separate single-objective (objective $1$) problem instances.

\textit{Proof sketch}.
We prove the theorem by constructing the augmented parameter set $\Theta$ and defining the feature set so that the feature vectors are aligned with the direction of objective parameters and has maximum length, thereby ensuring that Assumption~\ref{assump:Good} is satisfied. Then, we bound the Bayes Pareto regret  $\mathbb{E}_{\theta^*\sim \mathrm{UNIFORM}(\Theta)}[\regret(T)]$ for any action sequence $a(t), t\in [T]$. For this, we convert our problem to single objective problem, and then use Lemma~\ref{lem:Auer2002}. The proof of Theorem~\ref{thm:LB} is presented in 
Section~\ref{sec:lower_bound_proof}, and its supporting lemmas are presented in Section~\ref{sec:lower_bound_lemma}

\subsection{Technical lemmas for Theorem~\ref{thm:LB}}\label{sec:lower_bound_lemma}
We first construct the set of problem instances that can be converted to a single objective problem.
\begin{lemma}
\label{lem:LB_SO}
Suppose Assumptions~\ref{assump:Bdd}, ~\ref{assump:OD}, and ~\ref{assump:Good} hold. For all $d\ge2$ and $K\ge d^2$, there exist an augmented parameter set $\Theta$ and a set of feature vectors such that for any action sequence $a(t)\in [K]$ for $t\in [T]$, there exists another action sequence $a'(t)\in [d]$ for $t\in [T]$ that satisfies 
\begin{equation}
\label{eq:LB_SO}
 \mathbb{E}_{\theta^* \sim \mathrm{UNIFORM}(\Theta)}\left[\sum_{t=1}^T \Delta_{a(t)}\right] \ge \mathbb{E}_{\theta_1^*\sim \mathrm{UNIFORM}(\Theta_1) }\left[\sum_{t=1}^T (\max_{i \in [d]} x_i^\top \theta_1^* -x_{a'(t)}^\top \theta_1^*)\right].
\end{equation}
 
\end{lemma}
\begin{proof}
It is enough to show when $M=d$ and $K=d^2$, because if $M>d$ (Assumption~\ref{assump:OD} guarantees $M \ge d$) or $K>d^2$, we can make the same argument by setting $\theta_m^{(j)}=\theta_d^{(j)}$ for all $d\le m\le M$ and $j \in [d]$ or $x_i=x_d$ for $ d \le i \le K$. We first divide the cases by when $d=2$ and $d\ge3$.\\

\textbf{Case 1}. $d=M=2$ and $K=4$

\begin{figure}[h]
\begin{center}
    \includegraphics[width=0.4\textwidth]{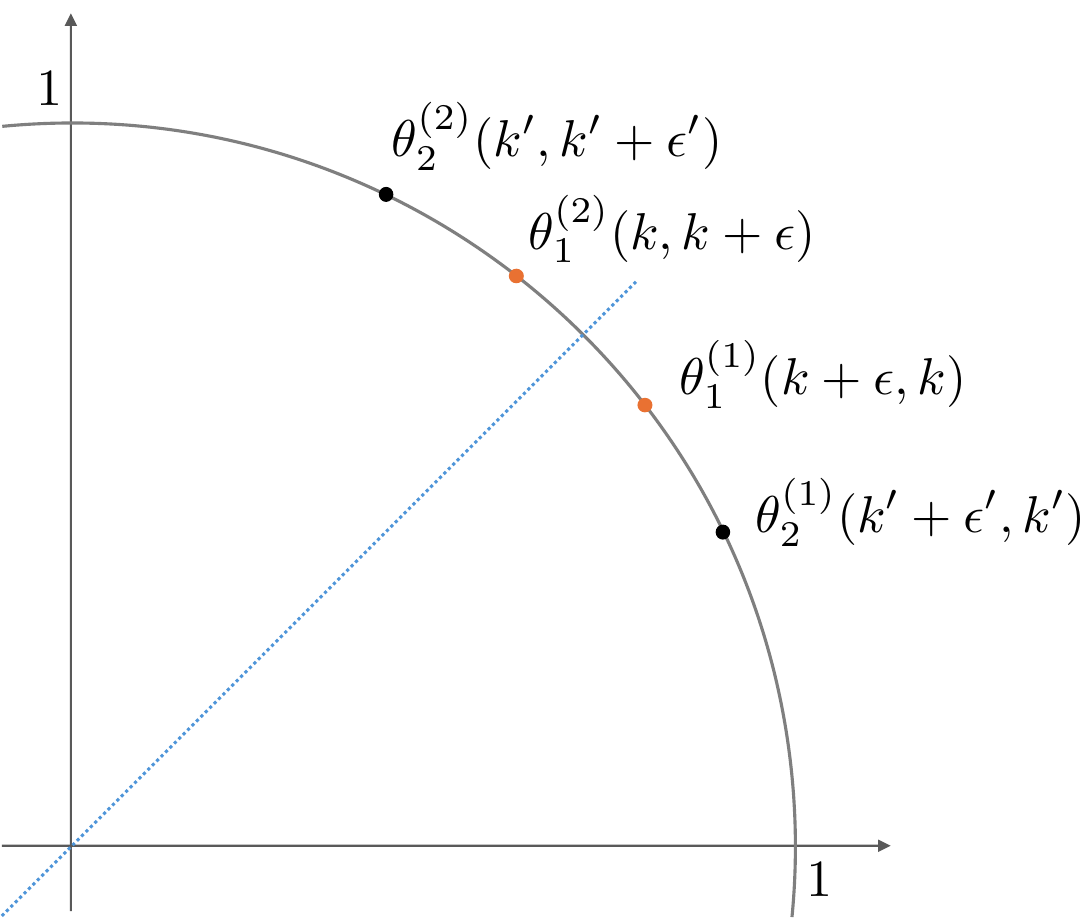} 
    \caption{Problem space $\Theta$ construction when $d=2$.}
    \label{fig:LB_2d}
\end{center}
\end{figure}

Fix $0<\epsilon < {1}$ and let $\theta_1^{(1)}=(k+\epsilon, k), ~\theta_1^{(2)}=(k, k+\epsilon), ~\theta_2^{(1)}=(k'+\epsilon', k'), ~\theta_2^{(2)}=(k', k'+\epsilon')$, where $\epsilon<\epsilon'<1$ and $2k^2 + 2k\epsilon + \epsilon^2 = 2(k')^2 + 2(k')\epsilon + \epsilon^2=1$.
Define the feature vectors $x_1 = \theta_1^{(1)},~ x_2 = \theta_1^{(2)}, ~x_3=\theta_2^{(1)}$, and $x_4=\theta_2^{(2)}$. Then all feature vectors and objective parameters have $l_2$ norm $1$, satisfying Assumption~\ref{assump:Bdd}. Also, each $\big(\theta_1^{(1)}, \theta_2 ^{(1)}\big)$ and $\big(\theta_1^{(2)}, \theta_2 ^{(2)}\big)$ satisfies Assumption~\ref{assump:OD}, and since $x_1, x_2, x_3, x_4$ are $\gamma$-good arms for $\theta_1^{(1)}, \theta_1^{(2)}, \theta_2^{(1)}, \theta_2^{(2)}$ for any $\gamma\le1$, the feature set satisfies Assumption~\ref{assump:Good}. Now, we show that the good arms for objective $1$ ($x_1$ and $x_2$) are always the better choice than other arms ($x_3$ and $x_4$) in both problem instances from the perspective of Pareto optimality. 

If $(\theta_1^*, \theta_2^*)=\big(\theta_1^{(1)}, 
\theta_2 ^{(1)}\big)$, 
\begin{align*}
&\mathbb{E}[ \Delta_1]= 1-\langle\theta_1^{(1)}, \theta_1 ^{(1)}\rangle 
~ (=0)\\
&\mathbb{E}[ \Delta_2] = 1-\langle\theta_1^{(2)}, \theta_1 ^{(1)}\rangle ~ (= \epsilon^2) \\
&\mathbb{E}[ \Delta_3]= 0=1-\langle\theta_1^{(1)}, \theta_1 ^{(1)}\rangle \\
&\mathbb{E}[ \Delta_4] = 1-\langle\theta_2^{(2)}, \theta_1 ^{(1)}\rangle 
> \epsilon^2 = 1-\langle\theta_1^{(2)}, \theta_1 ^{(1)}\rangle. 
\end{align*}
Otherwise, if $(\theta_1^*, \theta_2^*)=\big(\theta_1^{(2)}, 
\theta_2 ^{(2)}\big)$, 
\begin{align*}
&\mathbb{E}[ \Delta_1]= 1-\langle\theta_1^{(1)}, \theta_1 ^{(2)}\rangle 
~ (=\epsilon^2)\\
&\mathbb{E}[ \Delta_2] = 1-\langle\theta_1^{(2)}, \theta_1 ^{(2)}\rangle ~ (= 0) \\
&\mathbb{E}[ \Delta_3]=  1-\langle\theta_2^{(1)}, \theta_1 ^{(2)}\rangle 
> \epsilon^2 = 1-\langle\theta_1^{(1)}, \theta_1 ^{(2)}\rangle \\
&\mathbb{E}[ \Delta_4] =0 = 1-\langle\theta_1^{(2)}, \theta_1 ^{(2)}\rangle. 
\end{align*}
Therefore, if we define $a'(t)= 
\begin{cases}
1 & \text{if } a(t) = 1  \text{ or }3 \\
2 & \text{if } a(t) = 2  \text{ or }4
\end{cases}$, the statement of lemma is satisfied. \\

\clearpage
\textbf{Case 2.} $d=M\ge 3$ and $K=d^2$

Fix $0<\epsilon < 1$ and define $\Theta$ and feature vectors $x_i$, $i\in[d^2]$ as
\begin{align*}
    &x_1=\theta_1^{(1)}= (k+\epsilon , k, \ldots, k), \\ 
    &x_2=\theta_1^{(2)}= (k , k+\epsilon, k,  \ldots, k), \\
    &\ldots\\
    &x_d=\theta_1^{(d)}= (k ,   \ldots, k, k+\epsilon), \\
    &x_{d+1}=\theta_2^{(1)}= (k'+2\epsilon , k'-\epsilon, k' \ldots, k'), \\ 
    &x_{d+2}=\theta_2^{(2)}= (k' , k' +2\epsilon , k'-\epsilon, k' \ldots, k'), \\
    &\ldots\\
    &x_{2d}=\theta_2^{(d)}= (k'-\epsilon, k',   \ldots, k', k'+2\epsilon) \\
    &x_{2d+1}=\theta_3^{(1)}= (k'+2\epsilon , k', k'-\epsilon, k' \ldots, k'), \\
    &x_{2d+2}=\theta_3^{(2)}= (k', k'+2\epsilon , k', k'-\epsilon \ldots, k'), \\
    &\ldots\\
    &x_{d^2}=\theta_d^{(d)}= (k', k',   \ldots, k'-\epsilon, k'+2\epsilon),\\ 
    &\text{where }  dk^2 + 2k\epsilon +\epsilon^2 = d(k')^2 +2(k') \epsilon + 5\epsilon^2 = 1.    
\end{align*}

\begin{figure}[t]
\begin{center}
    \includegraphics[width=0.6\textwidth]{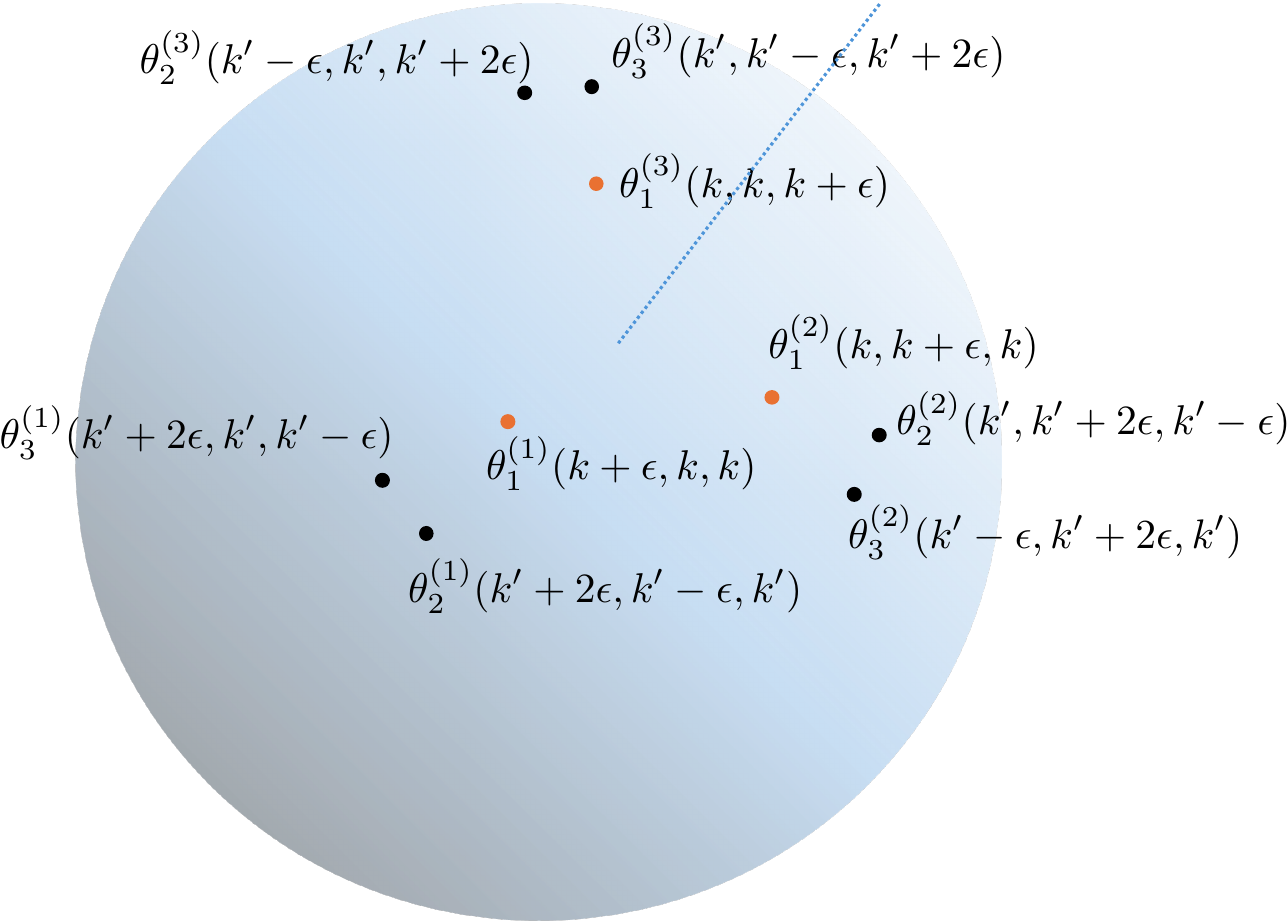} 
    \caption{Problem space $\Theta$ construction when $d=3$. The blue line represents the direction of $(1,1,1)$, and the sphere is a unit sphere.} 
    \label{fig:LB_3d}
\end{center}
\end{figure}
It is obvious that $k' < k $ and $\|x_i \| =1$ for all $i \in [d^2]$. Similar to the simple case when $d=2$, for $j \in [d]$, each $\big(\theta_1^{(j)}, \ldots ,  \theta_d ^{(j)}\big)$ satisfies Assumption~\ref{assump:OD}, and the feature set satisfies Assumption~\ref{assump:Good}. 

To prove the lemma, similar to the simple $d=2$ case, we will show that $\theta_1^{(j)}$ is always better than $\theta_m^{(j)}~ (m\neq1)$ for all $j \in [d]$. For any feature vector $x_i$, we denote by $\Delta(x_i)$ the sub-optimality gap of the feature vector, i.e. $\Delta(x_i):=\Delta_i$. Then, it is enough to show that for any $m, j \in [d]$ and $\theta^* \in \Theta$, $\mathbb{E}_{\theta^*}[\Delta{(x_{(m-1)d+j})}]=\mathbb{E}_{\theta^*}[\Delta{(\theta_m^{(j)})}] \ge \max_{j'\in[d]}\big(\theta_1^{(j')}\big)^\top\theta_1^*-\big(\theta_1^{(j)}\big)^\top\theta_1^*$ holds.\\

Let $\theta^*$ be the objective parameters for $(j_*)$-th instance, i.e.   $\theta^*=\big(\theta_1^{(j_*)}, \theta_2^{(j_*)}, \ldots, \theta_d^{(j_*)}\big) \in \Theta $. If $j_*=j$, then $\mathbb{E}_{\theta^*}[\Delta{(\theta_m^{(j)})}] =0 =  1-\big(\theta_1^{(j)}\big)^\top{\theta_1^{(j)}}$.

Suppose $j_* \neq j $. For all $m'\in [d]$, since 
\begin{align*}
    &\langle \theta_1^{(j)}, \theta_1^{(j_*)}\rangle=dk^2 + 2k \epsilon = 1-\epsilon^2 ,\\
    &\langle \theta_1^{(j)}, \theta_{m'}^{(j_*)}\rangle \le (k+\epsilon)(k') + k (k'+2\epsilon) + k(k'-\epsilon) + (d-3)kk'\\
    &~~~~~~~~~~~~~~~~~ = dkk'+k\epsilon+k'\epsilon \\
    &~~~~~~~~~~~~~~~~~ < {d\over2} {k^2 } + {d\over2}  {(k')^2} + k \epsilon + k'\epsilon\\  
     &~~~~~~~~~~~~~~~~~ = \left({1 \over 2}-{\epsilon^2 \over2}\right) + \left({1 \over 2}-{5\epsilon^2 \over2}\right)  \\
     &~~~~~~~~~~~~~~~~~ = 1-3\epsilon^2, 
\end{align*}
it holds that $\mathbb{E}_{\theta^*}[\Delta({\theta_1^{(j)}})] =\epsilon^2 =  1-\big(\theta_1^{(j)}\big)^\top{\theta_1^*}$. 

For $m\neq 1$ and $m' \neq1$, we have 
\begin{align*}
    &\langle \theta_m^{(j)}, \theta_{1}^{(j_*)}\rangle \le \max_{m'\neq1} \langle \theta_1^{(j)}, \theta_{m'}^{(j_*)}\rangle <  1-3\epsilon^2 ,\\
    &\langle \theta_m^{(j)}, \theta_{m'}^{(j_*)}\rangle \le 2(k') (k'+2\epsilon) + (k'-\epsilon)^2 + (d-3)(k')^2\\
    &~~~~~~~~~~~~~~~~~ = d(k')^2+2k'\epsilon+\epsilon^2 \\
    &~~~~~~~~~~~~~~~~~ = 1-4\epsilon^2, 
\end{align*}
and hence $\mathbb{E}_{\theta^*}[\Delta{(x_{(m-1)d+j})}]=\mathbb{E}_{\theta^*}[\Delta({\theta_m^{(j)}})]>3\epsilon^2>\epsilon^2 =  1-\big(\theta_1^{(j)}\big)^\top{\theta_1^*}$. 
Therefore, if we define $a'(t)= a(t)  ~\text{mod} ~m$ (if $a(t)/m \in \mathbb{N}$, then $a'(t)=m$), then the lemma holds.  
\end{proof}

\begin{lemma}
\label{lem:LB}
Suppose Assumptions~\ref{assump:Bdd}, ~\ref{assump:OD}, and ~\ref{assump:Good} hold. For all $0<\epsilon<1$, $d\ge2$, $K\ge d^2$, and any action sequence $a(t)$ for $t \in [T]$, there exists a augmented parameter set $\Theta$ and a set of features satisfying Equation (\ref{eq:LB_SO}), where for all $j\in [d]$, the expected reward of arm $j$ for objective $1$ is equal to $1$ in $j$-th problem instance and is $1-\epsilon^2$ in other instances $j' \in [d]-\{j\}$.  
\end{lemma}
\begin{proof}
The parameter set $\Theta$ and the feature set $\{x_1=\theta_1^{(1)}, x_2 = \theta_1^{(2)}, \ldots, x_{d^2}=\theta_d^{(d)}(=x_{d^2+1}=\ldots=x_{K})\}$ constructed in the proof of Lemma~\ref{lem:LB_SO} satisfy the properties required in the latter part of this lemma. For each $j \in [d]$, the feature vector of arm $j$ is given by $\theta_1^{(j)}$, so in the $j$-th instance, the expected reward for objective 1 is $\langle \theta_1^{(j)}, \theta_1^{(j)} \rangle = 1$. For any other instance $j' \in [d] \setminus {j}$, we have $\langle \theta_1^{(j)}, \theta_1^{(j')} \rangle = 1 - \epsilon^2$.
\end{proof}

The above lemma reduces the problem of bounding multi-objective regret to that of deriving a lower bound for the single-objective case. In particular, for each of $d$ instances, one arm among the $d$ arms has a single-objective expected reward larger than the others by $\epsilon^2$, which makes it possible to apply Lemma~\ref{lem:Auer2002}.

\subsection{Proof of Theorem~\ref{thm:LB}}\label{sec:lower_bound_proof}
\begin{proof}
By Lemma~\ref{lem:LB}, it is enough to bound the single objective regret 
$$\mathbb{E}_{\theta_1^*\sim \mathrm{UNIFORM}(\Theta_1) }[\sum_{t=1}^T (\max_{i \in [d]} x_i^\top \theta_1^* -x_{a'(t)}^\top \theta_1^*)],$$
where for each $\theta_1^{(j)} \in \Theta_1$, the expected reward of the arm $j$ is equal to $1$ , while the other arms $j' \in [d]-\{j\}$ have the expected reward $1-\epsilon^2$. If we set $\epsilon = \sqrt{1-{1 \over {1 + {1 \over2 }\sqrt{d \over T}}}}$, then the expected reward of arm $j$ is $1 \over {1 + {1 \over2 }\sqrt{d \over T}}$ for $j'(\neq j)$-th instances. Scaling by ${1 \over2}+{1 \over4}{\sqrt{d \over T}} > {1 \over2}$, we have that the expected reward of arm $j \in [d]$ is ${1 \over2}+{1 \over4}{\sqrt{d \over T}}$ for $j$-th instance, while it is $1 \over 2$ for other instances. 
Applying Lemma~\ref{lem:Auer2002}, we have 
\begin{equation*}
 \mathbb{E}_{\theta_1^*\sim \mathrm{UNIFORM}(\Theta_1) }[\sum_{t=1}^T (\max_{i \in [d]} x_i^\top \theta_1^* -x_{a'(t)}^\top \theta_1^*)] =\Omega( \sqrt{dT}).
\end{equation*}
Therefore, by Lemma~\ref{lem:LB},
\begin{align*}
    \sup_{(\theta_1^*, \ldots, \theta_M^*)} \left[\sum_{t=1}^T \mathbb{E}[\Delta_{a(t)}]\right] & \ge \mathbb{E}_{\theta^* \sim \mathrm{UNIFORM}(\Theta)}\left[\sum_{t=1}^T \Delta_{a(t)}\right] \\
    &\ge \mathbb{E}_{\theta_1^*\sim \mathrm{UNIFORM}(\Theta_1) }\left[\sum_{t=1}^T (\max_{i \in [d]} x_i^\top \theta_1^* -x_{a'(t)}^\top \theta_1^*)\right]\\
    &=\Omega( \sqrt{dT}). 
\end{align*}
\end{proof}


\section{Experiment}
\label{ap_sec:exp}

\subsection{Experimental Settings}\label{ap_subsec:exp_setting}

We validate the empirical performance of $\mogrorw$ and $\mogrorr$ in a linear bandit setting, comparing them with other multi-objective algorithms. Specifically, we experiment with a linear bandit where $y_m(t) = \mathcal{N}(x_i^T\truethetam, 0.1^2)$ for all $i \in [K]$ and $m \in [M]$. For each problem instance, $M$ objective parameters are sampled uniformly at random from the positive part of a unit sphere. Then, $K$ feature vectors ($K > 2M$) are generated by drawing samples from $\mathbb{B}^d$. In the fixed arms setting, the first $M$ feature vectors are sampled from a multivariate normal distribution with the true objective parameter as the mean and a covariance matrix of $0.1I_d$. These vectors are then scaled to ensure their magnitudes lie within the range $(3/4, 1)$. The remaining $K-M$ feature vectors are sampled uniformly at random from $\mathbb{B}^d$, with $M$ of these scaled to have magnitudes greater than $3/4$ and the rest scaled to have magnitudes less than $3/4$. Limiting the magnitudes of the feature vectors ensures that excessively large Pareto fronts, which could lead to meaningless results, are avoided. For the varying arms setting, contexts are drawn uniformly from $\mathbb{B}^d$. The results are averaged over $10$ independent problem instances for each $(d, K, M)$ combination, with each problem instance being repeated $20$ times to compute the final statistics. The experiments are run on Xeon(R) Gold 6226R CPU @ 2.90GHz (16 cores). 

We conduct experiments on our proposed near-greedy algorithms and the three baselines, : $\epsilon$-Greedy algorithm \molbepsgreedy with ($\epsilon=0.1$), the Upper Confidence Bound algorithm $\molbucb$ (linear version of $\texttt{MOGLM-UCB}$ of \citet{Lu2019}), and the Thompson Sampling algorithm $\molbts$~\citep{park2025thompson}. (The algorithms proposed by \citet{Cheng2024} are excluded from the experiments as they are specifically designed for problems with hierarchical objective structures.) For the $\mogrorw$ algorithm, we use $\text{Dirichlet}(\one_M)$ (uniform distribution on $\mathbb{S}^{M-1}$) for generating the weight vectors. 

\subsection{Performance Evaluation}
In this section, we compare effective Pareto regret and original Pareto regret performance with other multi-objective algorithms, in both fixed feature setting and stochastic context setting. 

\begin{figure*}[t!]
    \centering
    \begin{subfigure}{\textwidth}
        \centering
        \includegraphics[width=\linewidth, trim=0cm 0cm 0cm 0cm, clip]{images/PR_fixed_appendix_trial8.pdf} 
        \vspace{-20pt}
        \caption{Cumulative Pareto regret}
        \label{fig:PR_fixed}
    \end{subfigure}
    
    \vspace{+10pt}
    \begin{subfigure}{\textwidth}
        \centering
        \includegraphics[width=\linewidth, trim=0cm 0cm 0cm 2cm, clip]{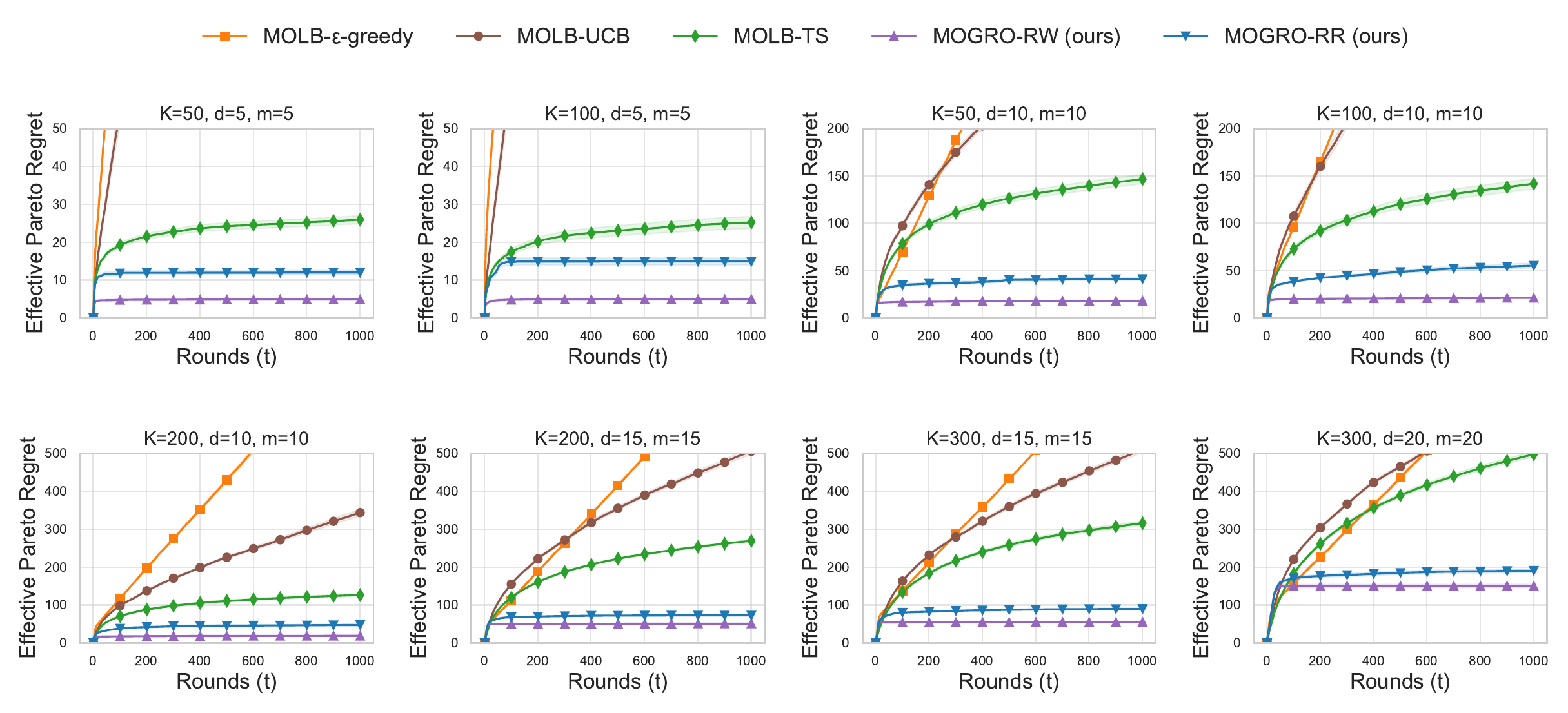} 
        \vspace{-20pt}
        \caption{Cumulative Effective Pareto regret}
        \label{fig:EPR_fixed}
    \end{subfigure}
    \vspace{-10pt}
    \caption{Evaluation of multi-objective bandit algorithms with fixed arms across various $(d, K, M)$ combinations.}
    \label{fig:regret_fixed}
\end{figure*}

\subsubsection{Fixed feature setting}
The empirical results are illustrated in Figure~\ref{fig:regret_fixed}, which reports the cumulative Pareto regret (Figure~\ref{fig:PR_fixed}) and cumulative effective Pareto regret (Figure~\ref{fig:EPR_fixed}) of the algorithms in the fixed-context setting. As shown in these figures, our simple algorithms, $\mogrorw$ and $\mogrorr$, exhibit superior empirical performance compared to existing methods. Notably, our proposed algorithms incur almost no effective Pareto regret after the initial exploration phase, whereas other algorithms continue to exhibit sublinear regret due to additional exploration terms even in later rounds.

Among the baselines, $\molbts$ and $\molbucb$ follow next in performance, while $\texttt{MOLB}-\epsilon-\texttt{greedy}$ performs the worst. The Pareto regret displays a trend similar to that of the effective Pareto regret, since Pareto regret is always upper bounded by effective Pareto regret by definition. However, \molbepsgreedy~ and $\molbucb$ are not only conservative but also explicitly pursue Pareto fairness, which leads to relatively weaker empirical performance in terms of effective Pareto regret. Overall, these results support our claim that in multi-objective settings—where multiple objectives give rise to many good arms—a brief initial exploration phase is sufficient, after which exploitation alone can be highly effective.

\subsubsection{Stochastic Context Setting}
\begin{figure*}[t!]
    \centering
    \begin{subfigure}{\textwidth}
        \centering
        \includegraphics[width=\linewidth, trim=0cm 0cm 0cm 0cm, clip]{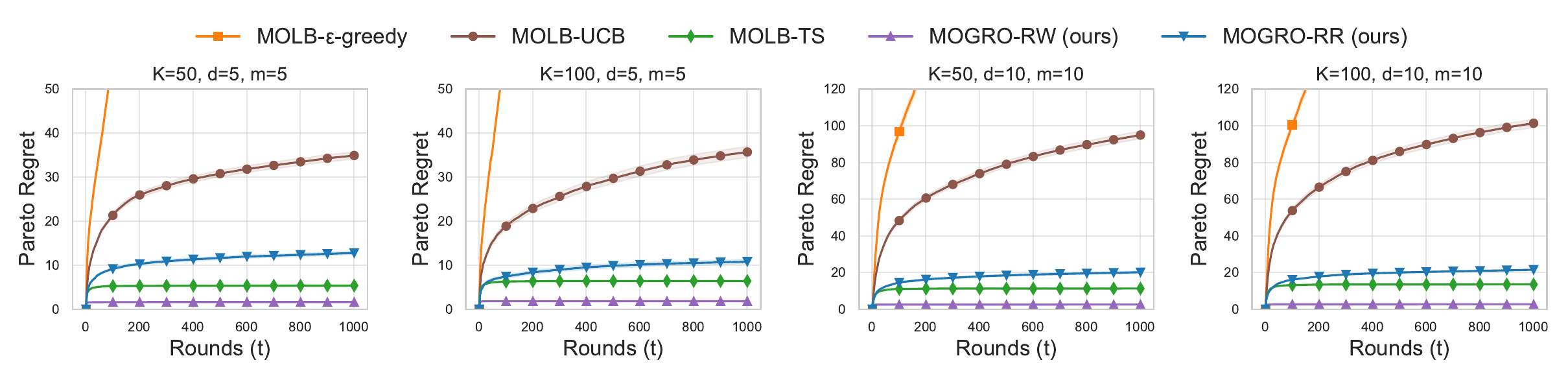} 
        \vspace{-20pt}
        \caption{Cumulative Pareto regret}
        \label{fig:PR_stochastic}
    \end{subfigure}
    
    \vspace{+10pt}
    \begin{subfigure}{\textwidth}
        \centering
        \includegraphics[width=\linewidth, trim=0cm 0cm 0cm 1.5cm, clip]{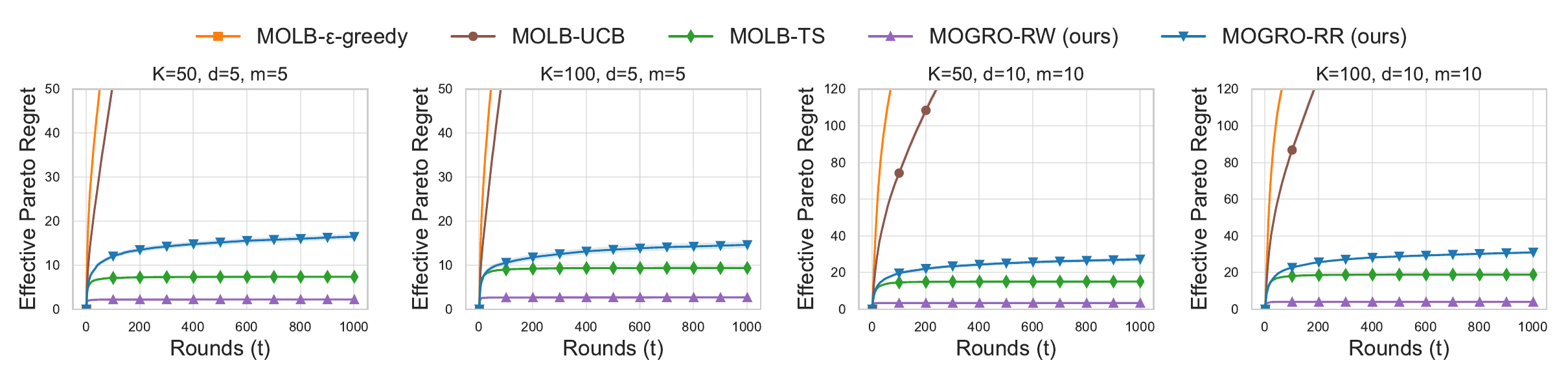} 
        \vspace{-20pt}
        \caption{Cumulative Effective Pareto regret}
        \label{fig:EPR_stochastic}
    \end{subfigure}
    \vspace{-10pt}
    \caption{Evaluation of multi-objective bandit algorithms with varying contexts across various $(d, K, M)$ combinations.}
    \label{fig:regret_stochastic}
\end{figure*}

In the stochastic-context setting, Figure~\ref{fig:regret_stochastic} shows that our proposed algorithm, $\mogrorw$, demonstrates exceptionally strong performance, achieving near-zero regret in most scenarios. The $\molbts$ algorithm also exhibits strong performance, followed by $\mogrorr$ and $\molbucb$, while \molbepsgreedy~ performs the worst. As in the fixed-context case, \molbepsgreedy~ and $\molbucb$ are not only conservative but also explicitly pursue Pareto fairness, which results in relatively weaker empirical performance in terms of effective Pareto regret.

Surprisingly, $\mogrorw$ incurs only a small regret in the initial rounds and essentially zero regret thereafter. This remarkable performance can be explained by two key factors. First, findings from free-exploration studies in single-objective contextual settings can be naturally extended to the multi-objective setting. In particular, the contextual setting itself induces free exploration, enabling greedy algorithms to operate stably without requiring explicit exploration mechanisms. Second, the performance is closely tied to the structure of the (effective) Pareto front and the manner in which its arms are selected. In multi-objective bandits, define the \textit{objective region} as the region formed by the weighted sums of all true objective vectors. Under this definition, any arm that is optimal for some direction within the objective region belongs to the effective Pareto front. Consequently, the probability that the weighted sum of the estimators lies within the objective region is higher than the probability that each individual objective estimator lies within the objective region. This observation provides a key explanation for why $\mogrorw$ (and $\molbts$) performs remarkably well—indeed, better than $\mogrorr$—in the contextual setting.

\subsection{Running Time Comparison}
 We measured the total running time over $1000$ rounds for each multi-objective algorithm. The results in Figure~\ref{fig:time} clearly demonstrate that our algorithms, $\mogrorr$ and $\mogrorr$, are significantly faster than the others. This substantial speedup stems from their simple structure, which avoids computing exploration bonuses or constructing empirical Pareto fronts.

Notably, both algorithms are even faster than the \molbepsgreedy~  algorithm. This is because \molbepsgreedy~ was implemented to approximate the Pareto front at every round, following the conventional design of algorithms that pursue Pareto optimality. As the number of arms increases, computing the empirical Pareto front becomes increasingly time-consuming; consequently, \molbepsgreedy~ and $\molbucb$ slow down markedly as the number of arms grows. In contrast, for a fixed $(d, M)$, the running time of our algorithms remains nearly unchanged even as the number of arms $K$ increases.

The $\molbts$ algorithm exhibits a somewhat different trend: its running time is more sensitive to increases in $d$ and $M$ than to the number of arms. This behavior is related to our implementation choice—we followed the experimental code released by \citet{park2025thompson}, rather than the procedure described in their paper. In the released code, the algorithm optimizes with respect to random weights, like $\mogrorw$, instead of estimating the empirical Pareto front at each round. While our algorithms exhibit a similar qualitative trend, they are incomparably faster than $\molbts$. The longer running time of $\molbts$, despite not estimating the empirical Pareto front, is due to the optimistic sampling procedure incorporated into the algorithm.

The result shows that our algorithms, $\mogrorw$ and $\mogrorw$, are well suited to real-world multi-objective settings with a large number of arms, such as large-scale recommendation systems, as they operate both rapidly and efficiently.

\begin{figure*}[t]
    \centering
    \includegraphics[width=0.8\linewidth, trim=0cm 0cm 0cm 0cm, clip]{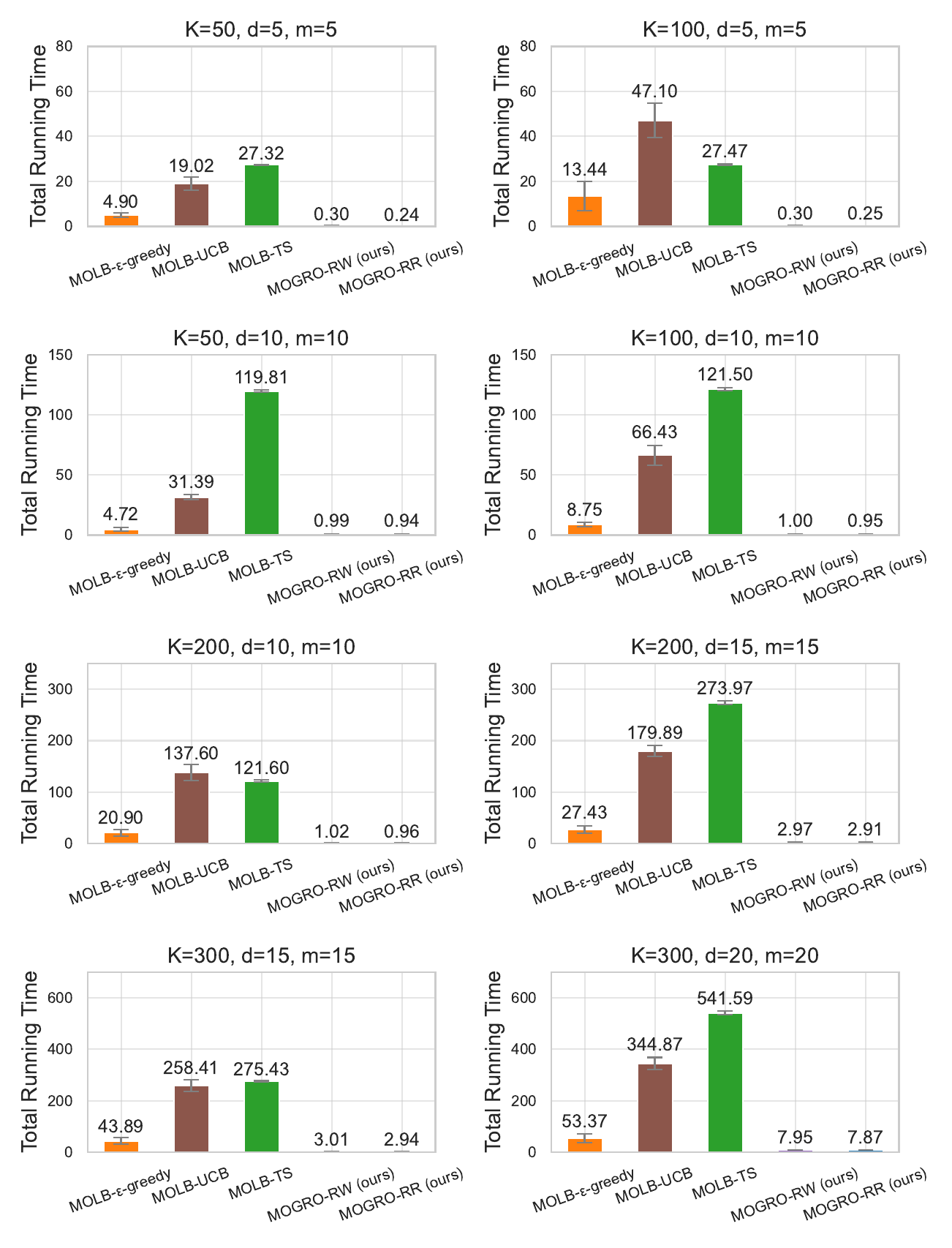} 
    \vspace{-10pt}  
    \caption{Total running time of multi-objective bandit algorithms for $1000$ rounds across various $(d, K, M)$ combinations.}
    \label{fig:time}
\end{figure*}
\clearpage

\subsection{Effective Pareto Fairness}
\begin{figure*}[ht]
    \centering
    \includegraphics[width=\linewidth, trim=0cm 17cm 0cm 0cm, clip]{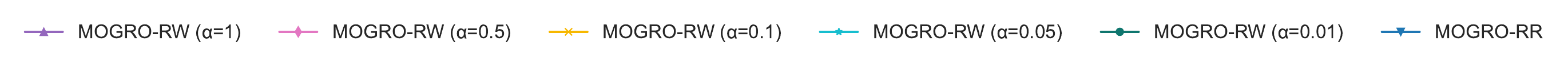} 
    \vspace{-10pt}  
    \begin{subfigure}{0.48\textwidth}
        \centering
        \includegraphics[width=\textwidth, trim=0cm 0cm 0cm 1cm, clip]{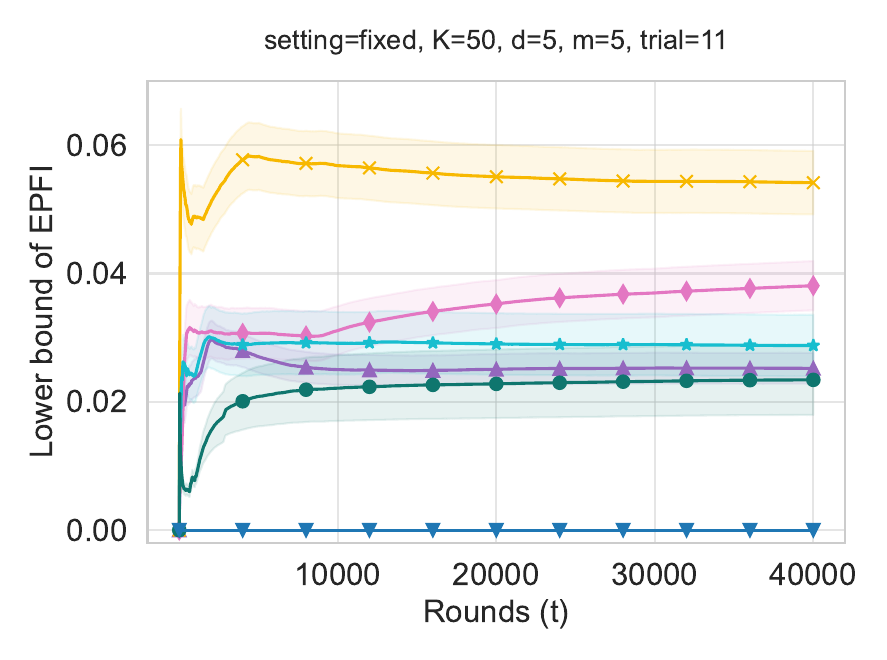}
        \vspace{-20pt}
        \caption{Lower bound of effective Pareto fairness index}
        \label{fig:EPFI_EPFI}
    \end{subfigure}
    \hfill 
    \begin{subfigure}{0.48\textwidth}
        \centering
        \includegraphics[width=\textwidth, trim=0cm 0cm 0cm 1cm, clip]{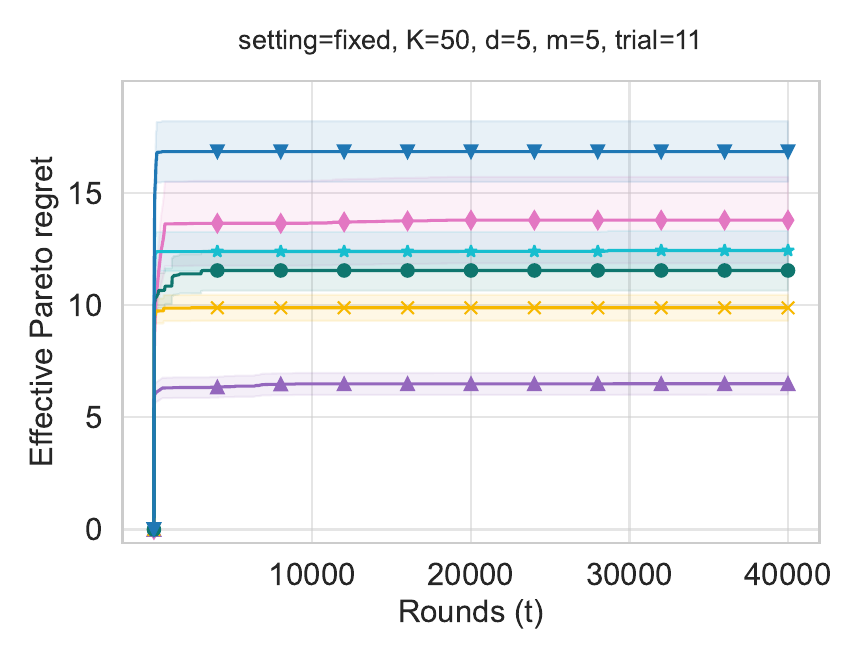}
        \vspace{-20pt}
        \caption{Cumulative effective Pareto regret}
        \label{fig:EPFI_EPR}
    \end{subfigure}
    
    \vspace{10pt}
    \caption{Effective Pareto Fairness of $\mogrorw$ under Different Weight Distributions}
    \label{fig:EPFI}
\end{figure*}

We conduct experiments to verify our theoretical claim that $\mogrorw$ achieves effective Pareto fairness. Specifically, we track the effective Pareto fairness index with $\epsilon = 0.1$ for $\mogrorw$ under five different weight distributions, $\text{Dirichlet}(\alpha \one_M)$ with $\alpha = 1, 0.5, 0.1, 0.05,$ and $0.01$, as well as for $\mogrorr$ (see \cref{def:EPF}), over $40{,}000$ rounds. To simplify the evaluation, we estimate a lower bound of the EPFI by checking whether the selected arm lies within an $\epsilon$-ball of each effective Pareto optimal arm, which is a stricter condition than being $\epsilon$-optimal with respect to a given weight vector. The number of effective Pareto front arms varies between $5$ and $10$ across instances.

Figure~\ref{fig:EPFI_EPFI} demonstrates that the lower bound of the EPFI for $\mogrorw$ converges to a positive value under all considered weight distributions. In particular, $\mogrorw$ with the weight distribution $\text{Dirichlet}(0.1 \one_M)$ tends to select arms from the effective Pareto front in a relatively balanced manner, whereas $\mogrorw$ with $\text{Dirichlet}(0.01 \one_M)$ exhibits a more skewed selection pattern. (Notably, $\mogrorr$ selects only arms that are optimal with respect to individual objective parameters and therefore does not pursue effective Pareto fairness.)

Since the weight distribution can be chosen by the user of the algorithm, it may be selected to reflect the specific goals of the problem at hand. Importantly, variations in the weight distribution do not lead to meaningful differences from the perspective of effective Pareto regret. As shown in Figure~\ref{fig:EPFI_EPR}, for all weight distributions, the regret becomes negligible after the initial exploration phase.

\subsection{Experiment Based on Real-world Data} 
\label{ap_subsec:exp_real}
\subsubsection{Settings}
We conducted experiments using the wine dataset from the UCI Machine Learning Repository to evaluate the performance of our algorithm in a bandit setting. The dataset contains 13 numerical attributes for each wine (Table~\ref{table:wine_data}); among these, we used alcohol, quality, and red as reward objectives, while the remaining 10 attributes were used as features. 
\begin{table*}[t]
     \centering
     \caption{3-objective bandit problem construction using off-line wine dataset.}
     \begin{center}
     \begin{tabular}{lcccccc}
     \toprule
     \textbf{Features} & fixed acidity & volatile acidity & citric acid & residual sugar & chlorides \\
  & free sulfur dioxide & total sulfur dioxide & density & pH & sulphates \\
     \midrule
     \textbf{Reward} & alcohol & quality & red & && \\
     \bottomrule
    
     \end{tabular}
     \end{center}
     \label{table:wine_data}
 \end{table*}
  \begin{figure}[t]
     \centering
     \includegraphics[width=\textwidth, trim=0cm 8cm 0cm 2cm, clip]{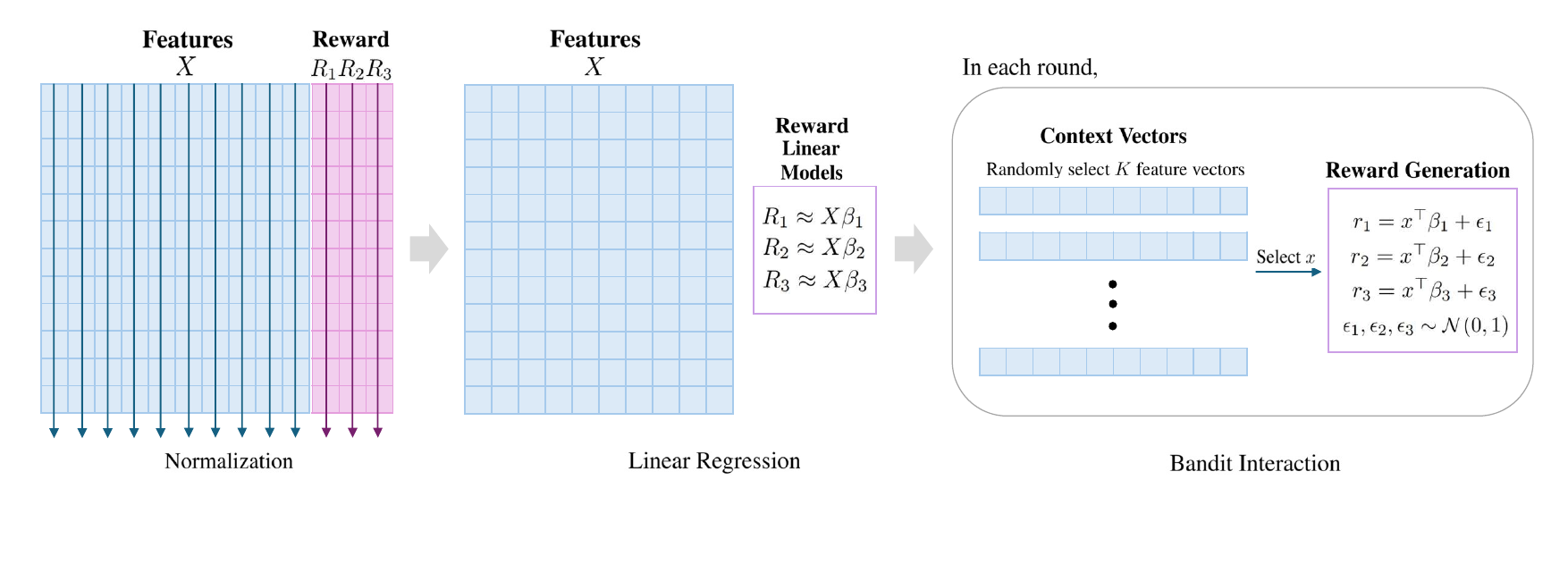}
     \vspace{-30pt}  
     \caption{Description of $3$-objective bandit experiment using real-world off-line data}
     \label{fig:semi_synthetic}
 \end{figure}
 
Figure~\ref{fig:semi_synthetic} illustrates how the offline data was adapted for the bandit experimental setup. For each reward objective, we first performed linear regression on the normalized features. Then, in each round, rewards were generated by adding noise to the predicted value based on the regression model. The noise was sampled from $\mathcal{N}(0,1)$ to mimic the variability observed in the original dataset. Experiments were conducted under two settings, $K = 50$ and $K = 100$, with $20$ episodes of $1000$ rounds each being generated for evaluation.

\subsubsection{Resutls}
Figure~\ref{fig:exp_real} shows plots of cumulative rewards over time for each objective, as well as the final cumulative rewards for two of the objectives achieved by each algorithm. In particular, when the weight distribution was set to $\text{Dirichlet}(2, 2, 2.1)$ (red  point), the $\mogrorw$ algorithm was observed to dominate all existing algorithms, \molbepsgreedy~(orange point), $\molbucb$ (brown point), and $\molbts$ (green point) algorithms across all three objectives. It is noteworthy that $\mogrorw$ under different weight distributions is Pareto optimal in the cumulative sense.

\begin{figure*}[t!]
    \centering
    \begin{subfigure}{\textwidth}
        \centering
        \includegraphics[width=\linewidth, trim=0cm 0cm 0cm 0cm, clip]{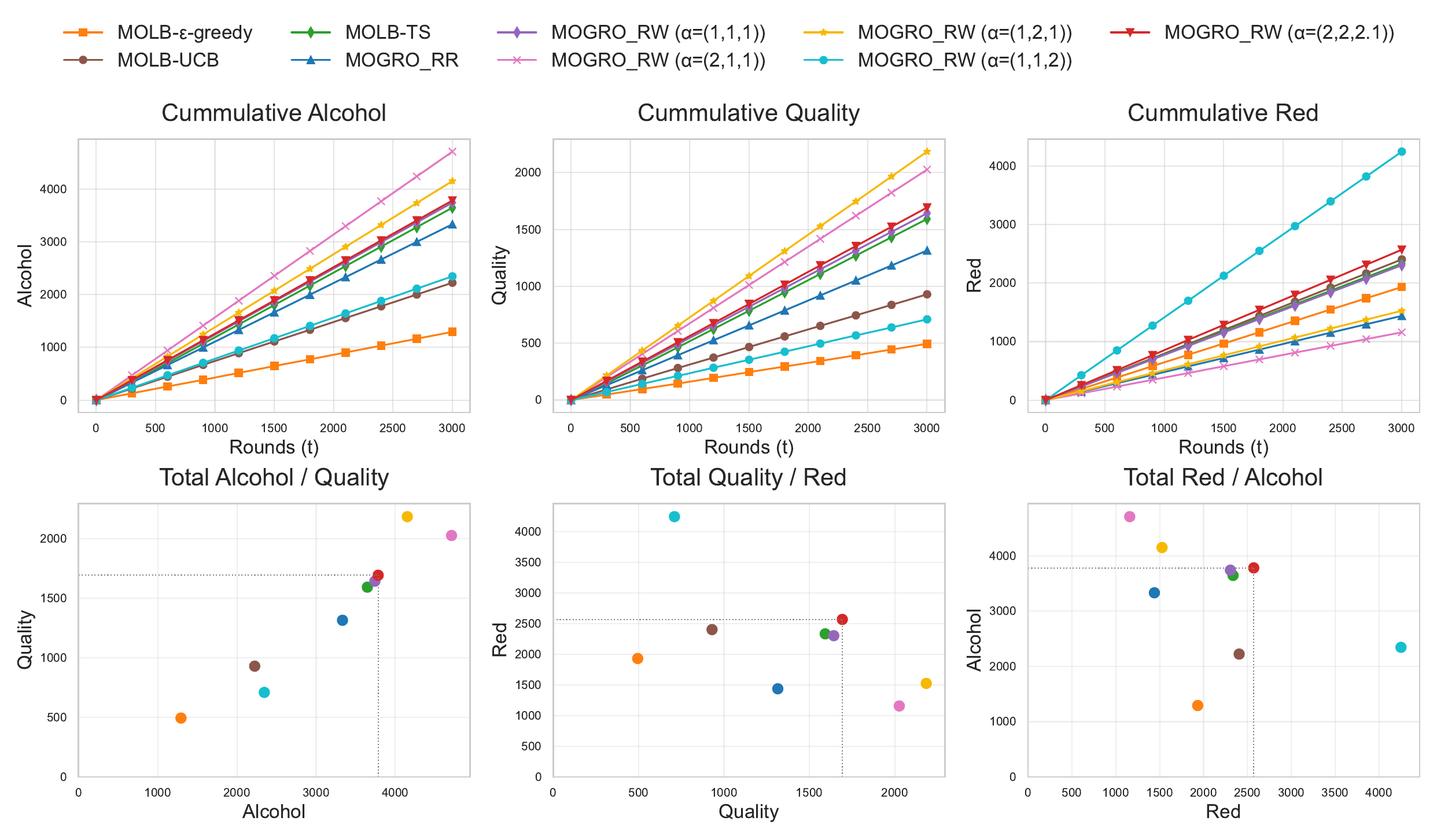} 
        \vspace{-20pt}
        \caption{$K=50$}
        \label{fig:exp_real_K50}
    \end{subfigure}
    
    \begin{subfigure}{\textwidth}
        \centering
        \includegraphics[width=\linewidth, trim=0cm 0cm 0cm 2.5cm, clip]{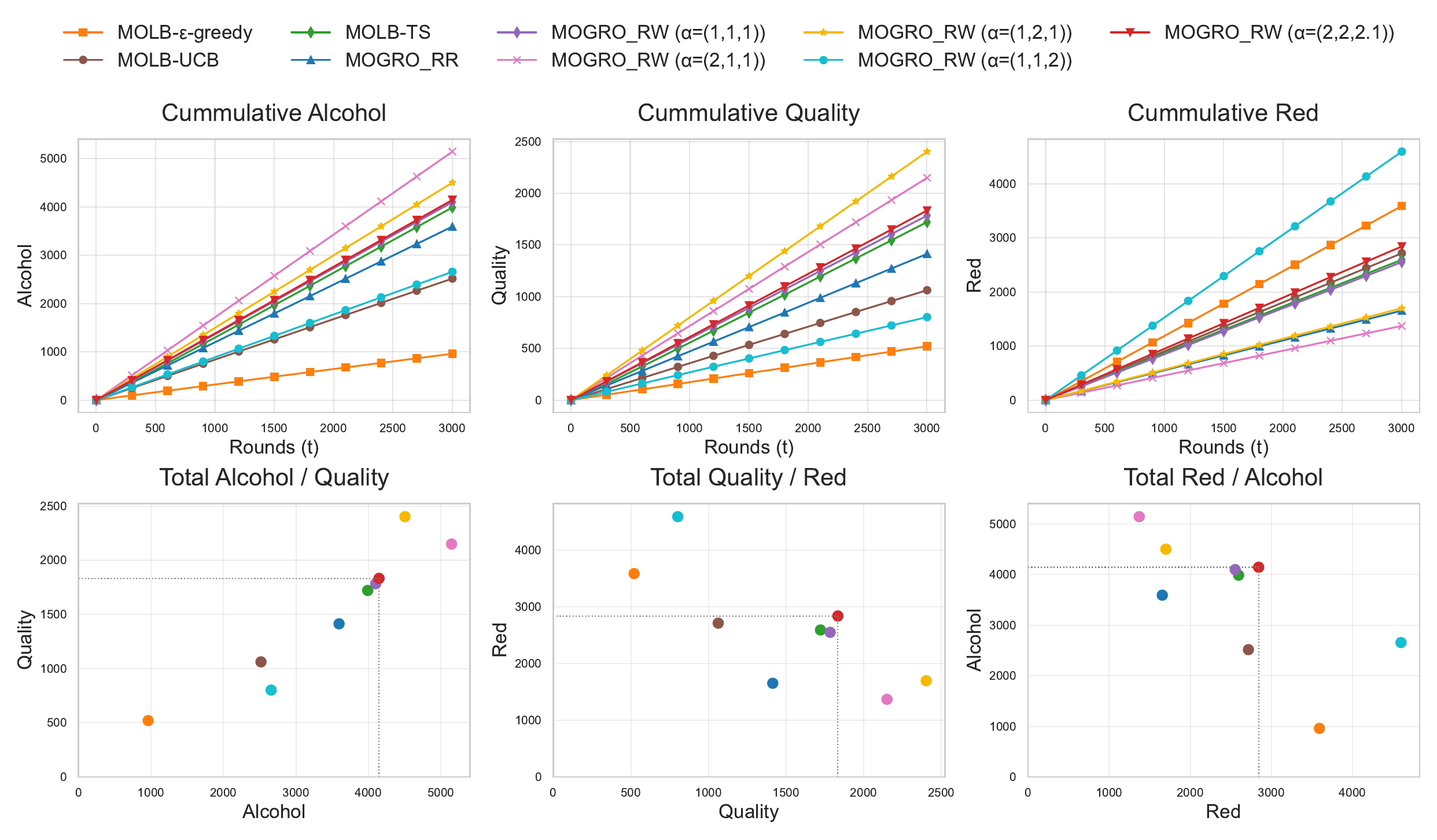} 
        \vspace{-20pt}
        \caption{$K=100$}
        \label{fig:exp_real_K100}
    \end{subfigure}
    \vspace{-10pt}
    \caption{Evaluation of multi-objective bandit algorithms on real-world data. In both experiments, the three plots in the top row report the cumulative reward for each objective, while the three plots in the bottom row display the final rewards of each algorithm as points in the two-objective space, with each axis corresponding to one objective. }
    \label{fig:exp_real}
\end{figure*}

\section{Auxiliary lemmas}
\label{ap_sec:lemmas}

\begin{lemma}\label{lem:rgl>0}
    If $\Wcal$ has a continuous density function $f$ which is positive on $\mathbb{S}^{M-1}$, then both regularity indices $\phi_{\epsilon,\Wcal}$ and $\psi_{\epsilon,\Wcal}$ are positive. 
\end{lemma}

\textit{Proof sketch.} It is enough to show $\psi_{\epsilon,\Wcal}>0$ since $\phi_{\epsilon,\Wcal}\ge \psi_{\epsilon,\Wcal}$. Fix $\epsilon>0$ and define $g:\mathbb{S}^{M-1} \rightarrow \RR^d$ such that $g(\bar{w}):=\mathbb{P}_{w\sim \Wcal}\left(\left\|\sum_{m\in [M]}w_m\theta_m^*-\sum_{m\in [M]}\bar{w}_m\theta_m^*\right\|_2 < \epsilon\right)$. Then, we can show that $g$ is a positive continuous function. Since $\mathbb{S}^{M-1}$ is compact, we have $\inf_{w\in \Delta^M}g(w)=\min_{w\in \mathbb{S}^{M-1}}g(w)>0$. 

\begin{lemma}[Lemma A.1. of \citet{Kannan2018}]
\label{lem:Kannan_A1}
    Let $\eta_1, \ldots, \eta_t$ be independent $\sigma^2$-subgaussian random variables. Let $x_1, \ldots,x_t$ be vectors in $\mathbb{R}^d$ with each $x_{s}$ chosen arbitrarily as a function of $(x_1, \eta_1), \ldots,(x_{s-1}, \eta_{t'-1})$ subject to $\|x_{s}\| \le x_{\max}$. Then with probability at least $1-\delta$, 
    \begin{equation*}
        \left\| \sum_{s=1}^t \eta_{s}x(s)\right\|_2 \le \sigma\sqrt{2x_{\max}dt \log (dt / \delta)}. 
    \end{equation*}
\end{lemma}

Note that, the above lemma holds even when $\eta_1, \ldots, \eta_t$ be conditionally $\sigma^2$-subgaussian random variables, because it was driven by using $\sigma^2$-subgaussian martingale.

\begin{lemma}[Lemma 8 of \citet{LihongLi2017}]
\label{lem:Chen_8}
    Given $\|x_i\| \le 1$ for all $i \in [K]$, suppose there is an integer $m$ such that $\lambdamin(V_m) \ge 1$, then for any $\delta >0$, with probability at least $1-\delta$, for all $t\ge m+1$,  
    \begin{equation*}
        \|S_t\|_{{V_{t}}^{-1}}^2\le 4 \sigma^2 ( {d \over 2} \log(1+{2t(x_{\max})^2 \over d})+\log({1 \over \delta})).
    \end{equation*}
\end{lemma}

\begin{lemma}[Theorem 3.1 of \citet{Tropp2011}]
\label{lem:Tropp_3.1}
    Let $\mathcal{H}_1 \subset \mathcal{H}_2 \cdots$ be a filtration and consider a finite sequence $\{X_k\}$ of positive semi-definite matrices with dimension $d$ adapted to this filtration. Suppose that $\lambda_{\max}(X_k) \le R$ almost surely. Define the series $Y \equiv \sum_k X_k$ and $W \equiv \sum_k \mathbb{E}[X_k | \mathcal{H}_{k-1}]$. Then for all $\mu \ge 0,~ \gamma \in [0,1)$ we have 
    \begin{equation*}
        \mathbb{P}[\lambdamin(Y) \le (1-\gamma) \mu ~~\text{and} ~~\lambdamin(W) \ge \mu] \le d({e^{-\gamma} \over (1-\gamma)^{1-\gamma}})^{\mu/R}. 
    \end{equation*}
\end{lemma}

\begin{lemma}[Theorem 5.1 of \citet{Auer2002}]
\label{lem:Auer2002}
    For any $T\ge K \ge 2$, consider the multi-armed bandit problem such that the probability slot machine pays $1$ is set to ${1 \over 2}+ {1 \over 4} \sqrt{K \over T}$ for one uniformly chosen arm and $1 \over 2$ for the rest of $K-1$ arms. Then, there exists $\gamma$ such that for any (multi-armed) bandit algorithm choosing action $a_t$ at time $t$, the expected regret is lower bounded by 
    \begin{equation*}
        \mathbb{E} \left(p_iT - \sum_{t=1}^T r_{t, a_t}\right) = \Omega( \sqrt{KT})_. 
    \end{equation*}
\end{lemma}

\begin{lemma}
\label{lem:succeq}
    For any random variable vector $X \sim D$, $\mathbb{E}[XX^\top]\succeq \mathbb{E}[X]\mathbb{E}[X]^\top$ 
\end{lemma}
\textit{Proof of Lemma~\ref{lem:succeq}.} For any unit vector $u \in \mathbb{R}^{d}$, $u^\top\mathbb{E}[XX^\top]u = \mathbb{E}[u^\top XX^\top u]=\mathbb{E}[\left\langle u, X \right\rangle^2]\ge (\mathbb{E}[\left\langle u,X \right\rangle])^2 = u^\top \mathbb{E}[X]\mathbb{E}[X]^\top u$.

\begin{lemma}
\label{lem:simple_linalg}
    Let $v$ be a vector in $S \subset \RR^d$ and $A$ be a $d \times d$ matrix. Then $\|Av\|_2 \ge (\min_{u \in S}u^\top A u) ~\|v\|_2$. 
\end{lemma}
\textit{Proof of Lemma~\ref{lem:simple_linalg}.} 
\begin{equation*}
    {\|Av\|_2 \over \|v\|_2} = \left\|A {v \over \|v\|_2}\right\|_2 \ge \min_{u \in S}\|Au\|_2 \ge \min_{u \in S} u^\top Au.  
\end{equation*}


\end{document}